\documentclass{article}

\usepackage{microtype}
\usepackage{graphicx}
\usepackage{subcaption}
\usepackage{booktabs} 

\usepackage{hyperref}

\usepackage[accepted]{icml2026}

\usepackage{amsmath}
\usepackage{amssymb}
\usepackage{mathtools}
\usepackage{amsthm}

\usepackage[capitalize,noabbrev]{cleveref}

\usepackage{tcolorbox}
\tcbuselibrary{breakable}

\usepackage{bm}

\usepackage{url}
\usepackage{wrapfig}
\numberwithin{equation}{section} 

\theoremstyle{plain}
\usepackage{thmtools,thm-restate}
\newtheorem{theorem}{Theorem}[section]

\theoremstyle{definition}
\newtheorem{definition}[theorem]{Definition}

\theoremstyle{remark}

\newtheorem*{remark*}{Remark}

\usepackage[textsize=tiny]{todonotes}

\icmltitlerunning{Towards Atoms of Large Language Models}

\begin{document}

\twocolumn[
  \icmltitle{Towards Atoms of Large Language Models}



  \begin{icmlauthorlist}
    \icmlauthor{Chenhui Hu}{institute,university}
    \icmlauthor{Pengfei Cao}{institute,university}
    \icmlauthor{Yubo Chen}{institute,university}
    \icmlauthor{Kang Liu}{institute,university}
    \icmlauthor{Jun Zhao}{institute,university}
  \end{icmlauthorlist}

  \icmlaffiliation{institute}{The Key Laboratory of Cognition and Decision Intelligence for Complex Systems, Institute of Automation, Chinese Academy of Sciences, Beijing, China.}
  \icmlaffiliation{university}{School of Artificial Intelligence, University of Chinese Academy of Sciences, Beijing, China}

  \icmlcorrespondingauthor{Pengfei Cao}{pengfei.cao@nlpr.ia.ac.cn}
  \icmlcorrespondingauthor{Jun Zhao}{jzhao@nlpr.ia.ac.cn}

  \icmlkeywords{Atom Theory, Fundamental Representational Units, Atomic Inner Product, Atoms, Large Language Models}

  \vskip 0.3in
]



\printAffiliationsAndNotice{}  

\begin{abstract}

    The fundamental representational units (FRUs) of large language models (LLMs) remain undefined, limiting further understanding of their underlying mechanisms. In this paper, we introduce \textbf{\textit{Atom Theory}} to systematically define, evaluate, and identify such FRUs, which we term atoms. Building on the atomic inner product (AIP), a non-Euclidean metric that captures the underlying geometry of LLM representations, we formally define atoms and propose two key criteria for ideal atoms: faithfulness ($R^2$) and stability ($q^*$). We further prove that atoms are identifiable under threshold-activated sparse autoencoders (TSAEs). Empirically, we uncover a pervasive representation shift in LLMs and demonstrate that the AIP corrects this shift to capture the underlying representational geometry. We find that two widely used units, neurons and features, fail to qualify as ideal atoms: neurons are faithful ($R^2\!=\!1$) but unstable ($q^*\!=\!0.5\%$), while features are more stable ($q^*\!=\!68.2\%$) but unfaithful ($R^2\!=\!48.8\%$). To find atoms of LLMs, leveraging atom identifiability under TSAEs, we show via large-scale experiments that reliable atom identification occurs only when the TSAE capacity matches the data scale. Guided by this insight, we identify FRUs with near-perfect faithfulness ($R^2\!=\!99.9\%$) and stability ($q^*\!=\!99.8\%$) across layers of Gemma2-2B, Gemma2-9B, and Llama3.1-8B, satisfying the criteria of ideal atoms statistically. Further analysis confirms that these atoms align with theoretical expectations and exhibit substantially higher monosemanticity. Overall, we propose and validate Atom Theory as a foundation for understanding the internal representations of LLMs.
  
\end{abstract}

\begin{figure}
    \centering
    \includegraphics[width=0.48\textwidth]{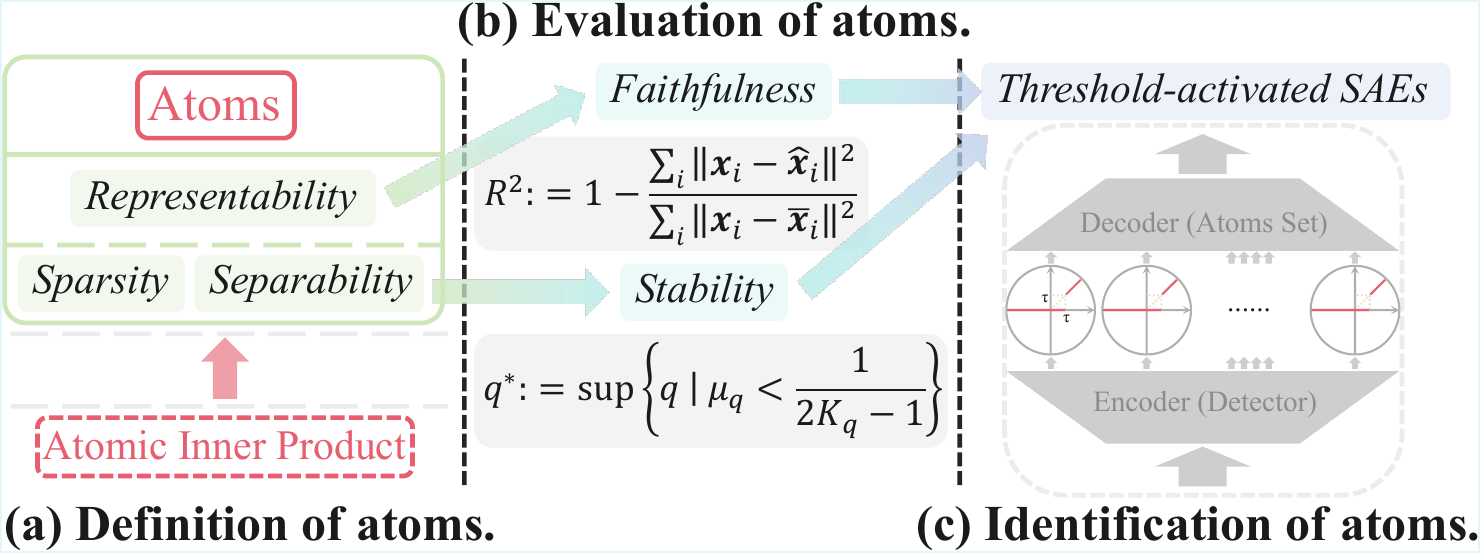}
    \caption{Illustration of Atom Theory. (a) Atoms are defined based on the atomic inner product, inducing representability, sparsity, and separability. (b) Atoms are evaluated by faithfulness ($R^2$) and stability ($q^*$), measuring fidelity and stable-atom fraction. (c) Threshold-activated SAEs enable atom identification, with the encoder as an atom detector and the decoder as the target atom set.}
  \label{fig:atoms_theory_illustration}
\end{figure}

\section{Introduction}

Large language models (LLMs), trained on vast corpora, exhibit emergent knowledge and reasoning abilities \citep{petroni2019language, brown2020language, achiam2023gpt}. Yet such information is not stored in explicit symbolic structures, but implicitly embedded within high-dimensional representations \citep{nanda2023emergent, gurnee2023finding}. This raises a fundamental question: \textbf{Do LLMs contain fundamental representational units (FRUs)—an atomic structure underlying how they encode and compose information?} These representational units are critical for understanding, interpreting, and controlling LLMs \citep{olah2020zoom}.

Traditionally, neurons have been regarded as such FRUs \citep{olah2017feature,dai2022knowledge,chen2024journey}. However, neurons frequently exhibit substantial polysemanticity \citep{elhage2022toy}, raising doubts about their validity for analysis. To address this issue, features extracted from internal representations \citep{cunningham2023sparse,chen2025knowledge} have been proposed as alternative FRUs \citep{olah2020zoom}. Yet this perspective remains controversial: (i) features fail to fully reconstruct the original representations \citep{bricken2023monosemanticity}, raising concerns about faithfulness; and (ii) features undergo splitting into finer ones or merging into broader ones under varying decomposition settings \citep{bussmann2025learning, chanin2025feature}, undermining stability. Although prior work implicitly treats neurons and features as FRUs, there is still no formal definition of FRUs for LLMs, which hinders principled evaluation and identification and ultimately limits theoretical clarity.

In this paper, we propose \textbf{\textit{Atom Theory}} to systematically define, evaluate, and identify the FRUs of LLMs, termed atoms (Fig.~\ref{fig:atoms_theory_illustration}). Specifically, to characterize the underlying geometry of LLM representations, we introduce the (non-Euclidean) atomic inner product (AIP) in $\S$~\ref{section:Atom_Theory_Atomic_Inner_Product}. Based on AIP, we formally define atoms (Fig.~\ref{fig:atoms_theory_illustration}(a)) in $\S$~\ref{section:Atom_Theory_Formal_Definition_of_Atoms} by three properties: \textbf{representability} (representations can be faithfully reconstructed from atoms), \textbf{sparsity} (each representation involves only a few atoms), and \textbf{separability} (atoms are approximately orthogonal under AIP). Representability ensures that atoms form a faithful set for the representation space. Sparsity and separability jointly enable efficient encoding of representations by approximately orthogonal atoms \citep{elhage2022toy}. Accordingly, sparsity and separability are tightly coupled, and we explicitly quantify their relationship in the following criteria.\vspace{-0.6pt} 

To operationalize this definition, we next introduce quantitative criteria for evaluating whether candidate units qualify as atoms (Fig.~\ref{fig:atoms_theory_illustration}(b)), theoretically in $\S$~\ref{section:Atom_Theory_Evaluation_of_Atoms} and practically in $\S$~\ref{section:Experiments_Neurons_or_Features_as_Ideal_Atoms}. Representability is measured by the coefficient of determination $R^2$, which quantifies \textit{faithfulness}. Sparsity and separability, drawing on compressed sensing \citep{donoho2006compressed, candes2006robust}, are unified into a single metric $q^*$ to quantify \textit{stability}. This metric corresponds to monorepresentationality: within the regime characterized by $q^*$, atoms and their combinations are distinguishable (i.e., not confusable), which is a desirable property under approximate orthogonality. Monorepresentationality provides a structural foundation for understanding LLM representations, offering the stability required for monosemanticity. Finally, to provide theoretical guarantees for atom identification, we prove that threshold-activated sparse autoencoders (TSAEs) (Fig.~\ref{fig:atoms_theory_illustration}(c)) can identify the target atom set. Overall, we establish a unified theoretical framework that provides guarantees for defining, evaluating, and identifying atoms.\vspace{-0.6pt} 

\textbf{Having introduced \textit{Atom Theory}, we validate that its foundation, the AIP, provides a principled basis for understanding LLM representations in $\S$~\ref{section:Experiments_Representation_Shift}.} Empirically, we uncover a pervasive representation shift across layers of differently scaled models from multiple LLM families (Fig.~\ref{fig:Representations_Shifting}), including GPT \citep{radford2019language,wang2022gpt}, Pythia \citep{biderman2023pythia}, Llama \citep{touvron2023llama,dubey2024llama}, and Gemma \citep{team2024gemma}. This shift arises from the Softmax operation in LLMs, which drives the centroid of the distribution of pairwise representation angles away from $90^\circ$ under the Euclidean inner product, thereby inducing a global bias in the representation space and distorting the representational geometry. Introducing the AIP effectively corrects this shift (Fig.~\ref{fig:Correcting_Representations_Shifting}), removing the global bias and restoring the centroid to $90^\circ$. This demonstrates that the AIP captures the underlying geometry of LLM representations.\vspace{-0.6pt} 

\textbf{Building on \textit{Atom Theory}, we evaluate whether candidate representational units satisfy the criteria for ideal atoms in $\S$~\ref{section:Experiments_Neurons_or_Features_as_Ideal_Atoms}.} We show that widely used representational units, neurons and features, remain distant from ideal atoms (Fig.~\ref{fig:neurons_features_ideal_atoms_radar_chart}). Although neurons, as the basic computational units of neural networks, exhibit perfect faithfulness ($R^2\!=\!1$), they display extremely low stability ($q^*\!=\!0.5\%$). Features achieve improved stability ($q^*\!=\!68.2\%$) but remain unstable and exhibit low faithfulness ($R^2\!=\!48.8\%$). These results quantitatively reveal the limitations of neurons and features, indicating that these common units are not ideal atoms.

\textbf{Leveraging \textit{Atom Theory}, we identify the atoms of LLMs in $\S$~\ref{section:Experiments_TSAE_Capacity_meet_Data_Scale} and $\S$~\ref{section:Experiments_Atoms_of_LLMs}.} Based on the theoretical identifiability of TSAEs, we conduct large-scale experiments to characterize the relationship between data scale and TSAE capacity (Fig.~\ref{fig:Scaling_Saes}), showing that reliable atom identification is achieved only when the TSAE capacity exceeds a critical threshold for a given data scale. This is intuitive: data scale determines the scale of atoms, which in turn dictates the TSAE capacity required for their identification. Guided by this insight, we achieve faithful reconstructions ($R^2\!=\!99.90\%$) across layers of Gemma2-2B, Gemma2-9B, and Llama3.1-8B using TSAEs with JumpReLU activation \citep{erichson2019jumprelu, rajamanoharan2024jumping}, and verify high stability of the identified units ($q^*\!=\!99.85\%$), yielding ideal atoms statistically (Tab.~\ref{tab:atom_faithfulness_stability}). Further analysis demonstrates that the identified atoms are consistent with theoretical expectations and exhibit substantially higher monosemanticity (Fig.~\ref{fig:monosemanticity_score_across_models}).

In summary, our contributions are as follows:\vspace{-4pt}
\begin{itemize}
    \item We propose Atom Theory, a rigorous theoretical framework based on AIP that systematically defines, evaluates, and identifies the FRUs of LLMs, i.e., atoms.\vspace{-1pt} 
    \item We empirically uncover a representation shift in LLMs and show that the AIP corrects this shift to characterize the underlying representational geometry, validating the representational foundation of Atom Theory.\vspace{-1pt} 
    \item Building on Atom Theory, we use faithfulness and stability to systematically and quantitatively reveal the limitations of neurons and features as FRUs.\vspace{-1pt} 
    \item Leveraging Atom Theory, we establish atom identifiability under TSAEs, characterize the relationship between data scale and TSAE capacity, and identify FRUs in LLMs that satisfy the criteria of ideal atoms.\footnote{Code available at \url{https://github.com/ChenhuiHu/towards_atoms}.} Further analysis shows that these atoms align with theoretical expectations and exhibit high monosemanticity.\vspace{-1pt} 
\end{itemize}

\vspace{-6pt} 
\section{Preliminary}\label{section:Preliminary}

\vspace{-2pt} 
\paragraph{Background on Language Models}

Consider an \(L\)-layer language model over a vocabulary \( V\). For an input sequence \(\bm{x}=[x_1,x_2,\cdots ,x_T]\) with \(x_i\in V\), the model assigns each token \(x_i\) an embedded representation \(\bm h_i^0\in\mathbb R^{H}\), which is updated at layer \(l\) as \(\bm h_i^l=\bm h_i^{l-1}+\bm a_i^{l}+\bm v_i^l\), where \(\bm a_{i}^l\) and \(\bm v_i^l\) denote the outputs of attention and MLP modules, respectively. From the residual-stream perspective, the representation after \(L\) layers is \(\bm h_i^L=\bm h_i^{0}+\sum_{l=1}^{L}\bm a_i^{l}+\sum_{l=1}^{L}\bm v_i^l\). The probability distribution \(\bm{y}\) over the next token is obtained from the final representation via \(\bm{y} = \mathrm{Softmax}(\bm W_U^\top \bm{h}_T^L)\), where \(\bm W_U \in \mathbb{R}^{H \times |V|}\) is the unembedding matrix.

\section{Atom Theory}\label{section:Atom_Theory}

In language models, all information is embedded in high-dimensional representations. Our objective is to identify the fundamental representational units (FRUs), which we term atoms. Formally, we consider a collection of representations \(M = \{\bm{m}_i\}_{i=1}^{|M|}\), where \(\bm{m}_i \in \mathbb{R}^H\). Each representation can be expressed as \(\bm{m}_i=\sum_{j}\delta(i,j)\bm d_j\), where \(\delta(i,j) \geq 0\) denotes the presence and magnitude of the \(j\)-th representational unit \(\bm d_j\in\mathbb{R}^H\) in the \(i\)-th representation \(\bm{m}_i\). However, in representation space, the family of representational units admitting such a decomposition is, in principle, infinite.

The central question is how to define FRUs, i.e., atoms. A natural criterion is distinguishability: each atom should be detectable or manipulable without interfering with others. In high-dimensional spaces, this criterion translates into orthogonality: atoms occupy mutually orthogonal directions, making their identities distinguishable via inner products. This motivates introducing an inner product to analyze the geometry of FRUs. Thus, the choice of inner product is critical. While the Euclidean inner product is commonly used, it is not necessarily appropriate for language models. Following \citet{park2023linear} and \citet{hu2025knowledge}, we consider the following reparameterization of \(\bm W_U\) and \(\bm{h}^L\):
\begin{equation}\label{equation:Reparameterization}
    {\bm W_U}' \leftarrow \bm A^{-\top} \bm W_U + \bm{b}\,\bm{1}^\top, \quad \bm{h}'^L \leftarrow \bm A\,\bm{h}^L,
\end{equation}
where \(\bm A\! \in \mathbb{R}^{H \times H}\) is an invertible linear transform, \(\bm{b} \!\in \mathbb{R}^{H}\), and \(\bm{1} \in \mathbb{R}^{|V|}\) is the all-ones vector. Owing to the translation invariance of Softmax, this reparameterization leaves the output distribution unchanged: \(\bm{y} = \mathrm{Softmax}(\bm W_U^\top \bm{h}^L) = \mathrm{Softmax}({\bm W_U'}^\top\! \bm{h}'^L)\). See Appendix~\ref{appendix:Proofs_Explanation} for further details.

Since the training objective of language models depends on representations solely through Softmax probabilities, different pairs \((\bm W_U, \bm h^L)\) under \eqref{equation:Reparameterization} are observationally indistinguishable, producing exactly the same outputs for all inputs. Thus, even for a trained checkpoint, such reparameterizations leave the model’s input–output behavior unchanged, so \(\bm h^L\) is identifiable only up to an invertible linear transformation \(\bm A\) in principle. Due to the residual-stream architecture and linearity of matrix multiplication, this invariance propagates to all hidden representations and their representational units \(\bm d_j\), which are likewise identifiable only up to \(\bm A\).

Crucially, the Euclidean inner product is not invariant under reparameterization \eqref{equation:Reparameterization}: \(\langle \bm{d}_i,\! \bm{d}_j \rangle\! \neq\! \langle \bm A\bm{d}_i,\! \bm A\bm{d}_j \rangle\). Therefore, the Euclidean geometric relations (e.g., angles and orthogonality) between $\bm d_i$ and $\bm d_j$ depend on the chosen parameterization, thus the Euclidean inner product does not provide a canonical geometry for language-model representations.

\vspace{-1pt}
\subsection{Atomic Inner Product}\label{section:Atom_Theory_Atomic_Inner_Product}

To better understand language-model representations and thereby define atoms within them, we require additional principles to specify an appropriate inner product. We therefore introduce an inner product with the desired properties.

\begin{definition}[Atomic Inner Product; AIP]
    Let \(\mathcal D\!=\!\mathrm{span}(D)\), where \(D\!=\!\{\bm d_j\}_{j=1}^{|D|}\) denotes the atom set. The \textbf{atomic inner product} \(\langle \cdot, \cdot \rangle_S\) is an inner product on \(\mathcal D\) such that \(\langle \bm{d}_i, \bm{d}_j \rangle_S\! =\! 0\) for all \(\bm{d}_i, \bm{d}_j \in D\) with \(i \neq j\).
\end{definition}

Atoms are indexed arbitrarily, and any permutation of indices leaves their geometry invariant. Hence, there is no principled basis for assigning different norms to different atoms. By this symmetry, we assume a common norm under the chosen inner product, i.e., \(\|\bm d_i\|_S \!=\! c>0\), \(\forall i \in [|D|]\). The constant \(c\) cancels naturally in the subsequent analysis.

By abuse of notation, we also use \(\bm D \in \mathbb{R}^{H \times |D|}\) to denote the matrix with columns \(\bm d_j\). We next characterize the atomic inner product in an explicit form.

\begin{restatable}[Explicit Form of the Atomic Inner Product]{theorem}{ExplicitForm}\label{theorem:Explicit_Form}
    Let \(\langle \bm{d}_i, \bm{d}_j \rangle_S = \bm{d}_i^\top \bm S \bm{d}_j\) be an atomic inner product with \(\bm S\in \mathbb{R}^{H\times H}\) symmetric and positive definite. If the columns of \(\bm D = [\bm{d}_1, \bm{d}_2, \cdots, \bm{d}_{|D|}]\) form the atom set such that \(\forall i,\ \|\bm{d}_i\|_S = c>0\), and \(\mathcal{D} \simeq \mathbb{R}^H\), then \(\bm S = c^2 (\bm D\bm D^\top)^{-1}\).
\end{restatable}

All proofs are provided in Appendix~\ref{appendix:Proofs}. By analogy with cosine similarity, we introduce the normalized atomic inner product to remove the dependence on \(c\).

\begin{restatable}[Normalized Atomic Inner Product; NAIP]{corollary}{NAIP}\label{corollary:NAIP}
    Let the atomic inner product be defined by \(\langle \bm{d}_i, \bm{d}_j \rangle_S = \bm{d}_i^\top \bm S \bm{d}_j\), where \(\bm S\) is symmetric and positive definite. Suppose the columns of \(\bm D = [\bm{d}_1, \cdots, \bm{d}_{|D|}]\) form the atom set satisfying \(\forall i ,\ \|\bm{d}_i\|_S = c>0\). Then, for any \(i, j\),
    \begin{equation}
        \langle\bm{d}_i, \bm{d}_j\rangle_{\tilde S} := \frac{\langle \bm{d}_i, \bm{d}_j \rangle_S}{\|\bm{d}_i\|_S \|\bm{d}_j\|_S} = \bm{d}_i^\top \tilde{\bm S} \bm{d}_j,\quad \tilde{\bm S} = (\bm D\bm D^\top)^{-1}.
    \end{equation}
    Consequently, the bilinear form \(\langle \bm{d}_i, \bm{d}_j \rangle_{\tilde{S}} = \bm{d}_i^\top \tilde{\bm S} \bm{d}_j\) defines the \textbf{normalized atomic inner product}.
\end{restatable}

Unlike the causal inner product \citep{park2023linear}, which is defined in the static and input-independent unembedding space over output tokens, our AIP (or NAIP) operates on the dynamic and input-dependent representation space, characterizing the geometry of representations and their constituent units. This enables direct analysis of internal representations and thereby allows the definition of atoms as fundamental representational units under the AIP.

\begin{remark*}
    Define \(\tilde{\bm{d}_i}\! = \!\tilde{\bm S}^{\frac{1}{2}} \bm{d}_i\) and \(\tilde{\bm{d}_j} \!=\! \tilde{\bm S}^{\frac{1}{2}} \bm{d}_j\). Under this transformation, \(\langle \bm{d}_i, \bm{d}_j \rangle_{\tilde{S}}=\langle \tilde{\bm{d}_i}, \tilde{\bm{d}_j} \rangle\), where the right-hand side denotes the Euclidean inner product. Hence, properties of the Euclidean inner product transfer directly to the NAIP; \(\tilde{\bm{d}_i}\) and \(\tilde{\bm{d}_j}\) are accordingly termed \textbf{\textit{normalized atoms}}.
\end{remark*}

\subsection{Formal Definition of Atoms}\label{section:Atom_Theory_Formal_Definition_of_Atoms}

Having established a principled perspective for understanding representations in language models, we proceed to formally define atoms. However, the preceding analysis assumes an idealized setting in which atoms are strictly orthogonal under the chosen inner product. Although this assumption yields a clean notion, it constrains the number of atoms to at most the representation dimension (\(|D| = H\)), rendering the formulation impractical.

\citet{elhage2022toy} observed that sparsity induces the emergence of approximately orthogonal representations in neural networks to cope with limited representational dimensionality, a phenomenon termed superposition. \citet{hu2025knowledge} later identified pervasive superposition in language models. Motivated by these findings, we introduce sparsity, which naturally leads to approximate orthogonality and enables a formally grounded, practical definition of atoms.

\begin{definition}[Sparsity Level]
    Let \(M \!=\! \{\bm{m}_i\}_{i=1}^{|M|} \!\subset\! \mathbb{R}^H\) be a collection of representations. Suppose there exist \(\bm D \!=\! [\bm{d}_1,\! \cdots\!, \bm{d}_{|D|}]\! \in\! \mathbb{R}^{H \times |D|}\) and \(\bm{\delta}_{i} \!\in\! \mathbb{R}^{|D|}_{\geq 0}\) such that \(\bm{m}_i\! =\! \bm D \bm{\delta}_i\), \(\forall i\!\in\![|M|]\). The \textbf{sparsity level} is \(K\!:=\!\max_{ i}\|\bm \delta_i\|_0\), where \(\|\cdot\|_0\) denotes the \(\ell_0\) norm.
\end{definition}

Sparsity enables the number of atoms to substantially exceed the ambient dimension, yielding an overcomplete structure with \(|D|\gg H\) that captures rich world knowledge, while simultaneously minimizing mutual interference.

\begin{remark*}
    In the basic setting of $\S$~\ref{section:Atom_Theory_Atomic_Inner_Product} with \(|D| = H\), one can verify that \(\tilde{\bm S} = (\bm D\bm D^\top)^{-1}\) satisfies \(\bm D^\top \tilde{\bm S} \bm D = \bm I_{|D|}\), so the atoms form an orthonormal basis under \(\langle \cdot, \cdot \rangle_S\). When \(|D|\!\gg\! H\), \(\tilde{\bm S} \!= (\bm D\bm D^\top)^{-1}\) is well defined provided \(\mathrm{rank}(\bm D) = H\), but \(\bm G \!:= \bm D^\top \tilde{\bm S} \bm D\) becomes a rank-\(H\) projection rather than \(\bm I_{|D|}\), thus the atoms cannot all be mutually orthogonal under \(\langle \cdot, \cdot \rangle_S\). Therefore, exact orthogonality is unattainable, motivating the introduction of approximate orthogonality.
\end{remark*}

This consideration motivates the following definition.

\begin{definition}[\(\epsilon\)-Approximately Orthogonal Atoms]\label{definition:Approximately_Orthogonal_Atoms}
    The atom set \(\{\bm d_i\}_{i=1}^{|D|}\) is \textbf{\(\epsilon\)-approximately orthogonal} if \(\bigl|\langle \bm d_i, \bm d_j \rangle_{\tilde S}\bigr| \le \epsilon,\ \forall i \neq j,\) where \(\langle \bm x, \bm y \rangle_{\tilde S} := \bm x^\top \tilde{\bm S} \bm y\) denotes the normalized atomic inner product and \(\tilde{\bm S} := (\bm D\bm D^\top)^{-1}\).
\end{definition}

\begin{remark*}
    In the ideal setting of exact orthogonality, the inner products \(\langle \bm d_i, \bm d_j \rangle_{\tilde S}\) (equivalently, \(\langle \tilde{\bm{d}_i}, \tilde{\bm{d}_j} \rangle\)) for \(i \neq j\) follow a Dirac measure concentrated at the origin. Under \(\epsilon\)-approximate orthogonality, they are instead expected to follow a Gaussian distribution \(\mathcal N(0, s^2)\) with small variance, which converges to the Dirac measure as \(s \to 0\).
\end{remark*}

We now introduce a formal definition of atoms.

\begin{definition}[Atoms]\label{definition:Atoms}
    Let \(M = \{ \bm{m}_i \}_{i=1}^{|M|}\) be a collection of representations. Suppose there exists \(\bm D \!= \![\bm{d}_1,\! \cdots, \!\bm{d}_{|D|}] \!\in \mathbb{R}^{H \times |D|}\) and \(\bm \Delta\!=\![\bm\delta_1,\!\cdots\!,\bm\delta_{|M|}]\!\in\mathbb R_{\ge0}^{|D|\times |M|}\) such that, for a given sparsity level \(K \in \mathbb{N}\),
    \begin{equation}\label{eq:atoms}
         \forall i\in[|M|],\quad \bm{m}_i = \bm D \bm{\delta}_i\ \text{ with }\ \|\bm{\delta}_i\|_0 \leq K.
    \end{equation}
    Furthermore, \(\bigl|\langle \bm d_i, \bm d_j \rangle_{\tilde S}\bigr|\! \le \!\epsilon,\ \forall i\! \neq\! j\), where \(\tilde{\bm S} \!:=\! (\bm D\bm D^\top)^{-1}\). Under these conditions, \(\{\bm d_i\}_{i=1}^{|D|}\) is called the \textbf{atom set} of \(M\), and each \(\bm d_i\) is referred to as an \textbf{atom}.
\end{definition}

Intuitively, atoms are characterized by three properties: \textbf{representability}, where each representation can be faithfully expressed by atoms, i.e., \(\bm m_i = \bm D \bm \delta_i\); \textbf{sparsity}, where each representation involves only a few atoms, i.e., \(\|\bm \delta_i\|_0 \le K\); and \textbf{separability}, where atoms are approximately orthogonal, i.e., \(|\langle \tilde{\bm d_i}, \tilde{\bm d_j} \rangle| \le \epsilon\). Representability is a natural requirement, while sparsity and separability jointly enable efficient encoding of separable information under approximate orthogonality. We further quantitatively characterize the relationship between sparsity and separability in $\S$~\ref{section:Atom_Theory_Evaluation_of_Atoms}.

\begin{remark*}
    \(K\) quantifies sparsity without enforcing a specific sparsity regime, ensuring broad applicability of the definition. Moreover, pre-multiplying both sides of \eqref{eq:atoms} by \(\tilde{\bm S}^{\frac{1}{2}}\) yields \(\tilde{\bm{m}}_i = \tilde{\bm D} \bm{\delta}_i\), with \(\tilde{\bm{m}}_i := \tilde{\bm S}^{\frac{1}{2}} \bm{m}_i\) and \(\tilde{\bm D} := \tilde{\bm S}^{\frac{1}{2}} \bm D = [\tilde{\bm{d}}_1, \cdots, \tilde{\bm{d}}_{|D|}]\), which simplifies subsequent derivations.
\end{remark*}

\subsection{Evaluation of Atoms}\label{section:Atom_Theory_Evaluation_of_Atoms}

Having formalized atoms via representability, sparsity, and separability, we now evaluate whether candidate representational units satisfy the criteria of ideal atoms. Representability is measured by the coefficient of determination \(R^2\! := \!1\! -\! \frac{\sum_i \|\bm x_i - \hat{\bm x}_i\|^2}{\sum_i \|\bm x_i - \bar{\bm x}\|^2}\), which quantifies the proportion of variance in the original representations explained by atoms and thus reflects faithfulness. While sparsity and separability can be quantified by \(K\) and \(\epsilon\), they do not directly indicate unit quality. In this section, we introduce a unified metric that integrates sparsity and separability to characterize the stability of representational units.

We note that \(\bm{\delta}_i \!\in \mathbb{R}^{|D|}\) can be viewed as a sparse representation of high-dimensional semantics, which is compressed by \(\bm D\! \in\! \mathbb{R}^{H \times |D|}\) to produce \(\bm{m}_i\!\in\! \mathbb{R}^{H}\). This perspective reveals a close connection to compressed sensing (\citet{donoho2006compressed}; \citet{candes2006robust}), whose core idea is that a signal sparse in some basis can be recovered from far fewer linear measurements than its ambient dimension, with the Restricted Isometry Property (RIP) \citep{goessmann2020restricted, chen2021conditioning} providing the essential guarantee.

\begin{definition}[Restricted Isometry Property; RIP]
    A matrix \(\tilde{\bm D} \in \mathbb{R}^{H \times |D|}\) is said to satisfy the \(K\)-RIP if there exists a constant \(\delta_K \in [0,1)\) such that, for any \(K\)-sparse vector \(\bm{\delta} \in \mathbb{R}^{|D|}\) (i.e., \(\|\bm\delta\|_0\le K\)),
    \begin{equation}
        (1-\delta_K)\|\bm{\delta}\|_2^2 \le \|\tilde{\bm  D} \bm{\delta}\|_2^2 \le (1+\delta_K)\|\bm{\delta}\|_2^2.
    \end{equation}
    Here, \(\delta_K\) is called the \(K\)-RIP constant of \(\tilde{\bm D}\).
\end{definition}

Intuitively, projecting a sparse vector via \(\tilde{\bm D}\) to lower dimension preserves its geometric structure, ensuring the possibility of recovery. Direct verification of the RIP is NP-hard, but coherence \citep{foucart2013invitation,murdock2020dataless} provides a computable upper bound on \(\delta_k\).

\begin{restatable}[Coherence–RIP Upper Bound]{theorem}{RIPUpperBound}\label{theorem:Coherence–RIP_Upper_Bound}
    Let \(\tilde{\bm D}\!\in \mathbb{R}^{H\times |D|}\) and define the coherence \(\mu \!:= \max_{i\neq j} |\langle \tilde{\bm d_i}, \tilde{\bm d_j}\rangle| \le \varepsilon\). For any \(K\)-sparse vector \(\bm{\delta} \in \mathbb{R}^{|D|}\) with \(\|\bm{\delta}\|_0 \le K\),
    \begin{equation}
        \big(1\!-\!(K\!-\!1)\mu\big)\|\bm{\delta}\|_2^2 \le\!\|\tilde{\bm D}\bm{\delta}\|_2^2\!\le\!\big(1\!+\!(K\!-\!1)\mu\big)\|\bm{\delta}\|_2^2.
    \end{equation}
    Hence \(\delta_K(\tilde{\bm D}) \le (K-1)\mu\); in particular, \(\tilde{\bm D}\) satisfies the \(K\)-RIP whenever \((K-1)\mu < 1\).
\end{restatable}

In other words, coherence provides a computable criterion for verifying the RIP, ensuring that all \(K\)-sparse vectors projected through \(\tilde{\bm D}\) preserve geometric structure, an essential prerequisite in compressed sensing. Nevertheless, the RIP alone does not preclude non-uniqueness \citep{donoho2001uncertainty,candes2005decoding}: even if \((K-1)\mu < 1\) holds, the sparse coefficients associated with representations need not be unique.

\begin{restatable}[Uniqueness and Exact \(\ell_1\!\) Recoverability]{theorem}{UniquenessRecovery}\label{theorem:Uniqueness_of_Sparse_Representation_and_Exact_Recovery}
    Let \(\tilde{\bm D}\in \mathbb{R}^{H\times |D|}\) and define the coherence \(\mu := \max_{i\neq j} |\langle \tilde{\bm d_i}, \tilde{\bm d_j}\rangle| \le \varepsilon\). If \(\mu < \frac{1}{2K-1}\), then for every \(\bm{\delta}\in\mathbb R^{|D|}\) with \(\|\bm{\delta}\|_0 \le K\), the \(K\)-sparse representation determined by \(\tilde{\bm m} = \tilde{\bm D}\bm{\delta}\) is unique; that is, no other \(K\)-sparse vector yields the same \(\tilde{\bm m}\). Moreover, \(\bm{\delta}\) is the unique minimizer of the convex program 
    \begin{equation}
        \min_{\bm{x} \in \mathbb{R}^{|D|}} \|\bm{x}\|_1\quad \text{subject to} \quad\tilde{\bm D}\,\bm{x} = \tilde{\bm m}.
    \end{equation}
\end{restatable}

Crucially, this intrinsically characterizes the monorepresentationality of \(\tilde{\bm D}\): under the condition \(\mu < \tfrac{1}{2K-1}\), any representation formed as a \(K\)-sparse linear combination of representational units in \(\tilde{\bm D}\) has a unique sparse decomposition, with no other combination yielding the same representation.

\begin{restatable}[Monorepresentationality / Injectivity]{corollary}{Monorepresentationality}\label{corollary:Monorepresentationality}
    \ Under the condition \(\mu < \tfrac{1}{2K-1}\) of Theorem~\ref{theorem:Uniqueness_of_Sparse_Representation_and_Exact_Recovery}, define \(\Sigma_K := \{\bm{\delta} \in \mathbb{R}^{|D|} : \|\bm{\delta}\|_0 \le K\}\) and \(\Phi : \Sigma_K \to \mathbb{R}^H\) by \(\Phi(\bm{\delta}) := \tilde{\bm D}\bm{\delta}\). Then \(\Phi\) is injective on \(\Sigma_K\). That is, for any \(\bm{x}, \bm{y} \in \Sigma_K\), if \(\tilde{\bm D}\bm{x} = \tilde{\bm D}\bm{y}\), it follows that \(\bm{x} = \bm{y}\).
\end{restatable}

Within this regime, representational units and their combinations are unambiguous in representation space, i.e., monorepresentationality. By contrast, monosemanticity (\citet{bricken2023monosemanticity}; \citet{templeton2024scaling}) concerns the alignment of a unit with a specific meaning, concept, or function. The former is formally provable, whereas the latter is statistical and interpretive. Monorepresentationality is a prerequisite for monosemanticity, as it provides the structural stability; otherwise, units would be non-unique and interpretation unstable. Without additional semantic anchoring or inductive assumptions, theoretical guarantees primarily concern structural stability induced by monorepresentationality, while monosemanticity must be determined empirically via further semantic alignment experiments.

Sparsity (\(K\)) and separability (\(\mu\)) are thus unified by the condition \(\mu < \frac{1}{2K-1}\), under which stability holds. More practical metrics are discussed in $\S$~\ref{section:Experiments_Neurons_or_Features_as_Ideal_Atoms} and Appendix~\ref{appendix:Atoms_of_LLMs_Evaluation_Details}.

\begin{remark*}
    For intuition, the above result can be viewed as a generalization of the strictly orthogonal case (\(|D| = H\)). In this setting, \(\tilde{ \bm D} \in \mathbb{R}^{H \times H}\) has orthogonal columns (\(\mu = 0\)), i.e., \(\tilde{ \bm D}^\top \tilde{\bm  D} = \bm I\). Any representation \(\tilde{\bm m} \in \mathbb{R}^H\) then admits a unique decomposition \(\bm\delta = \tilde{\bm  D}^\top \tilde{\bm m}\), satisfying \(\tilde{\bm m} = \tilde{ \bm D} \bm\delta\). Here the sparsity level \(K = H = |D|\) satisfies \(\mu < \tfrac{1}{2K-1}\). Thus, under strict orthogonality, representations and atom coefficients are in one-to-one correspondence, yielding a direct and unique identification of atoms.
\end{remark*}

\begin{figure*}
    \centering
    \includegraphics[width=0.98\textwidth]{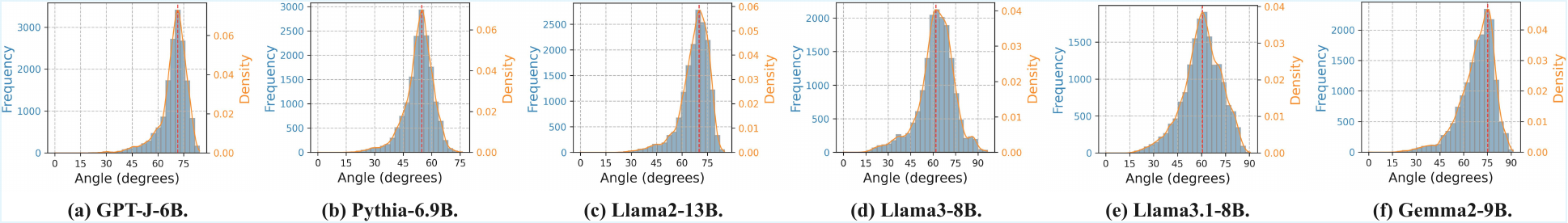}
    \caption{Representation shift at the final layer across multiple LLMs under the Euclidean inner product, with the centroid of pairwise representation angles deviating from \(90^\circ\). See Appendix~\ref{appendix:Representation_Shift} for full results.}
    \vspace{-1pt}
  \label{fig:Representations_Shifting}
\end{figure*}

\begin{figure*}
    \centering
    \includegraphics[width=0.98\textwidth]{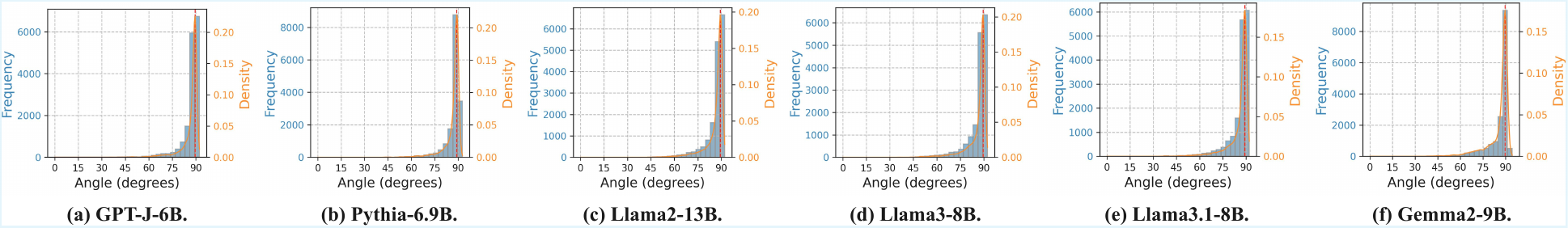}
    \caption{Correction of representation shift at the final layer across multiple LLMs via the atomic inner product, with the centroid of pairwise representation angles consistently approaching \(90^\circ\). See Appendix~\ref{appendix:Representation_Shift} for full results.}
    \vspace{-1pt}
  \label{fig:Correcting_Representations_Shifting}
\end{figure*}

\vspace{-1pt}
\subsection{Identification of Atoms}\label{section:Atom_Theory_Identification_of_Atoms}
\vspace{-1pt}

We have defined atoms and established criteria for ideal atoms. A central question remains whether such atoms can be identified in practice. Since sparse autoencoders (SAEs) are a standard approach for learning disentangled representations \cite{cunningham2023sparse}, we next demonstrate that, under appropriate conditions, SAEs can indeed identify these atoms, rendering the theory practically applicable.\vspace{-1pt}

\begin{restatable}[Identifiability of Threshold-activated SAEs; TSAEs]{theorem}{IdentifiabilitySAEs}\label{theorem:Identifiability_of_SAEs}
    Let \(M=\{\bm m_i \}_{i=1}^{|M|}\subset \mathbb R^H\) with \(\bm m_i = \bm D \bm \delta_i\), where \(\bm D = [\bm d_1,\cdots,\bm d_{|D|}] \in \mathbb{R}^{H\times |D|}\) satisfies \(|\langle \tilde{\bm{ d}_i}, \tilde{\bm{ d}_j} \rangle| \le \epsilon\) for all \(i\ne j\). Suppose each \(\bm \delta_i \in \mathbb{R}^{|D|}\) is \(K\)-sparse, i.e. \(\|\bm \delta_i\|_0 \le K\). Consider the threshold activation function\vspace{-7pt}
    \begin{equation}
        \sigma_\tau(x) =\begin{cases}0 & x < \tau, \\x & x \ge \tau,\end{cases}
        \vspace{-6pt}
    \end{equation}
    with threshold \(\tau > 0\). Assume there exist constants \(0 < \delta_{\min} \le \delta_{\max}\) such that, for each support \(\mathcal S_i = \mathrm{supp}(\bm \delta_i)\), \(\delta_{\min}\le\delta_{ij}\le\delta_{\max},\ \forall j\in \mathcal S_i\). If the amplitude gap and threshold satisfy \(\varepsilon K\delta_{\max} <\tau <\delta_{\min} - \varepsilon (K-1)\delta_{\max}\), which is feasible whenever \(\delta_{\min} > \varepsilon(2K-1)\delta_{\max}\), then setting TSAEs with \(\bm W_{dec}=\bm D\) and \(\bm W_{enc}=\bm D^\top \tilde{ \bm S}\) yields\vspace{-4pt}
    \begin{equation}
        \forall i,\quad \bm W_{dec}\,\sigma_\tau(\bm W_{enc}\bm m_i) = \bm m_i.
        \vspace{-4pt}
    \end{equation}
    This parameterization exactly recover the atom set \(D\).
    \vspace{-2pt}
\end{restatable}

Therefore, TSAEs can identify atoms in principle. By contrast, ReLU \cite{templeton2024scaling}, lacking a threshold term, fails to satisfy the support-separation condition and is thus theoretically insufficient. This responds to \citet{o2024compute}: the limitation of SAEs arises not from their linear–nonlinear mechanism, but from the absence of threshold activation, which prevents effective atom identification.\vspace{-1pt}

The threshold activation considered here denotes a class of activation functions rather than a specific instantiation. Existing examples include JumpReLU \citep{erichson2019jumprelu, rajamanoharan2024improving} and Top\(K\) \citep{makhzani2013k, gao2024scaling}. Although Top\(K\) is analogous in some respects, it relies on a fixed \(K\), limiting adaptivity and practical applicability. Moreover, the motivation differs fundamentally: threshold activation is proposed to enable effective atom identification under approximate orthogonality, whereas JumpReLU is introduced to address feature shrinkage \cite{rajamanoharan2024improving}, and Top\(K\) is designed to directly control sparsity \cite{gao2024scaling}.

\begin{remark*} 
    Although the theorem is stated for a uniform scalar threshold \(\tau\), it extends directly to a coordinate-wise threshold vector \(\bm \tau\), without affecting the squeeze condition or the proof. This generalization enlarges the feasible interval when activation magnitudes differ, thereby improving support separation and the robustness of atom identification. 
\end{remark*}

\section{Experiments}\label{section:Experiments}

We next empirically validate and apply Atom Theory. In $\S$~\ref{section:Experiments_Representation_Shift}, we uncover a pervasive representation shift in large language models (LLMs) and show that the atomic inner product (AIP) corrects it, capturing the representational geometry and validating the foundation of Atom Theory. In $\S$~\ref{section:Experiments_Neurons_or_Features_as_Ideal_Atoms}, we use faithfulness and stability to quantitatively reveal the limitations of neurons and features as fundamental representational units (FRUs). Leveraging the identifiability guarantees of threshold-activated SAEs (TSAEs), we establish in $\S$~\ref{section:Experiments_TSAE_Capacity_meet_Data_Scale} the relationship between data scale and TSAE capacity through large-scale experiments. Finally, in $\S$~\ref{section:Experiments_Atoms_of_LLMs}, we identify FRUs across LLMs that exhibit high faithfulness and stability, and demonstrate strong monosemanticity.

\subsection{Representation Shift}\label{section:Experiments_Representation_Shift}

In this section, we uncover a pervasive representation shift in LLMs and show that the AIP corrects it, capturing the representational geometry and grounding Atom Theory.

\vspace{-8pt}
\paragraph{Experimental Setup}

We randomly sample 128 subject entities from WikiData \cite{vrandevcic2014wikidata} and extract the corresponding activations across all layers of multiple LLM families, including GPT-2, GPT-J, Pythia, Llama-2, Llama-3, Llama-3.1, and Gemma2, which serve as the target representations. This sample size suffices to characterize the distribution of pairwise representation angles, and additional samples do not change the overall distribution (see Appendix~\ref{appendix:Representation_Shift}, Figs.~\ref{fig:Representations_Shifting_GPT2-Small_ablation}-\ref{fig:Representations_Shifting_Gemma2-9B_ablation}). It also facilitates ensuring no overlap with activations used in subsequent sampling.

We then collect 100K activations \(\bm{k}\) per layer from Wikipedia corpora (with no overlap with the target representations) and compute \(\mathbb{E}[\bm{k}\bm{k}^\top]\) to estimate the normalized AIP.

To analyze the distribution of representations, we use cosine similarity, which removes scale effects and aligns with the theoretical framework. For representations \(\bm u, \bm v \in \mathbb{R}^H\), the cosine similarity under the Euclidean inner product (EIP) is defined as \(\cos(\bm u,\! \bm v)\! \!=\! \frac{\langle \bm u, \bm v \rangle}{\|\bm u\|_2 \|\bm v\|_2}\). Under the AIP induced by \(\tilde {\bm S}\), the corresponding cosine similarity is \( \cos_{\tilde S}(\bm u,\! \bm v)\! =\! \frac{\langle \bm u, \bm v \rangle_{\tilde S}}{\|\bm u\|_{\tilde S} \|\bm v\|_{\tilde S}} \!=\! \frac{\bm u^\top\! \tilde{\bm  S} \bm v}{\sqrt{\bm u^\top\! \tilde{ \bm S} \bm u} \sqrt{\bm v^\top\! \tilde{ \bm S} \bm v}}\), where \(\tilde{ \bm S}\) is estimated in practice as \((\mathbb{E}[\bm k \bm k\!^\top])^{-1}\). For clarity, cosine similarities are further converted into angles. See Appendix~\ref{appendix:Representation_Shift} for more details.

\vspace{-8pt}
\paragraph{Experimental Results}

When representation angles are computed under the EIP, the centroid of the angular distribution deviates markedly from \(90^\circ\), revealing a representation shift (Fig.~\ref{fig:Representations_Shifting}; full results in Appendix~\ref{appendix:Representation_Shift}, Figs.~\ref{fig:Representations_Shifting_GPT2-Small}-\ref{fig:Representations_Shifting_Gemma2-9B}). This phenomenon consistently appears across all layers of LLM families, regardless of architectures or training corpora, indicating a systematic angular bias among representations. This bias shows that LLM representations are globally attracted toward a dominant direction, such that even unrelated representations exhibit high cosine similarity. As shown in Appendix~\ref{appendix:Representation_Shift} (Figs.~\ref{fig:Representations_Shifting_GPT2-Small_ablation}–\ref{fig:Representations_Shifting_Gemma2-9B_ablation}), increasing the sample size used to characterize the distribution leaves the centroid unchanged, confirming attraction toward the same direction.

In contrast, the AIP corrects this shift, restoring the centroid of the angular distribution to \(90^\circ\) (Fig.~\ref{fig:Correcting_Representations_Shifting}; full results in Appendix~\ref{appendix:Representation_Shift}, Figs.~\ref{fig:Solving_Representations_Shifting_GPT2-Small}-\ref{fig:Solving_Representations_Shifting_Gemma2-9B}). This indicates that the AIP removes the systematic angular bias induced by the EIP, so that angles between representations reflect genuine differences rather than metric-induced artifacts, which is essential for further analysis of LLM representations. Collectively, these results show that the AIP captures the underlying geometry of LLM representations and grounds Atom Theory.\vspace{-1pt}

\begin{figure}
    \centering
    \includegraphics[width=0.50\textwidth]{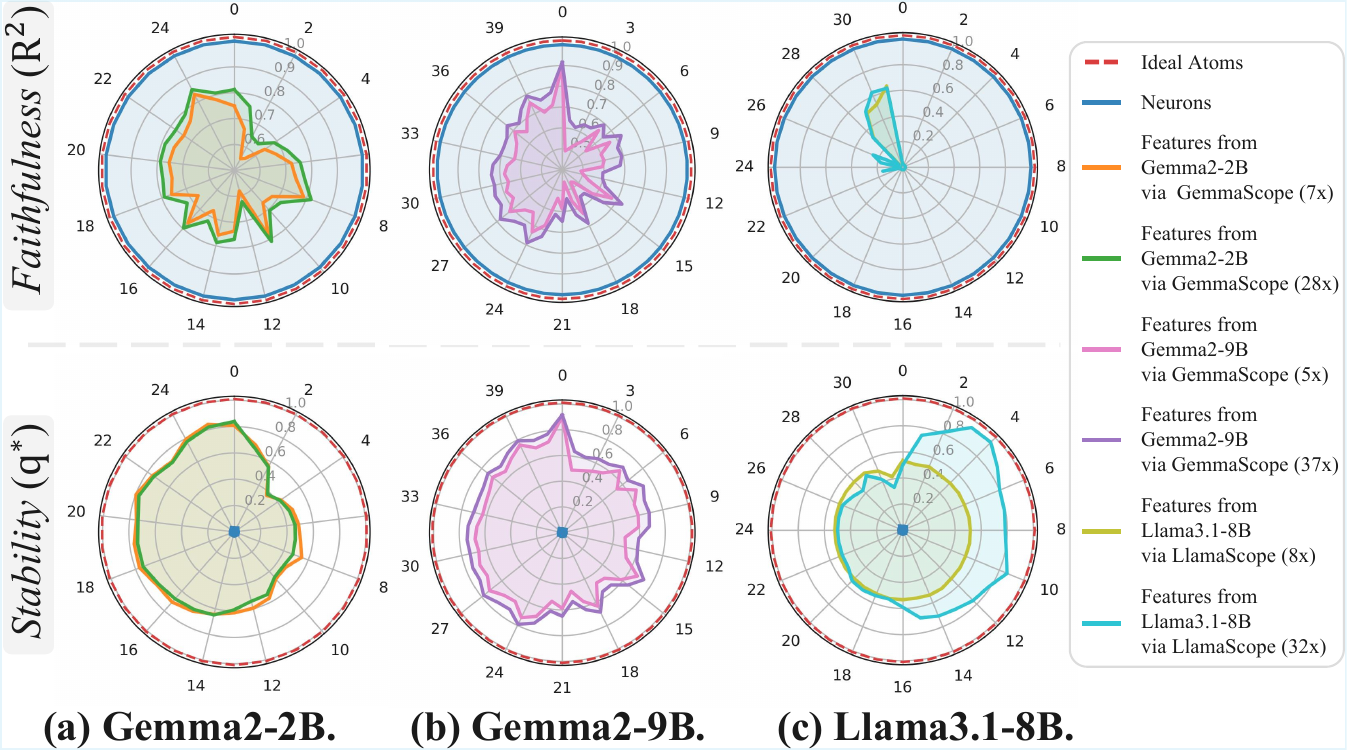}
    \caption{Comparison of neurons, features, and ideal atoms across all layers of different LLMs. Ideal atoms are required to exhibit both high faithfulness and high stability, corresponding to $R^2 = 1$ and $q^* = 1$, respectively. Values of $R^2$ below 0 are clipped to 0.}
    \vspace{-7pt}
  \label{fig:neurons_features_ideal_atoms_radar_chart}
\end{figure}

\vspace{-2pt}
\subsection{Neurons or Features as Ideal Atoms?}\label{section:Experiments_Neurons_or_Features_as_Ideal_Atoms}
\vspace{-2pt}

In this section, we evaluate whether commonly used representational units, including neurons and features, satisfy the criteria of ideal atoms, i.e., faithfulness and stability.

\vspace{-8pt}
\paragraph{Experimental Setup}

We use all 20K subject entities from the CounterFact dataset \cite{meng2022locating}, a subset of WikiData, and extract the corresponding neurons and features (details in Appendix~\ref{appendix:Atoms_of_LLMs_Baselines_Details}) activated on these entities across all layers of Gemma2-2B, Gemma2-9B, and Llama3.1-8B as the target representational units.\vspace{-1pt}

Under Atom Theory, ideal atoms are evaluated along two dimensions. \textbf{Faithfulness} measures how accurately representational units reconstruct the original representations and is quantified by the coefficient of determination \(R^2 := 1 - \frac{\sum_i \|\bm x_i - \hat{\bm x}_i\|^2}{\sum_i \|\bm x_i - \bar{\bm x}\|^2}\), where \(\hat{\bm x}_i\) and \(\bar{\bm x}\) denote the predicted representation and sample mean, respectively. \textbf{Stability} captures structural robustness and is quantified by the maximal quantile \(q^* := \sup\left\{ q \mid \mu_q < \frac{1}{2K_q - 1}\right\}\), where \(\mu_q\) and \(K_q\) denote quantile coherence and quantile sparsity. Ideal atoms satisfy \(R^2 = 1\) and \(q^* = 1\). Full definitions and evaluation details are provided in Appendix~\ref{appendix:Atoms_of_LLMs_Evaluation_Details}. \vspace{-1pt}

\vspace{-8pt}
\paragraph{Experimental Results}
As shown in Fig.~\ref{fig:neurons_features_ideal_atoms_radar_chart}, evaluation is summarized by two radar plots for faithfulness (\(R^2\), upper panel) and stability (\(q^*\), lower panel), whose axes correspond to layers of LLMs. Ideal atoms should approach 1 on each axis in both plots, thereby filling the unit circle in each plot. However, neurons, as the basic computational units of neural networks, exhibit perfect faithfulness (\(\overline{R^2}\! =\! 1\)) but extremely low stability (\(\overline{q^*}\! =\! 0.5\%\)), where \(\overline{\cdot}\) denotes average over layers (and models). Features improve stability (\(\overline{q^*}\! =\! 68.2\%\)) relative to neurons, yet remain unstable and display low faithfulness (\(\overline{R^2}\! = \!48.8\%\)). Theoretically, stability is necessary for monosemanticity ($\S$~\ref{section:Atom_Theory_Evaluation_of_Atoms}). Accordingly, the low stability of neurons implies polysemanticity, consistent with prior findings \citep{olah2020zoom}. Features exhibit higher stability and thus improved monosemanticity \citep{chen2025knowledge}, but still fall short of ideal atoms. Overall, both neurons and features exhibit a clear gap from ideal atoms.\vspace{-1pt}

\begin{figure}
    \centering
    \includegraphics[width=0.48\textwidth]{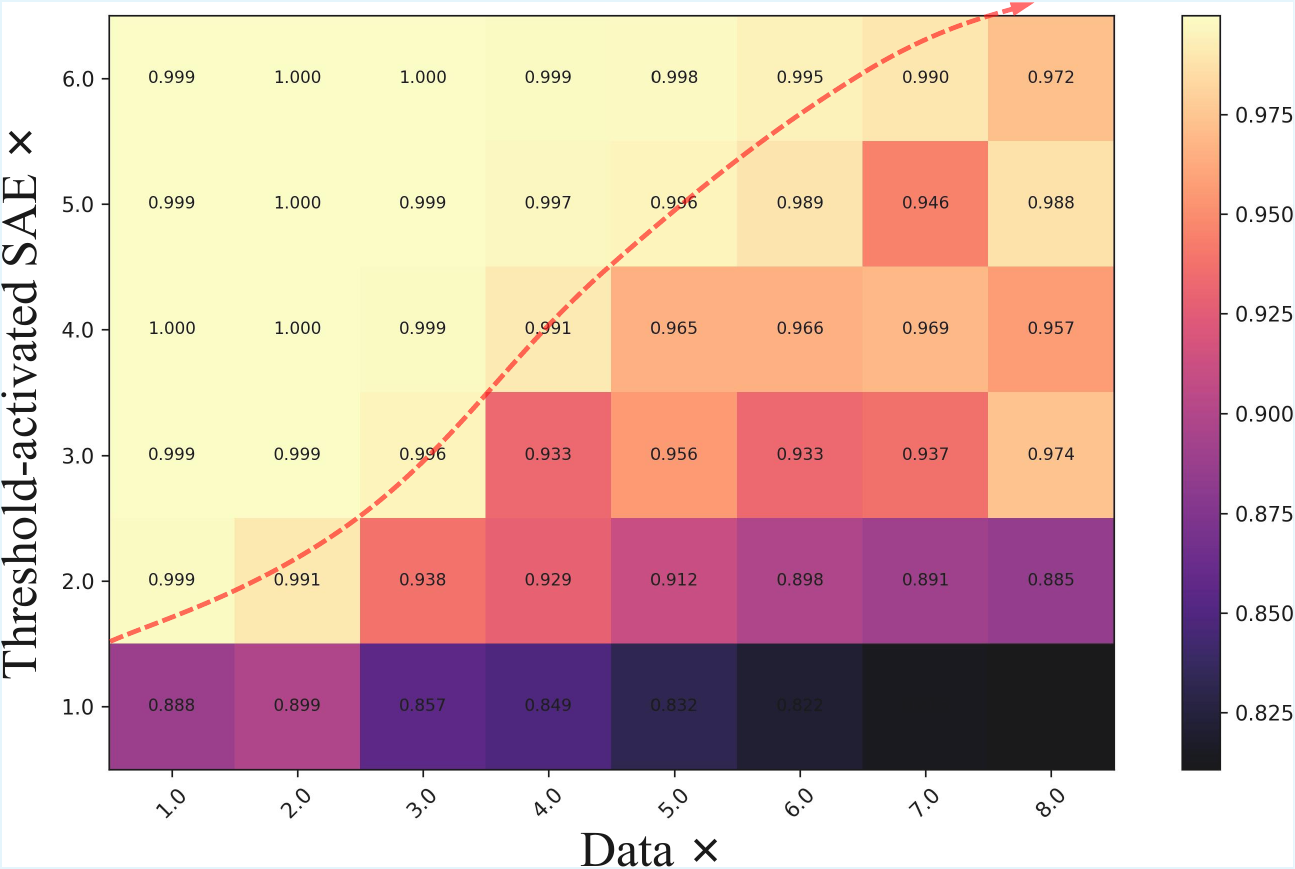}
    \caption{Matching TSAE capacity and data scale on Gemma2-2B (measured by \(R^2\)). Data \(\times\) and TSAE \(\times\) denote data scale and model capacity (interval 9,216). Red dashed lines mark the capacity range enabling reliable atom identification.}
    \vspace{-6pt}
  \label{fig:Scaling_Saes}
\end{figure}

\vspace{-2pt}
\subsection{TSAE Capacity Meet Data Scale}\label{section:Experiments_TSAE_Capacity_meet_Data_Scale}
\vspace{-2pt}

Next, we address the practical problem of identifying the FRUs of LLMs. Theorem~\ref{theorem:Identifiability_of_SAEs} shows that TSAEs are in principle capable of identifying atoms, but it does not specify how to choose the model capacity given the data scale in practice. In particular, we must determine the relationship between the data scale \(|M|\) and the model capacity \(|D|\) of TSAEs in order to reliably identify atoms.\vspace{-1pt}

\vspace{-8pt}
\paragraph{Experimental Setup}

We use subject entities from WikiData and extract the corresponding activations from the first layer of Gemma2-2B, yielding a sufficiently large representation dataset (1.9B samples), which is randomly sampled during training (details in Appendix~\ref{appendix:Atoms_of_LLMs}). Using an interval of 9,216, we evaluate \(R^2\) obtained by training TSAEs across varying data scales and model capacities. Due to the high computational cost of this grid search, experiments are conducted exclusively on Gemma2-2B. As implied by the feasible interval condition in Theorem~\ref{theorem:Identifiability_of_SAEs}, faithfulness is achievable only when stability holds. For brevity, we report faithfulness here, and present extensive empirical results on stability of the identified atoms in $\S$~\ref{section:Experiments_Atoms_of_LLMs}, confirming that the representations indeed exhibit a stable atomic structure.\vspace{-1pt}

\vspace{-8pt}
\paragraph{Experimental Results}

As shown in Fig.~\ref{fig:Scaling_Saes}, we find that high faithfulness occurs only when TSAE capacity exceeds a critical threshold for a given data scale. Intuitively, the data scale determines the scale of FRUs, which in turn determines required TSAE capacity for their identification. This finding also prompts a rethink of current SAE training paradigms, which heuristically choose model capacity and then train on massive activations from large corpora. While broadly applicable across tasks, their limited faithfulness ultimately constrains downstream reliability.\vspace{-1pt}

\vspace{-2pt}
\subsection{Atoms of LLMs}\label{section:Experiments_Atoms_of_LLMs}
\vspace{-2pt}
In this section, we identify widespread atoms of LLMs that exhibit high faithfulness and stability. Further analysis confirms that these atoms align with theoretical expectations and exhibit substantially stronger monosemanticity.\vspace{-1pt}

\vspace{-8pt}
\paragraph{Experimental Setup}

Consistent with $\S$~\ref{section:Experiments_Neurons_or_Features_as_Ideal_Atoms}, we use all subject entities from CounterFact and extract the corresponding activations across all layers of Gemma2-2B, Gemma2-9B, and Llama3.1-8B. Guided by the insights from $\S$~\ref{section:Experiments_TSAE_Capacity_meet_Data_Scale}, we uniformly adopt a \(4\times\) TSAE with JumpReLU activation \citep{erichson2019jumprelu, rajamanoharan2024jumping} to identify atoms. We evaluate the identified units with faithfulness (\(R^2\)) and stability (\(q^*\)). Full details on data, training, and computational costs are provided in Appendix~\ref{appendix:Atoms_of_LLMs}.\vspace{-1pt}

\vspace{-8pt}
\paragraph{Experimental Results}

\begin{table}[t]
\centering
\caption{Faithfulness and stability of identified units across models. Reported values are averaged over all layers. More detailed and supplementary results are provided in Appendix~\ref{appendix:Atoms_of_LLMs_Evaluation_Details} (Tabs.~\ref{tab:atom_quantiles_Gemma2-2B}–\ref{tab:atom_quantiles_Llama3.1-8B}).}
\label{tab:atom_faithfulness_stability}
\begin{tabular}{lcc}
\toprule
\textbf{Model} & \textit{Faithfulness} (\(R^2\)) & \textit{Stability} (\(q^*\)) \\
\midrule
\textsc{Gemma2-2B}   & 99.92\(\%\)  & 99.74\(\%\)  \\
\textsc{Gemma2-9B}   & 99.93\(\%\)  & 99.87\(\%\)  \\
\textsc{Llama3.1-8B} & 99.85\(\%\)  & 99.95\(\%\)  \\
\bottomrule
\end{tabular}
\vspace{-6pt}
\end{table}

Across Gemma2-2B (layers 1–26), Gemma2-9B (layers 1–42), and Llama3.1-8B (layers 1–30), we consistently achieve high faithfulness (\(\overline{R^2}\!=\!99.90\%\)) and stability (\(\overline{q^*}\!=\!99.85\%\)), as shown in Tab.~\ref{tab:atom_faithfulness_stability}. These results indicate that the identified units approach ideal atoms statistically. Further analysis (see Appendix~\ref{appendix:Atoms_of_LLMs_Experimental_Analysis}) reveals that the training process is minimally sensitive to hyperparameters (Fig.~\ref{fig:Robust_Train_Loss}); the encoder and decoder, when randomly initialized and trained independently without weight tying or additional constraints, converge to structures consistent with Theorem~\ref{theorem:Identifiability_of_SAEs} (Figs.~\ref{fig:Encoder_After_Training}-\ref{fig:Encoder_After_Training_Llama3.1-8B}); the identified atoms are approximately orthogonal under AIP, approaching the theoretical Dirac distribution (Figs.~\ref{fig:Naip_Distribution_Gemma2-2B_Full}-\ref{fig:Naip_Distribution_Llama3.1-8B_Full}).\vspace{-3.5pt}

Although stability provides a necessary structural foundation for monosemanticity, we validate this empirically. Specifically, we uniformly sample representational units across layers of models, evaluate monosemanticity using LLM-as-a-Judge with manual verification (see Appendix~\ref{appendix:Atoms_of_LLMs_Monosemanticity_Evaluation} for details and case studies). As shown in Fig.~\ref{fig:monosemanticity_score_across_models}, atoms with high faithfulness and stability consistently exhibit stronger monosemanticity.\vspace{-3.5pt}

\begin{figure}
    \centering
    \includegraphics[width=0.46\textwidth]{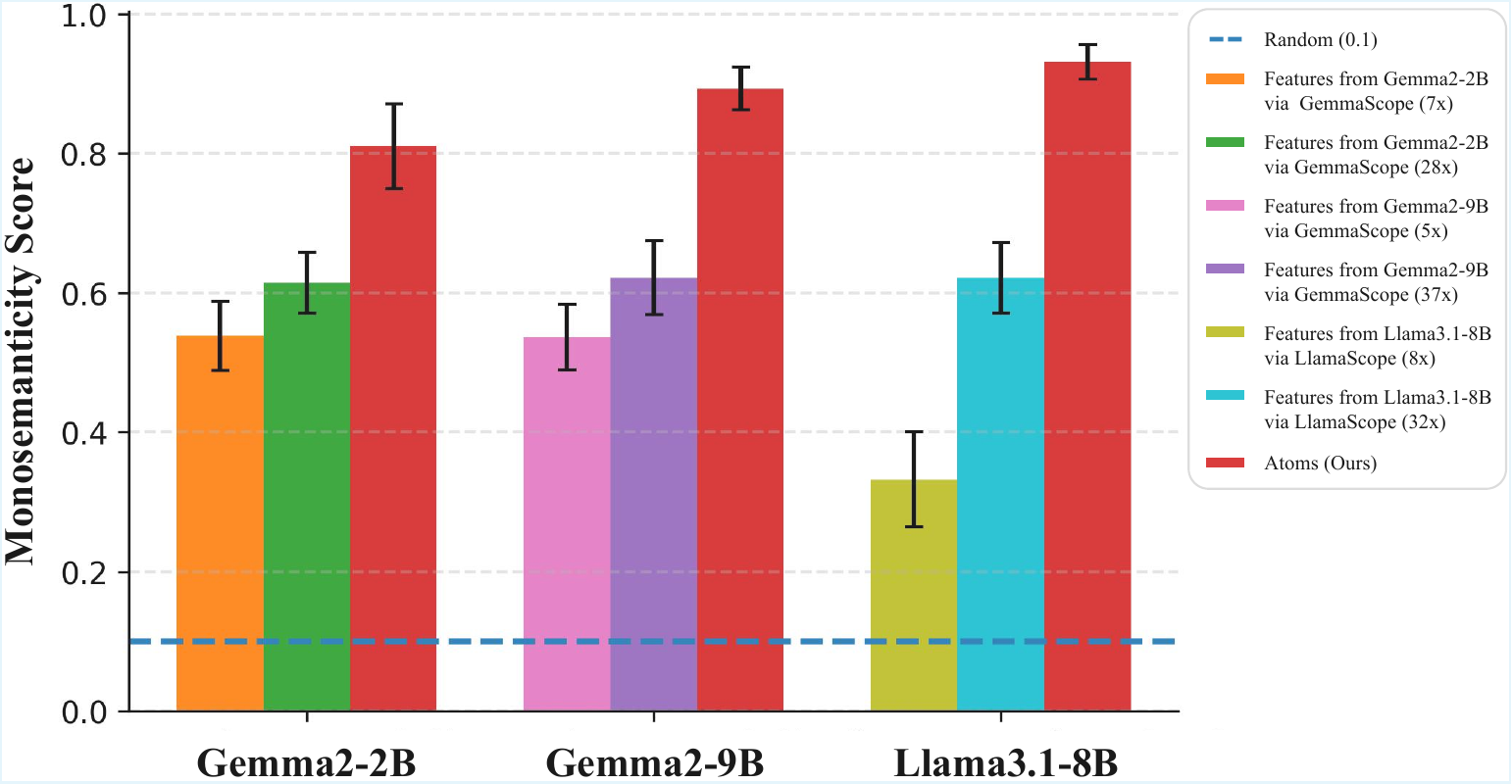}
    \caption{Monosemanticity scores of representational units across models, using GPT-5.2 with manual verification. The blue dashed line indicates random-guess performance (0.1).}
    \vspace{-8pt}
  \label{fig:monosemanticity_score_across_models}
\end{figure}

These findings reveal that LLMs contain FRUs and offer a new perspective on their internal representations.

\vspace{-4pt}
\section{Related Work}
\vspace{-2pt}

\textbf{Neurons.} Early interpretability studies regarded neurons, the minimal computational units of neural networks, as the basic units of analysis. Subsequent work sought to ascribe functional interpretations to individual neurons \citep{bills2023language, geva2020transformer}. However, neurons face the polysemantic problem, activating for multiple semantically unrelated patterns \citep{olah2020zoom}, a phenomenon attributed to superposition \citep{elhage2022toy}. These findings indicate that neurons are unsuited as such units, motivating a shift from neurons to features \citep{olah2020zoom}.

\textbf{Features.} Although features initially lacked a unified formal definition \citep{elhage2022toy}, they are commonly understood as linear directions with specific meaning \citep{hewitt2019structural, park2023linear, gurnee2023finding, chen2025knowledge}. Sparse autoencoders (SAEs) were introduced to learn such features \citep{cunningham2023sparse} and subsequently scaled to larger settings \citep{gao2024scaling, templeton2024scaling}. \citet{rajamanoharan2024improving} and \citet{rajamanoharan2024jumping} optimized the architecture to mitigating feature shrinkage \citep{wright2024addressing}. Despite widespread adoption \citep{lieberum2024gemma, he2024llama}, existing SAEs remain limited by incomplete reconstruction, with the unreconstructed component termed "dark matter" \citep{engels2024decomposing}, and instability from feature splitting and merging (\citealp{bussmann2025learning}; \citealp{chanin2025feature}), undermining their suitability as FRUs. We therefore propose atoms as FRUs and develop Atom Theory.

\textbf{Other related perspectives.} Our work focuses on neurons and features as two major interpretability paradigms, comparing and unifying them under the criterion of FRUs. Other interpretability studies consider higher-level objects such as attention heads \citep{olsson2022context,geva2023dissecting}, circuits \citep{meng2022locating,yao2024knowledge}, and localized model components \citep{chen2025knowledge,hu2024wilke}. Moreover, broader work has studied AI systems from the perspective of compressed sensing \citep{feng2014compressive,giryes2016deep,papyan2017convolutional,bank2018etf,liang2020image,wang2022overparameterized,li2023provable}, offering a related perspective distinct from our focus on representational units.

\section{Discussion}

\paragraph{Applicability of Atom Theory} 

Although the empirical analysis of Atom Theory in this paper primarily focuses on LLMs, the theory itself is not inherently restricted to LLMs and may be extended to other types of neural networks, such as diffusion models. Specifically, the theory is grounded in representational distinguishability. We regard distinguishability as a first principle for analyzing neural representations: if representations encoding different information cannot be distinguished, then the collapsed representations cannot effectively carry or express the corresponding information. Moreover, further analysis of Eq.~\ref{equation:Reparameterization} suggests that the motivation for introducing the AIP lies in the gauge symmetry of representation space induced by matrix multiplication, where different coordinate parameterizations can represent the same function. This implies that the coordinate parameterization of representations is not unique, and the AIP is introduced to characterize this representational geometry. Since matrix multiplication is fundamental and ubiquitous in neural networks, such gauge symmetry can be expected to arise broadly across diverse neural architectures. This is consistent with our empirical observations on language models, as well as preliminary observations on diffusion models that are not included as formal results in this paper. Therefore, the AIP is not merely applicable to language models, but can serve as a more general tool for analyzing the representational geometry of neural networks. In summary, if neural representations encoding different information are expected to be distinguishable, Atom Theory has broader potential applicability beyond language models.

\paragraph{Training Dynamics and Model Selection}

Theoretically, ideal atoms correspond to representational structures that achieve both high faithfulness and stability. In practice, due to limitations in optimization and data, this ideal state cannot be attained exactly, but can be approached. Consequently, candidate solutions on the Pareto frontier are typically close to the ideal regime. In this regime, the differences among candidates are usually small, and the resulting representational structures are functionally similar. Thus, model selection is not arbitrary or subjective, but rather resembles choosing from an approximate equivalence class of solutions that satisfy the theoretical criteria. This differs substantially from sparse autoencoders trained on activations from large-scale corpora. The Pareto frontiers obtained by such sparse autoencoders are often farther from the ideal regime and typically exhibit a stronger trade-off. We conjecture that this phenomenon is related to the training objective: fitting large-scale corpus activations with a finite-capacity sparse autoencoder often allows the model to capture only the most dominant subset of representational structures. Accordingly, the theoretical characterization of this training paradigm remains an open question.

\vspace{-6pt}
\paragraph{Extended Analysis and Applications}

The primary goal of this paper is to establish a theoretical framework for characterizing, evaluating, and identifying FRUs of LLMs. Accordingly, we focus on three aspects: (i) definitions and theoretical characterization; (ii) the development of evaluation metrics; and (iii) systematic validation across models, rather than optimization for specific downstream applications. Prior work has shown that representation analysis can support interpretability, intervention, behavior localization, and model safety. However, existing analyses often fall short of systematically characterizing how internal representations support complex behaviors in a structured manner. In contrast, FRUs provide a systematic view for understanding neural representations from the perspective of functional organization. Future work may construct large-scale representation repositories to systematically study the mechanisms underlying edge cases, anomalous behaviors, and safety-sensitive scenarios. From this perspective, developing scalable methods to incrementally construct such repositories remains an important direction.

\section{Conclusion}
\vspace{-2pt}

This paper introduces and validates Atom Theory for characterizing the fundamental representational units (FRUs) of large language models (LLMs). We formalize atoms, establish criteria for ideal atoms, and prove their identifiability via threshold-activated sparse autoencoders (TSAEs). Empirically, we show that LLMs widely contain FRUs with high faithfulness and stability, exhibiting strong monosemanticity. Overall, these results elevate interpretability from heuristic analysis to a principled theory of FRUs in LLMs.

\section*{Acknowledgements}

This work was supported by the National Natural Science Foundation of China (No.U24A20335, No.62406321), the independent research project of the Key Laboratory of Cognition and Decision Intelligence for Complex Systems, and CIPS-SMP-Zhipu Large Model Fund.

\section*{Impact Statement}

This work contributes to the interpretability of large language models by introducing Atom Theory to identify fundamental representational units. Improved understanding of internal representations may enhance transparency, safety, and reliability of AI systems, supporting model auditing and alignment. We expect the net impact to be positive.


\bibliography{example_paper}

@article{park2023linear,
  title={The linear representation hypothesis and the geometry of large language models},
  author={Park, Kiho and Choe, Yo Joong and Veitch, Victor},
  journal={arXiv preprint arXiv:2311.03658},
  year={2023}
}

@inproceedings{hu2025knowledge,
  title={Knowledge in superposition: Unveiling the failures of lifelong knowledge editing for large language models},
  author={Hu, Chenhui and Cao, Pengfei and Chen, Yubo and Liu, Kang and Zhao, Jun},
  booktitle={Proceedings of the AAAI Conference on Artificial Intelligence},
  volume={39},
  pages={24086--24094},
  year={2025}
}

@article{o2024compute,
  title={Compute optimal inference and provable amortisation gap in sparse autoencoders},
  author={O'Neill, Charles and Gumran, Alim and Klindt, David},
  journal={arXiv preprint arXiv:2411.13117},
  year={2024}
}

@article{geva2020transformer,
  title={Transformer feed-forward layers are key-value memories},
  author={Geva, Mor and Schuster, Roei and Berant, Jonathan and Levy, Omer},
  journal={arXiv preprint arXiv:2012.14913},
  year={2020}
}

@article{meng2022locating,
  title={Locating and editing factual associations in gpt},
  author={Meng, Kevin and Bau, David and Andonian, Alex and Belinkov, Yonatan},
  journal={Advances in neural information processing systems},
  volume={35},
  pages={17359--17372},
  year={2022}
}

@misc{bills2023language,
   title={Language models can explain neurons in language models},
   author={Bills, Steven and Cammarata, Nick and Mossing, Dan and Tillman, Henk and Gao, Leo and Goh, Gabriel and Sutskever, Ilya and Leike, Jan and Wu, Jeff and Saunders, William},
   year={2023},
   howpublished = {\url{https://openaipublic.blob.core.windows.net/neuron-explainer/paper/index.html}}
}

@article{olah2020zoom,
  title={Zoom in: An introduction to circuits},
  author={Olah, Chris and Cammarata, Nick and Schubert, Ludwig and Goh, Gabriel and Petrov, Michael and Carter, Shan},
  journal={Distill},
  volume={5},
  number={3},
  pages={e00024--001},
  year={2020}
}

@article{elhage2022toy,
  title={Toy models of superposition},
  author={Elhage, Nelson and Hume, Tristan and Olsson, Catherine and Schiefer, Nicholas and Henighan, Tom and Kravec, Shauna and Hatfield-Dodds, Zac and Lasenby, Robert and Drain, Dawn and Chen, Carol and others},
  journal={arXiv preprint arXiv:2209.10652},
  year={2022}
}

@inproceedings{hewitt2019structural,
  title={A structural probe for finding syntax in word representations},
  author={Hewitt, John and Manning, Christopher D},
  booktitle={Proceedings of the 2019 Conference of the North American Chapter of the Association for Computational Linguistics: Human Language Technologies, Volume 1 (Long and Short Papers)},
  pages={4129--4138},
  year={2019}
}

@article{gurnee2023finding,
  title={Finding neurons in a haystack: Case studies with sparse probing},
  author={Gurnee, Wes and Nanda, Neel and Pauly, Matthew and Harvey, Katherine and Troitskii, Dmitrii and Bertsimas, Dimitris},
  journal={arXiv preprint arXiv:2305.01610},
  year={2023}
}

@article{cunningham2023sparse,
  title={Sparse autoencoders find highly interpretable features in language models},
  author={Cunningham, Hoagy and Ewart, Aidan and Riggs, Logan and Huben, Robert and Sharkey, Lee},
  journal={arXiv preprint arXiv:2309.08600},
  year={2023}
}

@article{gao2024scaling,
  title={Scaling and evaluating sparse autoencoders},
  author={Gao, Leo and la Tour, Tom Dupr{\'e} and Tillman, Henk and Goh, Gabriel and Troll, Rajan and Radford, Alec and Sutskever, Ilya and Leike, Jan and Wu, Jeffrey},
  journal={arXiv preprint arXiv:2406.04093},
  year={2024}
}

@book{templeton2024scaling,
  title={Scaling monosemanticity: Extracting interpretable features from claude 3 sonnet},
  author={Templeton, Adly},
  year={2024},
  publisher={Anthropic}
}

@article{rajamanoharan2024improving,
  title={Improving dictionary learning with gated sparse autoencoders},
  author={Rajamanoharan, Senthooran and Conmy, Arthur and Smith, Lewis and Lieberum, Tom and Varma, Vikrant and Kram{\'a}r, J{\'a}nos and Shah, Rohin and Nanda, Neel},
  journal={arXiv preprint arXiv:2404.16014},
  year={2024}
}

@article{rajamanoharan2024jumping,
  title={Jumping ahead: Improving reconstruction fidelity with jumprelu sparse autoencoders},
  author={Rajamanoharan, Senthooran and Lieberum, Tom and Sonnerat, Nicolas and Conmy, Arthur and Varma, Vikrant and Kram{\'a}r, J{\'a}nos and Nanda, Neel},
  journal={arXiv preprint arXiv:2407.14435},
  year={2024}
}

@article{lieberum2024gemma,
  title={Gemma scope: Open sparse autoencoders everywhere all at once on gemma 2},
  author={Lieberum, Tom and Rajamanoharan, Senthooran and Conmy, Arthur and Smith, Lewis and Sonnerat, Nicolas and Varma, Vikrant and Kram{\'a}r, J{\'a}nos and Dragan, Anca and Shah, Rohin and Nanda, Neel},
  journal={arXiv preprint arXiv:2408.05147},
  year={2024}
}

@article{he2024llama,
  title={Llama scope: Extracting millions of features from llama-3.1-8b with sparse autoencoders},
  author={He, Zhengfu and Shu, Wentao and Ge, Xuyang and Chen, Lingjie and Wang, Junxuan and Zhou, Yunhua and Liu, Frances and Guo, Qipeng and Huang, Xuanjing and Wu, Zuxuan and others},
  journal={arXiv preprint arXiv:2410.20526},
  year={2024}
}

@article{team2024gemma,
  title={Gemma 2: Improving open language models at a practical size},
  author={Team, Gemma and Riviere, Morgane and Pathak, Shreya and Sessa, Pier Giuseppe and Hardin, Cassidy and Bhupatiraju, Surya and Hussenot, L{\'e}onard and Mesnard, Thomas and Shahriari, Bobak and Ram{\'e}, Alexandre and others},
  journal={arXiv preprint arXiv:2408.00118},
  year={2024}
}

@article{dubey2024llama,
  title={The llama 3 herd of models},
  author={Dubey, Abhimanyu and Jauhri, Abhinav and Pandey, Abhinav and Kadian, Abhishek and Al-Dahle, Ahmad and Letman, Aiesha and Mathur, Akhil and Schelten, Alan and Yang, Amy and Fan, Angela and others},
  journal={arXiv e-prints},
  pages={arXiv--2407},
  year={2024}
}

@article{engels2024decomposing,
  title={Decomposing the dark matter of sparse autoencoders},
  author={Engels, Joshua and Riggs, Logan and Tegmark, Max},
  journal={arXiv preprint arXiv:2410.14670},
  year={2024}
}

@article{bussmann2025learning,
  title={Learning multi-level features with matryoshka sparse autoencoders},
  author={Bussmann, Bart and Nabeshima, Noa and Karvonen, Adam and Nanda, Neel},
  journal={arXiv preprint arXiv:2503.17547},
  year={2025}
}

@article{chanin2025feature,
  title={Feature Hedging: Correlated Features Break Narrow Sparse Autoencoders},
  author={Chanin, David and Dulka, Tom{\'a}{\v{s}} and Garriga-Alonso, Adri{\`a}},
  journal={arXiv preprint arXiv:2505.11756},
  year={2025}
}

@article{bricken2023monosemanticity,
   title={Towards Monosemanticity: Decomposing Language Models With Dictionary Learning},
   author={Bricken, Trenton and Templeton, Adly and Batson, Joshua and Chen, Brian and Jermyn, Adam and Conerly, Tom and Turner, Nick and Anil, Cem and Denison, Carson and Askell, Amanda and Lasenby, Robert and Wu, Yifan and Kravec, Shauna and Schiefer, Nicholas and Maxwell, Tim and Joseph, Nicholas and Hatfield-Dodds, Zac and Tamkin, Alex and Nguyen, Karina and McLean, Brayden and Burke, Josiah E and Hume, Tristan and Carter, Shan and Henighan, Tom and Olah, Christopher},
   year={2023},
   journal={Transformer Circuits Thread},
   note={https://transformer-circuits.pub/2023/monosemantic-features/index.html}
}

@article{petroni2019language,
  title={Language models as knowledge bases?},
  author={Petroni, Fabio and Rockt{\"a}schel, Tim and Lewis, Patrick and Bakhtin, Anton and Wu, Yuxiang and Miller, Alexander H and Riedel, Sebastian},
  journal={arXiv preprint arXiv:1909.01066},
  year={2019}
}

@article{brown2020language,
  title={Language models are few-shot learners},
  author={Brown, Tom and Mann, Benjamin and Ryder, Nick and Subbiah, Melanie and Kaplan, Jared D and Dhariwal, Prafulla and Neelakantan, Arvind and Shyam, Pranav and Sastry, Girish and Askell, Amanda and others},
  journal={Advances in neural information processing systems},
  volume={33},
  pages={1877--1901},
  year={2020}
}

@article{nanda2023emergent,
  title={Emergent linear representations in world models of self-supervised sequence models},
  author={Nanda, Neel and Lee, Andrew and Wattenberg, Martin},
  journal={arXiv preprint arXiv:2309.00941},
  year={2023}
}

@article{achiam2023gpt,
  title={Gpt-4 technical report},
  author={Achiam, Josh and Adler, Steven and Agarwal, Sandhini and Ahmad, Lama and Akkaya, Ilge and Aleman, Florencia Leoni and Almeida, Diogo and Altenschmidt, Janko and Altman, Sam and Anadkat, Shyamal and others},
  journal={arXiv preprint arXiv:2303.08774},
  year={2023}
}

@article{olah2017feature,
  title={Feature visualization},
  author={Olah, Chris and Mordvintsev, Alexander and Schubert, Ludwig},
  journal={Distill},
  volume={2},
  number={11},
  pages={e7},
  year={2017}
}

@article{donoho2006compressed,
  title={Compressed sensing},
  author={Donoho, David L},
  journal={IEEE Transactions on information theory},
  volume={52},
  number={4},
  pages={1289--1306},
  year={2006},
  publisher={IEEE}
}

@article{candes2006robust,
  title={Robust uncertainty principles: Exact signal reconstruction from highly incomplete frequency information},
  author={Cand{\`e}s, Emmanuel J and Romberg, Justin and Tao, Terence},
  journal={IEEE Transactions on information theory},
  volume={52},
  number={2},
  pages={489--509},
  year={2006},
  publisher={IEEE}
}

@article{vrandevcic2014wikidata,
  title={Wikidata: a free collaborative knowledgebase},
  author={Vrande{\v{c}}i{\'c}, Denny and Kr{\"o}tzsch, Markus},
  journal={Communications of the ACM},
  volume={57},
  number={10},
  pages={78--85},
  year={2014},
  publisher={ACM New York, NY, USA}
}

@inproceedings{wright2024addressing,
  title={Addressing feature suppression in saes},
  author={Wright, Benjamin and Sharkey, Lee},
  booktitle={AI Alignment Forum},
  volume={6},
  year={2024}
}

@article{chen2025knowledge,
  title={The knowledge microscope: Features as better analytical lenses than neurons},
  author={Chen, Yuheng and Cao, Pengfei and Liu, Kang and Zhao, Jun},
  journal={arXiv preprint arXiv:2502.12483},
  year={2025}
}

@article{radford2019language,
  title={Language models are unsupervised multitask learners},
  author={Radford, Alec and Wu, Jeffrey and Child, Rewon and Luan, David and Amodei, Dario and Sutskever, Ilya and others},
  journal={OpenAI blog},
  volume={1},
  number={8},
  pages={9},
  year={2019}
}

@article{wang2022gpt,
  title={GPT-J-6B: a 6 billion parameter autoregressive language model (2021)},
  author={Wang, Ben and Komatsuzaki, Aran},
  journal={URL https://github. com/kingoflolz/mesh-transformer-jax},
  year={2022}
}

@inproceedings{biderman2023pythia,
  title={Pythia: A suite for analyzing large language models across training and scaling},
  author={Biderman, Stella and Schoelkopf, Hailey and Anthony, Quentin Gregory and Bradley, Herbie and O’Brien, Kyle and Hallahan, Eric and Khan, Mohammad Aflah and Purohit, Shivanshu and Prashanth, USVSN Sai and Raff, Edward and others},
  booktitle={International Conference on Machine Learning},
  pages={2397--2430},
  year={2023},
  organization={PMLR}
}

@article{touvron2023llama,
  title={Llama 2: Open foundation and fine-tuned chat models},
  author={Touvron, Hugo and Martin, Louis and Stone, Kevin and Albert, Peter and Almahairi, Amjad and Babaei, Yasmine and Bashlykov, Nikolay and Batra, Soumya and Bhargava, Prajjwal and Bhosale, Shruti and others},
  journal={arXiv preprint arXiv:2307.09288},
  year={2023}
}

@article{erichson2019jumprelu,
  title={Jumprelu: A retrofit defense strategy for adversarial attacks},
  author={Erichson, N Benjamin and Yao, Zhewei and Mahoney, Michael W},
  journal={arXiv preprint arXiv:1904.03750},
  year={2019}
}

@article{makhzani2013k,
  title={K-sparse autoencoders},
  author={Makhzani, Alireza and Frey, Brendan},
  journal={arXiv preprint arXiv:1312.5663},
  year={2013}
}

@article{in2001wikipedia,
  title={Wikipedia: the free encyclopedia},
  author={In, SWOT},
  journal={San Francisco (CA): Wikimedia Foundation},
  volume={2012},
  number={16.4},
  pages={2012},
  year={2001}
}

@article{hendrycks2021measuring,
  title={Measuring mathematical problem solving with the math dataset},
  author={Hendrycks, Dan and Burns, Collin and Kadavath, Saurav and Arora, Akul and Basart, Steven and Tang, Eric and Song, Dawn and Steinhardt, Jacob},
  journal={arXiv preprint arXiv:2103.03874},
  year={2021}
}

@inproceedings{dai2022knowledge,
  title={Knowledge neurons in pretrained transformers},
  author={Dai, Damai and Dong, Li and Hao, Yaru and Sui, Zhifang and Chang, Baobao and Wei, Furu},
  booktitle={Proceedings of the 60th Annual Meeting of the Association for Computational Linguistics (Volume 1: Long Papers)},
  pages={8493--8502},
  year={2022}
}

@inproceedings{chen2024journey,
  title={Journey to the center of the knowledge neurons: Discoveries of language-independent knowledge neurons and degenerate knowledge neurons},
  author={Chen, Yuheng and Cao, Pengfei and Chen, Yubo and Liu, Kang and Zhao, Jun},
  booktitle={Proceedings of the AAAI Conference on Artificial Intelligence},
  volume={38},
  pages={17817--17825},
  year={2024}
}

@incollection{foucart2013invitation,
  title={An invitation to compressive sensing},
  author={Foucart, Simon and Rauhut, Holger},
  booktitle={A mathematical introduction to compressive sensing},
  pages={1--39},
  year={2013},
  publisher={Springer}
}

@inproceedings{murdock2020dataless,
  title={Dataless model selection with the deep frame potential},
  author={Murdock, Calvin and Lucey, Simon},
  booktitle={Proceedings of the IEEE/CVF Conference on Computer Vision and Pattern Recognition},
  pages={11257--11265},
  year={2020}
}

@article{goessmann2020restricted,
  title={The restricted isometry of relu networks: Generalization through norm concentration},
  author={Goe{\ss}mann, Alex and Kutyniok, Gitta},
  journal={arXiv preprint arXiv:2007.00479},
  year={2020}
}

@article{chen2021conditioning,
  title={Conditioning of random feature matrices: Double descent and generalization error},
  author={Chen, Zhijun and Schaeffer, Hayden},
  journal={arXiv preprint arXiv:2110.11477},
  year={2021}
}

@article{candes2005decoding,
  title={Decoding by linear programming},
  author={Candes, Emmanuel J and Tao, Terence},
  journal={IEEE transactions on information theory},
  volume={51},
  number={12},
  pages={4203--4215},
  year={2005},
  publisher={IEEE}
}

@article{donoho2001uncertainty,
  title={Uncertainty principles and ideal atomic decomposition},
  author={Donoho, David L and Huo, Xiaoming and others},
  journal={IEEE transactions on information theory},
  volume={47},
  number={7},
  pages={2845--2862},
  year={2001}
}

@inproceedings{geva2023dissecting,
  title={Dissecting recall of factual associations in auto-regressive language models},
  author={Geva, Mor and Bastings, Jasmijn and Filippova, Katja and Globerson, Amir},
  booktitle={Proceedings of the 2023 Conference on Empirical Methods in Natural Language Processing},
  pages={12216--12235},
  year={2023}
}

@article{olsson2022context,
  title={In-context learning and induction heads},
  author={Olsson, Catherine and Elhage, Nelson and Nanda, Neel and Joseph, Nicholas and DasSarma, Nova and Henighan, Tom and Mann, Ben and Askell, Amanda and Bai, Yuntao and Chen, Anna and others},
  journal={arXiv preprint arXiv:2209.11895},
  year={2022}
}

@inproceedings{hu2024wilke,
  title={Wilke: Wise-layer knowledge editor for lifelong knowledge editing},
  author={Hu, Chenhui and Cao, Pengfei and Chen, Yubo and Liu, Kang and Zhao, Jun},
  booktitle={Findings of the Association for Computational Linguistics: ACL 2024},
  pages={3476--3503},
  year={2024}
}

@article{yao2024knowledge,
  title={Knowledge circuits in pretrained transformers},
  author={Yao, Yunzhi and Zhang, Ningyu and Xi, Zekun and Wang, Mengru and Xu, Ziwen and Deng, Shumin and Chen, Huajun},
  journal={Advances in neural information processing systems},
  volume={37},
  pages={118571--118602},
  year={2024}
}

@article{giryes2016deep,
  title={Deep neural networks with random gaussian weights: A universal classification strategy?},
  author={Giryes, Raja and Sapiro, Guillermo and Bronstein, Alex M},
  journal={IEEE Transactions on Signal Processing},
  volume={64},
  number={13},
  pages={3444--3457},
  year={2016},
  publisher={IEEE}
}

@article{papyan2017convolutional,
  title={Convolutional neural networks analyzed via convolutional sparse coding},
  author={Papyan, Vardan and Romano, Yaniv and Elad, Michael},
  journal={Journal of Machine Learning Research},
  volume={18},
  number={83},
  pages={1--52},
  year={2017}
}

@article{bank2018etf,
  title={An ETF view of dropout regularization},
  author={Bank, Dor and Giryes, Raja},
  journal={arXiv preprint arXiv:1810.06049},
  year={2018}
}

@article{wang2022overparameterized,
  title={Overparameterized relu neural networks learn the simplest models: Neural isometry and exact recovery},
  author={Wang, Yifei and Hua, Yixuan and Cand{\'e}s, Emmanuel and Pilanci, Mert},
  journal={arXiv preprint arXiv:2209.15265},
  year={2022}
}

@article{li2023provable,
  title={Provable identifiability of two-layer relu neural networks via lasso regularization},
  author={Li, Gen and Wang, Ganghua and Ding, Jie},
  journal={IEEE Transactions on Information Theory},
  volume={69},
  number={9},
  pages={5921--5935},
  year={2023},
  publisher={IEEE}
}

@article{feng2014compressive,
  title={Compressive sensing inverse synthetic aperture radar imaging based on Gini index regularization},
  author={Feng, Can and Xiao, Liang and Wei, Zhi-Hui},
  journal={International Journal of Automation and Computing},
  volume={11},
  number={4},
  pages={441--448},
  year={2014},
  publisher={Springer}
}

@article{liang2020image,
  title={Image encryption algorithm based on compressive sensing and fractional DCT via polynomial interpolation},
  author={Liang, Ya-Ru and Xiao, Zhi-Yong},
  journal={International Journal of Automation and Computing},
  volume={17},
  number={2},
  pages={292--304},
  year={2020},
  publisher={Springer}
}
\bibliographystyle{icml2026}

\newpage
\appendix
\onecolumn
\section{Proofs}\label{appendix:Proofs}

\subsection{Explanation for Equation~\ref{equation:Reparameterization}}\label{appendix:Proofs_Explanation}

For
\begin{equation}
\tag{\ref{equation:Reparameterization}}
    {W_U}' \leftarrow A^{-\top} W_U + \bm b \, \bm 1^\top,\ \bm h'^{L} \leftarrow A \bm h^L,
\end{equation}
we provide a simple derivation as follows:
\begin{align}
    {W_U'}^\top \bm h' &= (A^{-\top} W_U + \bm b \, \bm 1^\top)^\top (A \bm h) \\&= W_U^\top (A^{-1} A) \bm h + \bm 1 (\bm b^\top A \bm h) \\&= W_U^\top \bm h + c(\bm h) \bm 1,
\end{align}
where \(c(\bm h) = \bm b^\top A \bm h \in \mathbb{R}\) is a scalar. Using the translation invariance property of softmax, \(\mathrm{Softmax}(\bm z + c \bm 1) = \mathrm{Softmax}(\bm z)\), the result follows.

It is important to note that this is the only form, meaning that \(W_U\) can only be identified up to an invertible transformation plus a bias, and \(\bm h\) can only be identified up to an invertible transformation.

\subsection{Proof of Theorem~\ref{theorem:Explicit_Form}}
\ExplicitForm*
\begin{proof}
    Since \(\langle \cdot, \cdot \rangle_S\) is atomic inner product, for any pair of atoms \(\bm d_i\) and \(\bm d_j\) we have
    \begin{equation}
        \bm d_i^{\top} S \bm d_j \;=\;\langle \bm d_i, \bm d_j \rangle_{S}=\begin{cases}0 & i \ne j, \\[4pt]c^2 & i = j .\end{cases}
    \end{equation}
    Applying this property to the atom set \(D=[\bm d_1,\bm d_2,\cdots,\bm d_{|D|}]\) yields
    \begin{equation}
        c^2 I = D^\top S D .
    \end{equation}
    Since \(\{\bm d_1,\bm d_2,\cdots,\bm d_{|D|}\}\) spans the atomic space \(\mathcal D\), and \(\mathcal D \simeq \mathbb R^H\), it follows that \(\mathrm{rank}(D)=H\). Let \(S^{1/2}\) denote the symmetric positive-definite square root of \(S\). Then
    \begin{equation}
        D^\top S D = (S^{1/2} D)^\top (S^{1/2} D),
    \end{equation}
    which implies \(\mathrm{rank}(I) = \mathrm{rank}(D^\top S D) = \mathrm{rank}(S^{1/2} D) = \mathrm{rank}(D)\). Therefore, \(\mathrm{rank}(D) = |D| \leq H\). Combining this with the earlier condition gives \(|D| = \mathrm{rank}(D) = H\), which shows that \(D\) is invertible. Consequently,
    \begin{equation}
        S = c^2 (D D^\top)^{-1}.
    \end{equation}
\end{proof}

\subsection{Proof of Corollary~\ref{corollary:NAIP}}
\NAIP*
\begin{proof}
    Since \(\langle \bm{d}_i, \bm{d}_j \rangle_S = 0\) for \(i \neq j\), and \(\|\bm{d}_i\|_S = \|\bm{d}_j\|_S = c > 0\), it follows that \(D^\top S D = c^2 I\). Therefore, \(\tilde{S} = \frac{1}{c^2} S\) satisfies \(D^\top \tilde{S} D = I\), which implies that the atoms are orthonormal. Since $D$ is invertible, we also have \(\tilde{S} = D^{-\top} D^{-1} = (D D^\top)^{-1}\), and thus \(\langle \cdot, \cdot \rangle_{\tilde{S}}\) is a symmetric positive-definite inner product.
\end{proof}

\subsection{Proof of Theorem~\ref{theorem:Coherence–RIP_Upper_Bound}}
\RIPUpperBound*
\begin{proof}
    Let \(\mathrm{supp}(\boldsymbol \delta) = \mathcal{S} \subseteq [|D|]\) and \(|\mathcal{S}| \leq K\). Then, we have
    \begin{equation}
        \|\tilde{D} \bm{\delta}\|_2^2 = \bm{\delta}^\top (\tilde{D}^\top \tilde{D}) \bm{\delta} = \sum_{i \in \mathcal{S}} \delta_i^2 + 2 \sum_{\substack{i < j \\ i, j \in \mathcal{S}}} \delta_i \delta_j \langle \tilde{\bm{d}_i}, \tilde{\bm{d}_j} \rangle.
    \end{equation}
    By the fact that \(|\langle \tilde{\bm{d}_i}, \tilde{\bm{d}_j} \rangle| \leq \mu\) and applying the triangle inequality, we obtain:
    \begin{align}
        \|\tilde{D} \bm{\delta}\|_2^2 & \geq \sum_{i \in \mathcal{S}} \delta_i^2 - 2 \mu \sum_{i < j} |\delta_i \delta_j|, \\
        \|\tilde{D} \bm{\delta}\|_2^2 & \leq \sum_{i \in \mathcal{S}} \delta_i^2 + 2 \mu \sum_{i < j} |\delta_i \delta_j|.
    \end{align}
    Next, we observe that:
    \begin{equation}
        \left( \sum_{i \in \mathcal{S}} |\delta_i| \right)^2 = \sum_{i \in \mathcal{S}} \delta_i^2 + 2 \sum_{i < j} |\delta_i \delta_j| \leq |\mathcal{S}| \sum_{i \in \mathcal{S}} \delta_i^2 \leq K \sum_{i \in \mathcal{S}} \delta_i^2.
    \end{equation}
    Thus, we have \(2 \sum_{i < j} |\delta_i \delta_j| \leq (K - 1) \sum_{i \in \mathcal{S}} \delta_i^2\). Substituting this back, we conclude the proof.
\end{proof}

\subsection{Proof of Theorem~\ref{theorem:Uniqueness_of_Sparse_Representation_and_Exact_Recovery}}
\UniquenessRecovery*
\begin{proof}
    We first prove that under the condition $\mu < \tfrac{1}{2K-1}$, the $K$-sparse representation is unique.

    Suppose there exist two distinct \(K\)-sparse coefficient vectors \(\bm \delta, \bm \delta'\) such that \(\tilde D \bm \delta = \tilde D \bm \delta'\). Let \(\bm h = \bm \delta - \bm \delta' \neq \bm 0\). Then \(\tilde D \bm h = 0\) and \(\|\bm h\|_0 \leq 2K\). By Theorem~\ref{theorem:Coherence–RIP_Upper_Bound} (applied with \(K\) replaced by \(2K\)), we have
    \begin{equation}
        \big(1 - (2K-1)\mu \big)\,\|\bm h\|_2^2\ \leq\ \|\tilde D \bm h\|_2^2\ =\ 0.
    \end{equation}
    If $\mu < \tfrac{1}{2K-1}$, then the prefactor on the left is strictly positive, which forces $\|\bm h\|_2=0$. This contradicts $\bm h \neq \bm 0$. Hence uniqueness holds.

    Next, we prove that under the same condition $\mu < \tfrac{1}{2K-1}$, the sparse vector $\bm \delta$ is also the unique solution of the convex optimization problem
    \begin{equation}
        \min_{\bm x \in \mathbb{R}^{|D|}} \|\bm x\|_1\quad \text{s.t.} \quad\tilde D \bm x = \tilde{\bm m}.
    \end{equation}
    The overall strategy is as follows: (i) show that the null space property of order $K$ (NSP$_K$) holds under the assumption $\mu < \frac{1}{2K-1}$; (ii) recall the equivalence NSP$_K$ $\iff$ exact and unique recovery of any $K$-sparse solution via noiseless $\ell_1$-minimization.

    Formally, the null space property of order $K$ (NSP$_K$) is defined as
    \begin{equation}
        \forall \bm h \in \ker(\tilde D) \setminus \{\bm 0\},\ \forall \mathcal S \subseteq [n],\ |\mathcal S| \le K:\quad\boxed{\,\|\bm h_\mathcal S\|_1 < \|\bm h_{\mathcal S^c}\|_1\,},
    \end{equation}
    where $\ker(\tilde D) = \{ \bm h : \tilde D \bm h = \bm0 \}$, $\mathcal S \subseteq [n]$ is an index set with $[n]=\{1,\dots,n\}$, $\bm h_\mathcal S$ denotes the restriction of $\bm h$ to the coordinates in $\mathcal S$ (with other entries set to zero), and \(\mathcal S^c=[n]\setminus \mathcal S\) is a complementary set.

    \paragraph{Step (i): Proof of NSP$_K$.} Let $G = \tilde D^\top \tilde D$ denote the Gram matrix. Since each column of $\tilde D$ is normalized, we have $G_{jj} = 1$ and $|G_{ij}| \leq \mu$ for $i \neq j$. Take any $\bm h \in \ker(\tilde D)\setminus\{\bm 0\}$ and any index set $\mathcal S$ with $|\mathcal S| = K$. 
    
    Since $\tilde D \bm h = 0$, we have $G \bm h = 0$. For any $j$,
    \begin{equation}
        0 = (G \bm h)_j = \sum_i G_{ji} h_i= G_{jj} h_j + \sum_{i \neq j} G_{ji} h_i\quad \Rightarrow \quad h_j = -\sum_{i \neq j} G_{ji} h_i.
    \end{equation}
    
    Taking absolute values and using $|G_{ji}| \leq \mu$, we obtain
    \begin{equation}
        |h_j| \leq \mu \sum_{i \neq j} |h_i|.
    \end{equation}
    
    Summing over $j \in \mathcal S$ gives
    \begin{equation}
        \sum_{j \in \mathcal S} |h_j|\ \leq \ \mu \sum_{j \in \mathcal S} \sum_{i \neq j} |h_i|.
    \end{equation}
    
    The inner summation can be decomposed into contributions from $i \in \mathcal S \setminus \{j\}$ and $i \in \mathcal S^c$:
    \begin{equation}
        \sum_{j \in \mathcal S}\sum_{i \neq j}|h_i|= \underbrace{\sum_{j \in \mathcal S}\sum_{\substack{i \in \mathcal S \\ i \neq j}} |h_i|}_{\text{each }i \in \mathcal S \text{ counted }K-1\text{ times}}+ \underbrace{\sum_{j \in \mathcal S}\sum_{i \in \mathcal S^c} |h_i|}_{\text{each }i \in \mathcal S^c \text{ counted }K\text{ times}}.
    \end{equation}
    
    Hence,
    \begin{equation}
        \|\bm h_\mathcal S\|_1 \ \leq\ \mu \big( (K-1)\|\bm h_\mathcal S\|_1 + K \|\bm h_{\mathcal S^c}\|_1 \big).
    \end{equation}
    
    Rearranging,
    \begin{equation}
        \big(1 - (K-1)\mu \big)\|\bm h_\mathcal S\|_1 \ \leq\ K\mu\,\|\bm h_{\mathcal S^c}\|_1.
    \end{equation}
    
    Dividing through by the positive factor $1 - (K-1)\mu$, define
    \begin{equation}
        \alpha := \frac{K\mu}{1 - (K-1)\mu}.
    \end{equation}
    
    When $\mu < \tfrac{1}{2K-1}$, we have $\alpha < 1$. Since $\bm h \neq 0$ and $\tilde D \bm h = 0$, it is impossible for $\|\bm h_{\mathcal S^c}\|_1 = 0$ (otherwise both terms would vanish, forcing $\bm h=0$, a contradiction). Therefore,
    \begin{equation}
        \|\bm h_\mathcal S\|_1 \ \leq\ \alpha \|\bm h_{\mathcal S^c}\|_1 \ < \ \|\bm h_{\mathcal S^c}\|_1.
    \end{equation}
    
    Take any $\mathcal S_0$ with $|\mathcal S_0| = k \leq K$, and extend it to a superset $\mathcal S \supseteq \mathcal S_0$ such that $|\mathcal S| = K$. Then,
    \begin{equation}
        \|\bm h_{\mathcal S_0}\|_1 \ \leq\ \|\bm h_{\mathcal S}\|_1, \qquad \|\bm h_{\mathcal S_0^c}\|_1 \ \geq\ \|\bm h_{\mathcal S^c}\|_1.
    \end{equation}
    
    If we know that $\|\bm h_\mathcal S\|_1 < \|\bm h_{\mathcal S^c}\|_1$ holds for all $\mathcal S$ of size $K$, then it follows that
    \begin{equation}
        \|\bm h_{\mathcal S_0}\|_1 \ \leq\ \|\bm h_{\mathcal S}\|_1 \ <\ \|\bm h_{\mathcal S^c}\|_1 \ \leq\ \|\bm h_{\mathcal S_0^c}\|_1.
    \end{equation}
    
    Thus, the inequality also holds for any $\mathcal S_0$ with $|\mathcal S_0| \leq K$, which establishes NSP$_K$.

    \paragraph{Step (ii): Equivalence between NSP$_K$ and $\ell_1$ recovery.}

    \paragraph{NSP$_K$ $\Rightarrow$ unique $\ell_1$ recovery:} Suppose $\hat{\bm x}$ is another feasible solution such that $\tilde D \hat{\bm x} = \tilde D \bm \delta$. Let $\bm h = \hat{\bm x} - \bm \delta \in \ker(\tilde D) \setminus \{\bm 0\}$, and let $\mathcal S = \mathrm{supp}(\bm \delta)$ with $|\mathcal S|\le K$. Then
    \begin{equation}
        \|\hat{\bm x}\|_1 = \|\bm \delta + \bm h\|_1= \|\bm \delta_\mathcal S + \bm h_\mathcal S\|_1 + \|\bm h_{\mathcal S^c}\|_1\ \geq\ \|\bm \delta_\mathcal S\|_1 - \|\bm h_\mathcal S\|_1 + \|\bm h_{\mathcal S^c}\|_1\ >\ \|\bm \delta_\mathcal S\|_1= \|\bm \delta\|_1,
    \end{equation}
    where the strict inequality follows from NSP$_K$. Hence, $\bm \delta$ is the unique minimizer of the $\ell_1$ problem.

    \paragraph{Unique $\ell_1$ recovery $\Rightarrow$ NSP$_K$:} We argue by contradiction. If NSP$_K$ does not hold, then there exists $\bm h \in \ker(\tilde D)\setminus\{\bm 0\}$ and some $\mathcal S$ with $|\mathcal S|\le K$ such that $\|\bm h_\mathcal S\|_1 \geq \|\bm h_{\mathcal S^c}\|_1$. Take any nonzero $K$-sparse $\bm \delta$ with $\mathrm{supp}(\bm \delta) = \mathcal S$, and choose $\delta_j = \alpha_j \operatorname{sgn}(h_j)$ with $\alpha_j \geq |h_j|$ coordinate-wise. Consider $\hat{\bm x} = \bm \delta - \bm h$. Since $\tilde D \bm h = \bm 0$, both $\bm \delta$ and $\hat{\bm x}$ are feasible, and
    \begin{equation}
        \|\hat{\bm x}\|_1 = \|\bm \delta_\mathcal S - \bm h_\mathcal S\|_1 + \|\bm h_{\mathcal S^c}\|_1= \|\bm \delta_\mathcal S\|_1 - \|\bm h_\mathcal S\|_1 + \|\bm h_{\mathcal S^c}\|_1\ \leq \ \|\bm \delta_\mathcal S\|_1= \|\bm \delta\|_1.
    \end{equation}
    
    Thus, $\bm \delta$ is not the unique minimizer of the $\ell_1$ problem (and may even fail to be a minimizer). This contradicts the uniqueness assumption. Therefore, NSP$_K$ must hold.
\end{proof}

\subsection{Proof of Corollary~\ref{corollary:Monorepresentationality}}
\Monorepresentationality*
\begin{proof}
    Take arbitrary \(\bm{x}, \bm{y} \in \Sigma_K\) such that \(\tilde{D}\bm{x} = \tilde{D}\bm{y}\). Let \(\tilde{\bm{m}} := \tilde{D}\bm{x}\), then clearly \(\tilde{\bm{m}} = \tilde{D}\bm{y}\), with \(\|\bm{x}\|_0 \le K\) and \(\|\bm{y}\|_0 \le K\).
    
    By Theorem~\ref{theorem:Uniqueness_of_Sparse_Representation_and_Exact_Recovery}, under the condition \(\mu < \tfrac{1}{2K-1}\), the \(K\)-sparse representation \(\bm{\delta}\) determined by the equation \(\tilde{\bm{m}} = \tilde{D}\bm{\delta}\) is unique: among all coefficient vectors satisfying \(\|\bm{\delta}\|_0 \le K\), there exists only one that generates \(\tilde{\bm{m}}\).
    
    Since both \(\bm{x}\) and \(\bm{y}\) are feasible solutions satisfying \(\|\bm{\delta}\|_0 \le K\) and \(\tilde{D}\bm{\delta} = \tilde{\bm{m}}\), uniqueness forces \(\bm{x} = \bm{y}\). Hence, \(\Phi\) is injective on \(\Sigma_K\).
\end{proof}

\subsection{Proof of Theorem~\ref{theorem:Identifiability_of_SAEs}}
\IdentifiabilitySAEs*
\begin{proof}
    Consider a single-layer linear–nonlinear encoder of the form $W_{\mathrm{dec}}\,\sigma_\tau(W_{\mathrm{enc}}\bm m_i)$, with training objective
    \begin{equation}
        W_{\mathrm{dec}}\,\sigma_\tau(W_{\mathrm{enc}}\bm m_i) = \bm m_i, \quad \forall i.
    \end{equation}
    
    Set $W_{\mathrm{dec}} = D$ and $W_{\mathrm{enc}} = D^\top \tilde S$. Denote $\mathcal S_i = \mathrm{supp}(\bm \delta_i)$. Then

    \begin{align}
        D^\top \tilde S \bm m_i 
        &= 
        \begin{bmatrix}
        \bm d_1^\top \\ \bm d_2^\top \\ \vdots \\ \bm d_{|D|}^\top
        \end{bmatrix} 
        \tilde S
        \begin{bmatrix}
        \bm d_1 & \bm d_2 & \cdots & \bm d_{|D|}
        \end{bmatrix}\bm \delta_i \\
        &= 
        \begin{bmatrix}
        \bm d_1^\top \tilde S \bm d_1 & \cdots & \bm d_1^\top \tilde S \bm d_{|D|}\\
        \bm d_2^\top \tilde S \bm d_1 & \cdots & \bm d_2^\top \tilde S \bm d_{|D|}\\
        \vdots & \ddots & \vdots \\
        \bm d_{|D|}^\top \tilde S \bm d_1 & \cdots & \bm d_{|D|}^\top \tilde S \bm d_{|D|}
        \end{bmatrix}\bm \delta_i \\
        &= G\,\bm \delta_i,\quad G := D^\top \tilde S D .
    \end{align}
    
    By NAIP, we have $G_{kk} = 1$ and for $k \neq j$, $|G_{kj}| = |\langle \tilde{\bm d_k}, \tilde{\bm d_j} \rangle| \leq \varepsilon$. Thus, for any index $k$,
    \begin{equation}
        (G\bm \delta_i)_k =\begin{cases}\delta_{ik} + \underbrace{\sum_{j \in \mathcal S_i \setminus \{k\}} \delta_{ij} G_{kj}}_{=:e_{ik}}, & k \in \mathcal S_i, \\[6pt]\underbrace{\sum_{j \in \mathcal S_i} \delta_{ij} G_{kj}}_{=:e_{ik}}, & k \notin \mathcal S_i.\end{cases}
    \end{equation}
    
    Using the coherence bound, we obtain the deterministic perturbation estimate
    \begin{equation}
        \begin{cases}(G\bm \delta_i)_k \ \geq\ \delta_{ik} - \varepsilon (K-1)\delta_{\max}, & k \in \mathcal S_i, \\[4pt](G\bm \delta_i)_k \ \leq\ \varepsilon K \delta_{\max}, & k \notin \mathcal S_i.\end{cases}
    \end{equation}
    
    Choose a threshold $\tau$ such that
    \begin{equation}
        \varepsilon K \delta_{\max} \ <\ \tau\ <\ \delta_{\min} - \varepsilon (K-1)\delta_{\max}.
    \end{equation}
    
    This ensures support separation
    \begin{equation}
        \begin{cases}(G\bm \delta_i)_k > \tau, & k \in \mathcal S_i, \\[4pt](G\bm \delta_i)_k < \tau, & k \notin \mathcal S_i.\end{cases}
    \end{equation}

    Therefore, the coordinate-wise nonlinearity
    \begin{equation}
        \sigma_\tau(x) := \begin{cases}0 & x < \tau, \\x & x \geq \tau\end{cases}
    \end{equation}
    produces activations $\bm z_i := \sigma_\tau(G\bm \delta_i)$, with $\mathrm{supp}(\bm z_i) = \mathcal S_i$. For $k \in \mathcal S_i$,
    \begin{equation}
        z_{ik} = (G\bm \delta_i)_k = \delta_{ik} + e_{ik}.
    \end{equation}
    
    Since for any $j \neq k$, $G_{kj} = \bm d_k^\top \tilde S \bm d_j$ is distributed approximately as $\mathcal N(0, s^2)$ with small variance $s$, it follows that
    \begin{equation}
        \mathbb E[e_{ik}] = \mathbb E\!\left[\sum_{j \in \mathcal S_i \setminus \{k\}} \delta_{ij} G_{kj}\right] = \sum_{j \in \mathcal S_i \setminus \{k\}} \delta_{ij}\, \mathbb E[G_{kj}] = 0.
    \end{equation}
    
    Since $\mathbb E[e_{ik}] = 0$ for all $i, k$, the law of large numbers implies that, in probability,
    \begin{equation}
        D\,\sigma_\tau(D^\top \tilde S \bm m_i) = \bm m_i, \quad \forall i.
    \end{equation}
    
    Thus, under this parametrization, the SAE recovers the target atom set $D$.
\end{proof}

\section{Representation Shift}\label{appendix:Representation_Shift}

In this section, we provide further details on representation shift, including the experimental setup, ablation studies, and additional analyses, to enable a comprehensive understanding of this phenomenon.

Specifically, we randomly sample 128 subject entities from the WikiData dataset \citep{vrandevcic2014wikidata} and extract the corresponding activations across all layers of multiple language model families, including GPT-2 (GPT2-Small, GPT2-Medium, GPT2-Large), GPT-J (GPT-J-6B), Pythia (Pythia-1B, Pythia-1.4B, Pythia-2.8B, Pythia-6.9B), Llama-2 (Llama2-7B, Llama2-13B), Llama-3 (Llama3-8B), Llama-3.1 (Llama3.1-8B), and Gemma-2 (Gemma2-2B, Gemma2-9B), which serve as the target representations for our analysis.

For each entity, we extract activations at the position of its final token, which has been empirically identified via causal tracing as the key site of knowledge extraction in language models \citep{meng2022locating}. Thus, for each layer we obtain 128 representations, yielding \(128 \times 128 = 16{,}384\) representation pairs for analyzing angular distributions. This sample size is sufficient to capture the overall distributional characteristics, as additional samples do not alter the distribution (Figs.~\ref{fig:Representations_Shifting_GPT2-Small_ablation}–\ref{fig:Representations_Shifting_Gemma2-9B_ablation}), while also facilitating non-overlap with activations used in subsequent sampling.

We first compute the angles between representation pairs using the Euclidean inner product. Specifically, for representations \(\bm u, \bm v \in \mathbb{R}^H\), the Euclidean cosine similarity is defined as \(\cos(\bm u, \bm v) = \frac{\langle \bm u, \bm v \rangle}{\|\bm u\|_2 \|\bm v\|_2}\), which we then convert to angles for clearer visualization.

The full results (Figs.~\ref{fig:Representations_Shifting_GPT2-Small}–\ref{fig:Representations_Shifting_Gemma2-9B}) show that when angles are computed under the Euclidean inner product, the centroid of the angular distribution deviates markedly from \(90^\circ\), indicating a representation shift. This phenomenon consistently appears across all layers and model families, independent of model architectures and training corpora, and thus reflects a systematic non-uniformity in representation distributions. Such anisotropy distorts the underlying geometry of representations, which should be isotropic for random epresentations. In LLMs, this anisotropy is nearly indiscriminate: representations are globally attracted toward a dominant direction, yielding high cosine similarity even for unrelated samples. Direct evidence is provided in Figs.~\ref{fig:Representations_Shifting_GPT2-Small_ablation}–\ref{fig:Representations_Shifting_Gemma2-9B_ablation}: increasing the sample size used to estimate the distribution leaves the centroid unchanged, confirming that representations are indeed attracted toward the same direction.

We then collect 100K activations \(\bm k\) per layer for each model from Wikipedia (manually verified to have no overlap with the target representations) and compute \(\mathbb{E}[\bm k \bm k^\top]\) to estimate the normalized atomic inner product. Under the inner-product space induced by \(\tilde S\), the cosine similarity is defined as \(\cos_{\tilde S}(\bm u,\bm v) = \frac{\langle \bm u, \bm v\rangle_{\tilde S}}{\|\bm u\|_{\tilde S}\|\bm v\|_{\tilde S}}= \frac{\bm u^\top \tilde S \bm v}{\sqrt{\bm u^\top \tilde S \bm u}\sqrt{\bm v^\top \tilde S \bm v}}\), where \(\tilde S\) is estimated in practice as (\(\mathbb{E}[\bm k \bm k^\top])^{-1}\). As before, cosine similarities are converted to angles for clearer visualization.

As shown in Figs.~\ref{fig:Solving_Representations_Shifting_GPT2-Small}–\ref{fig:Solving_Representations_Shifting_Gemma2-9B}, the atomic inner product corrects the representation shift, restoring the centroid of the angular distribution to \(90^\circ\) and thereby capturing the underlying representational geometry, aligned with theoretical expectations. This effect holds consistently across all layers and model families. Further inspection shows that the long tail of the distribution (high-similarity pairs) corresponds to highly related samples, indicating that LLM representations are intrinsically isotropic when measured with the appropriate metric.

Although we identify the correct inner product for LLM representations and verify that their global geometry matches theoretical expectations, substantial superposition remains pervasive \citep{hu2025knowledge}. For example, Fig.~\ref{fig:Superposition_Gemma2-2B} shows that activations in Gemma2-2B still exhibit widespread superposition, indicating that these representations are not fully disentangled. The persistence of superposition suggests that raw activations are not the most suitable fundamental representational units. We therefore decompose high-dimensional representations into atoms that better satisfy the criteria for FRUs; the resulting decomposition (Fig.~\ref{fig:Solving_Superposition_Gemma2-2B}) demonstrates that the identified atoms effectively disentangle representations and resolve superposition.

\begin{figure}
    \centering
    \includegraphics[width=0.96\textwidth]{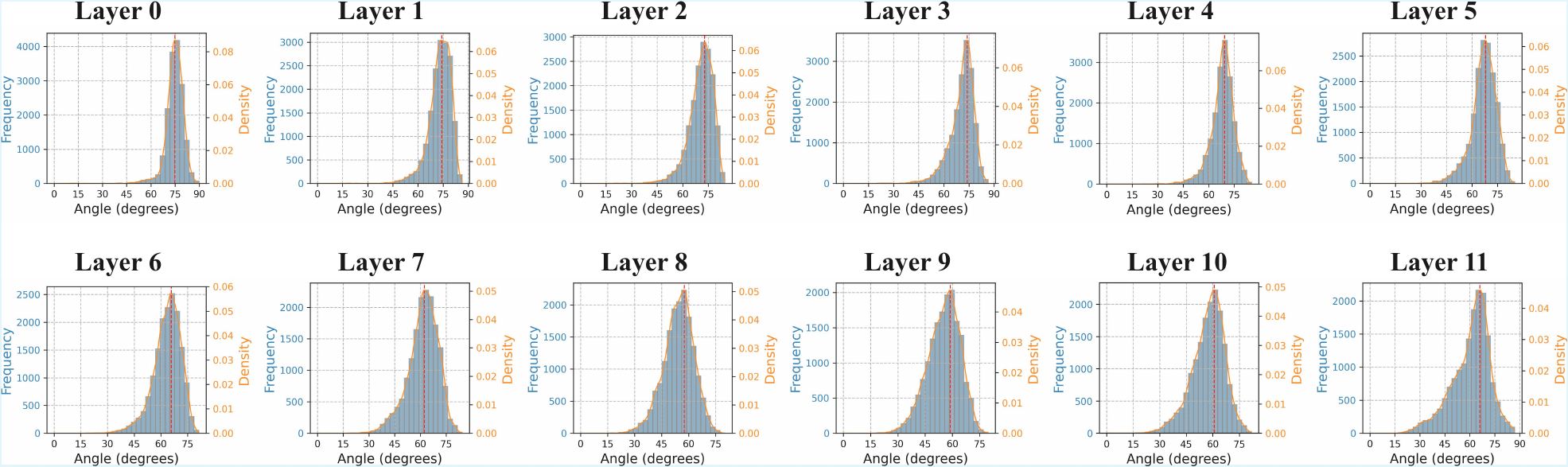}
    \caption{Representation shift of GPT2-Small.}
  \label{fig:Representations_Shifting_GPT2-Small}
\end{figure}

\begin{figure}
    \centering
    \includegraphics[width=0.96\textwidth]{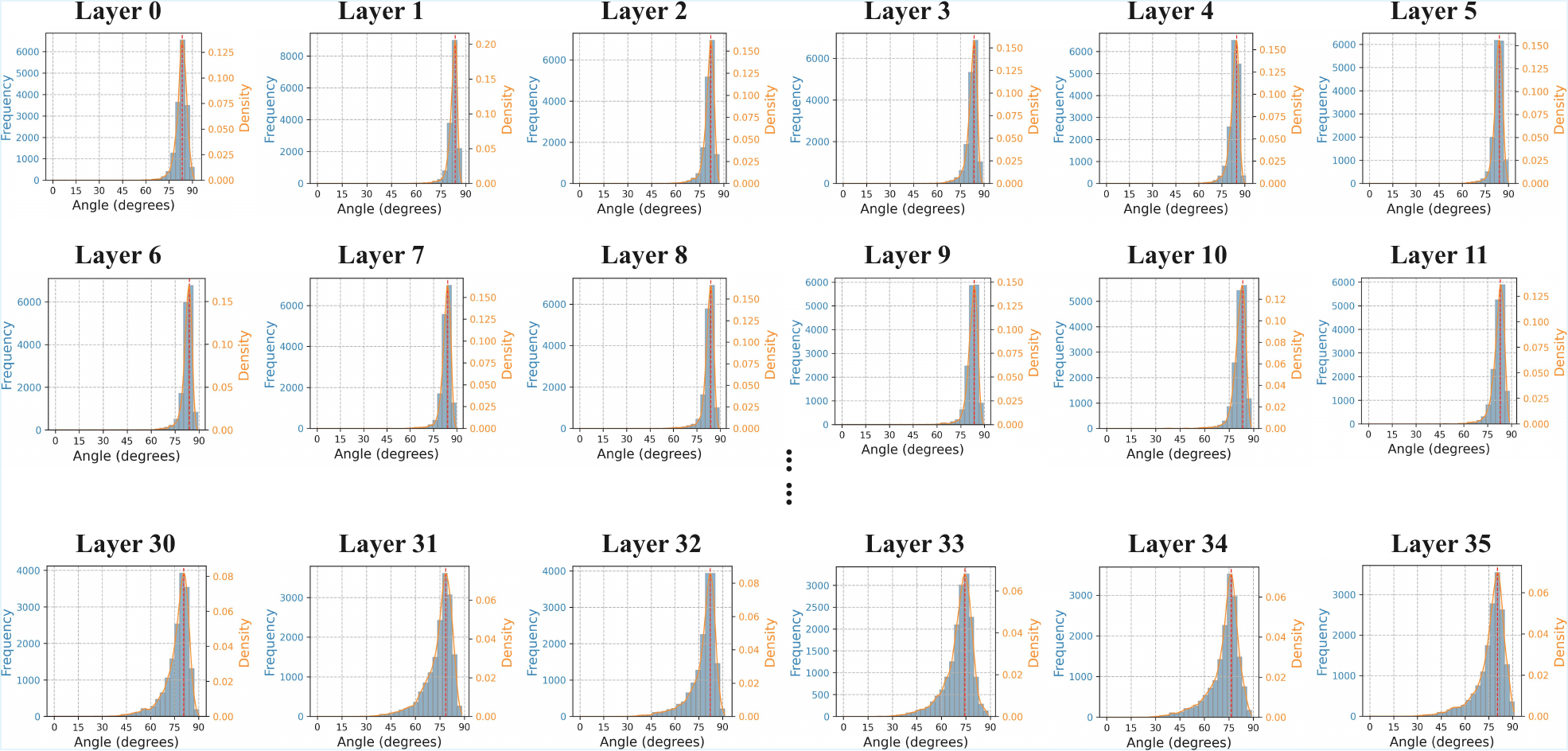}
    \caption{Representation shift of GPT2-Large.}
  \label{fig:Representations_Shifting_GPT2-Large}
\end{figure}

\begin{figure}
    \centering
    \includegraphics[width=0.96\textwidth]{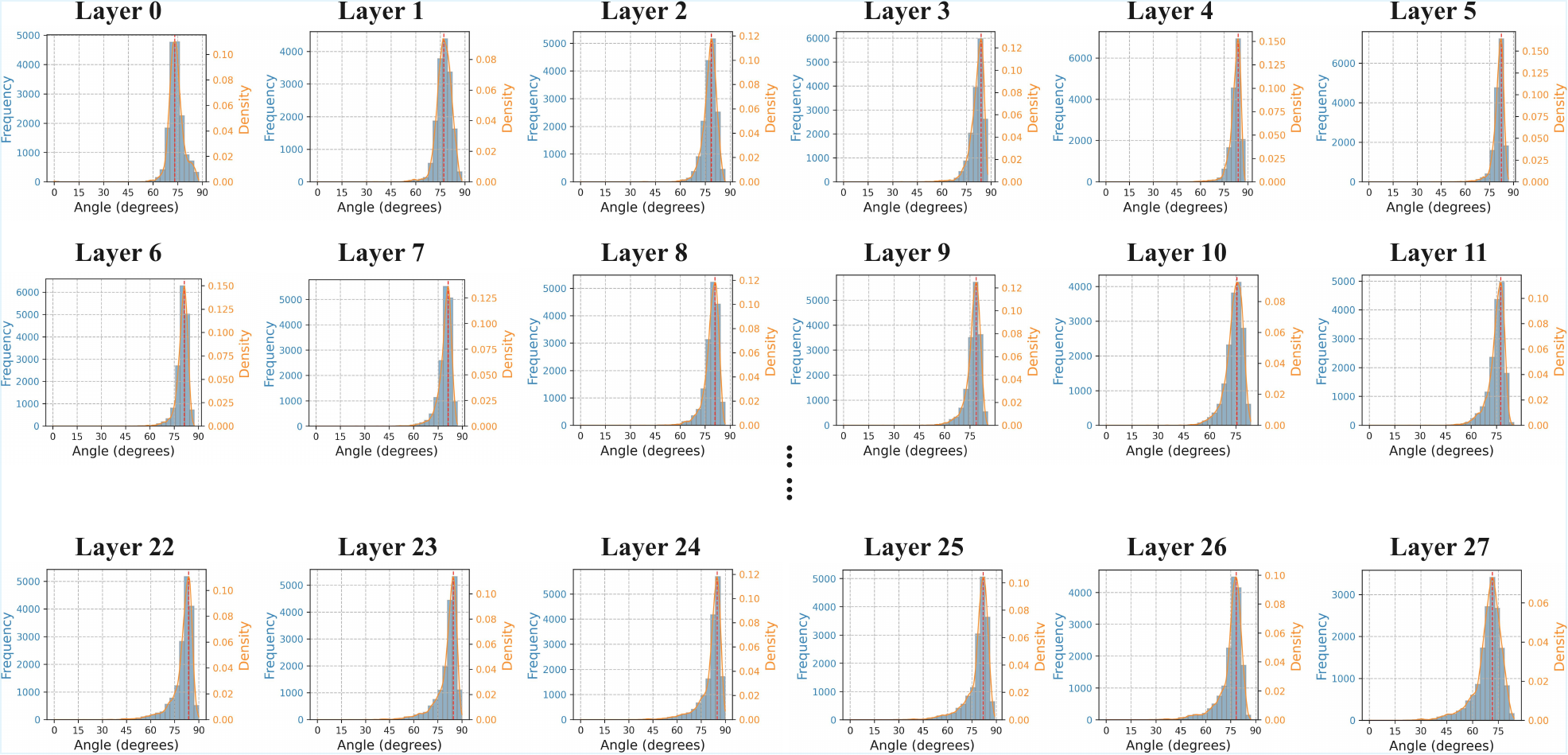}
    \caption{Representation shift of GPT-J-6B.}
  \label{fig:Representations_Shifting_GPT-J-6B}
\end{figure}

\begin{figure}
    \centering
    \includegraphics[width=0.96\textwidth]{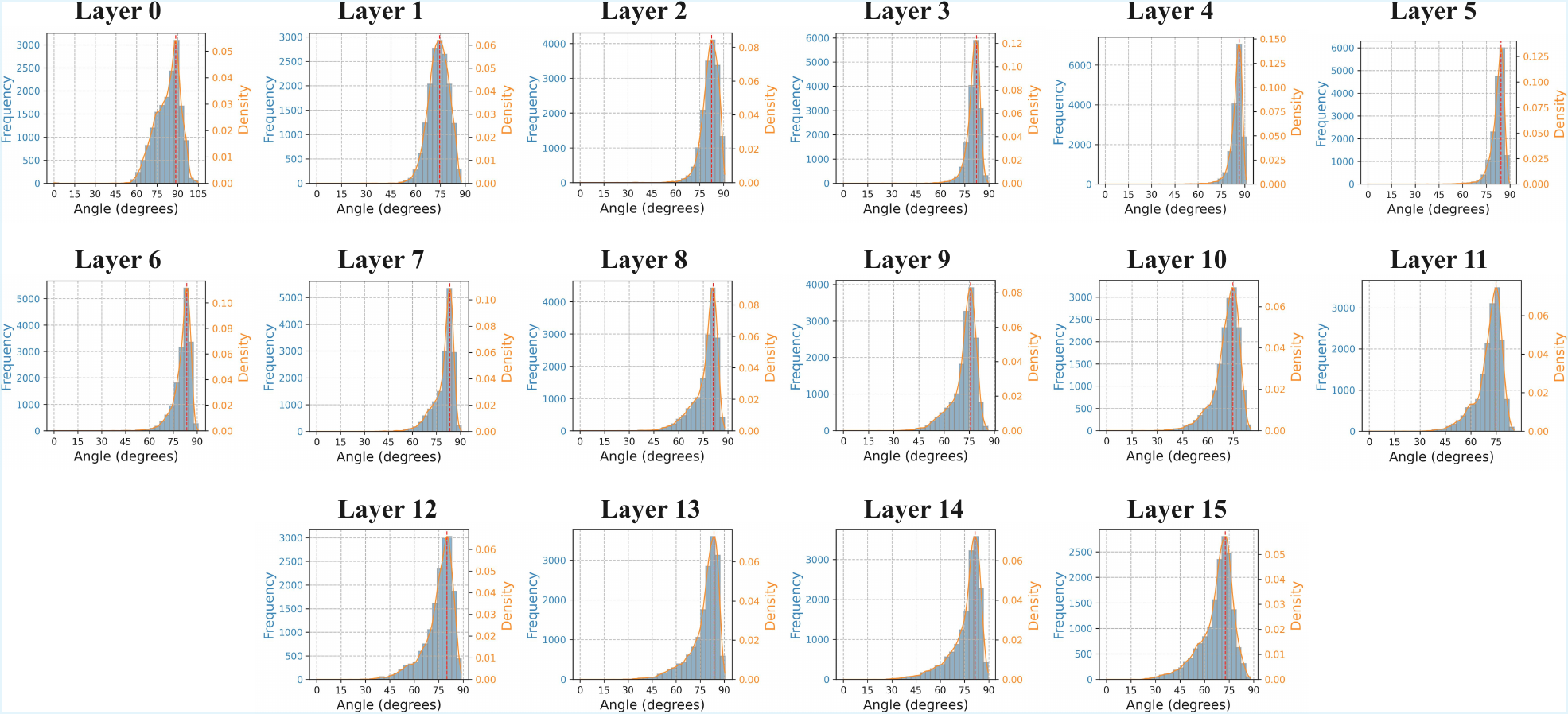}
    \caption{Representation shift of Pythia-1B.}
  \label{fig:Representations_Shifting_Pythia-1B}
\end{figure}

\begin{figure}
    \centering
    \includegraphics[width=0.96\textwidth]{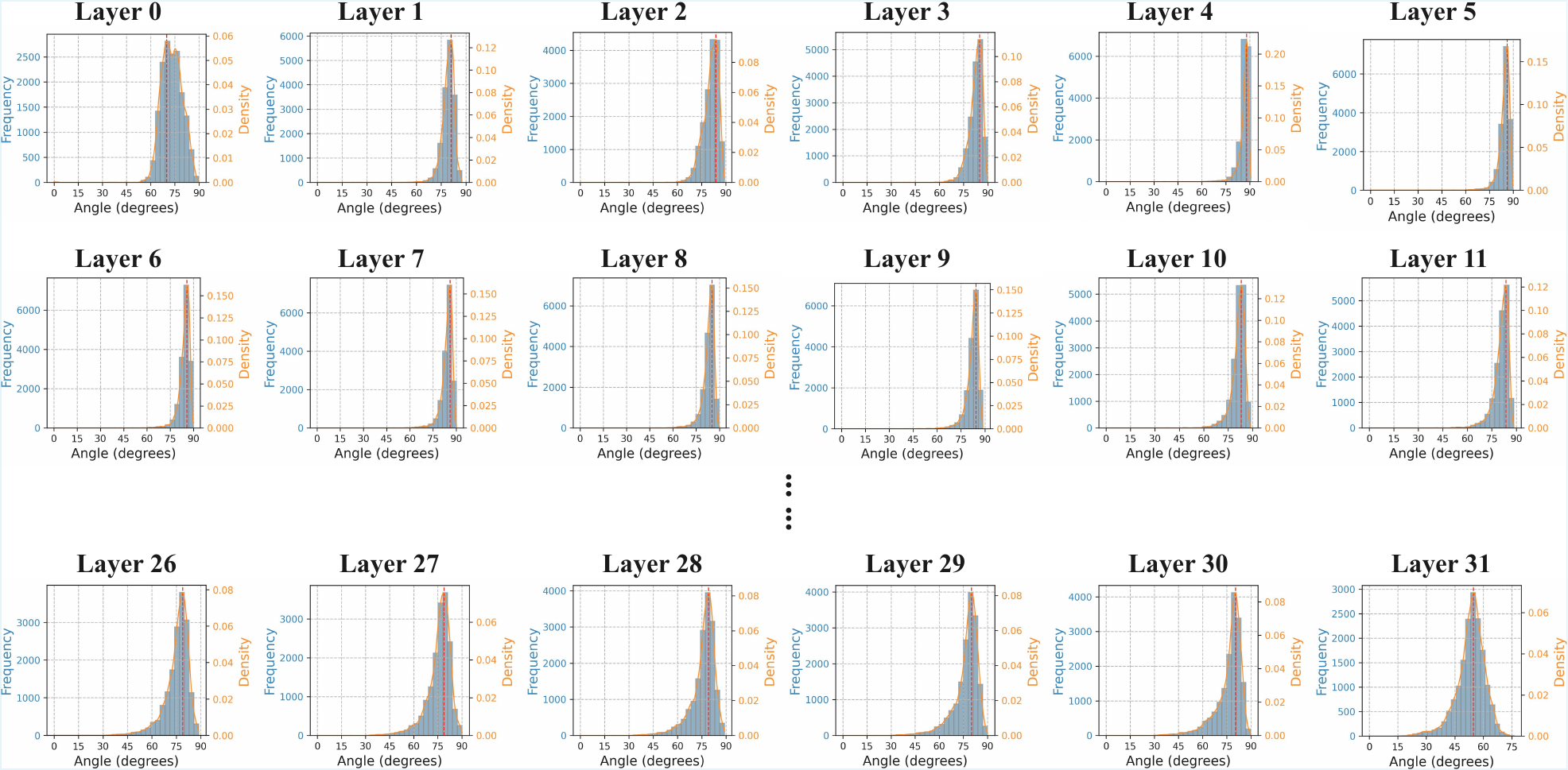}
    \caption{Representation shift of Pythia-6.9B.}
  \label{fig:Representations_Shifting_Pythia-6.9B}
\end{figure}

\begin{figure}
    \centering
    \includegraphics[width=0.96\textwidth]{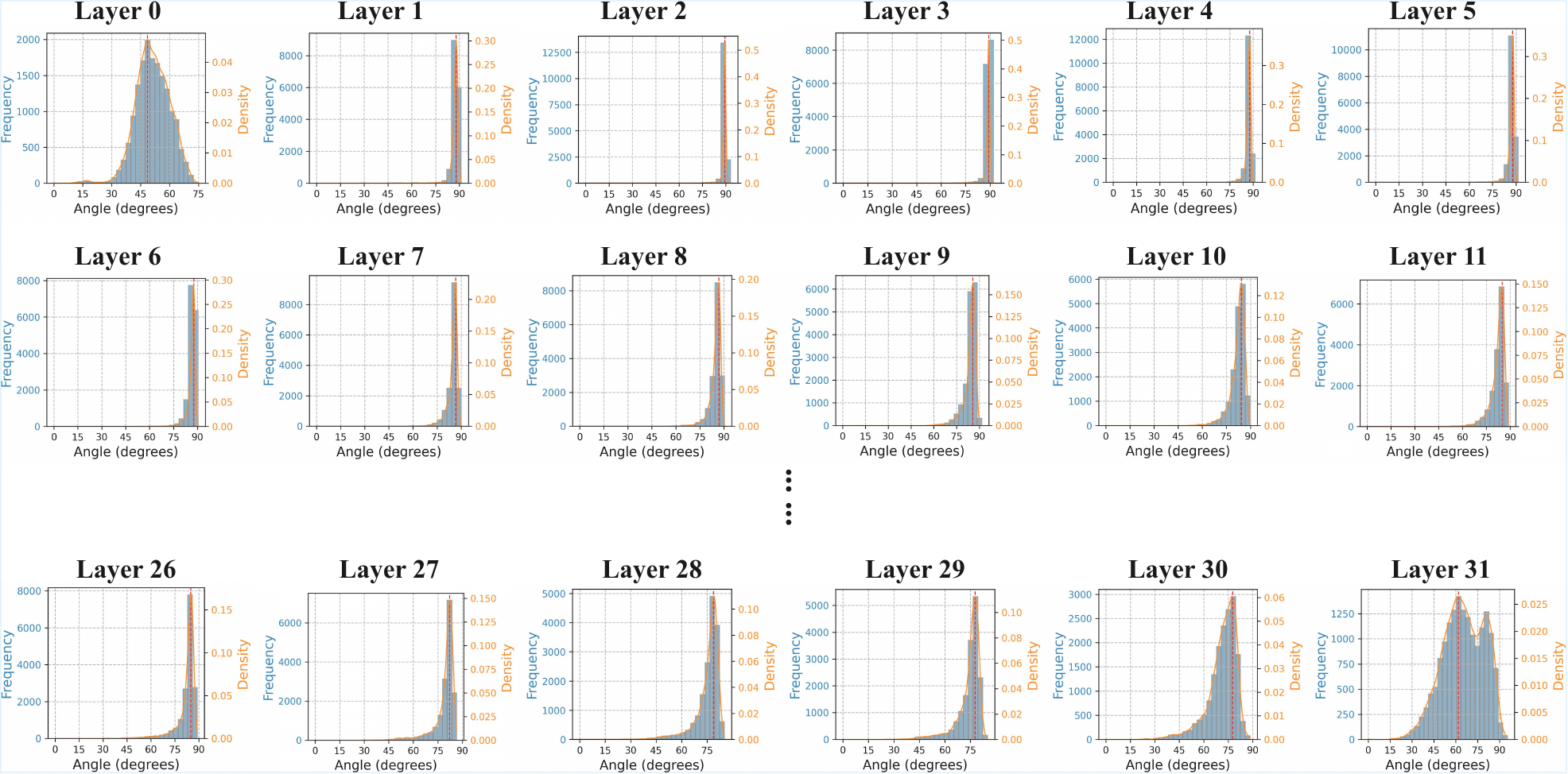}
    \caption{Representation shift of Llama2-7B.}
  \label{fig:Representations_Shifting_Llama2-7B}
\end{figure}

\begin{figure}
    \centering
    \includegraphics[width=0.96\textwidth]{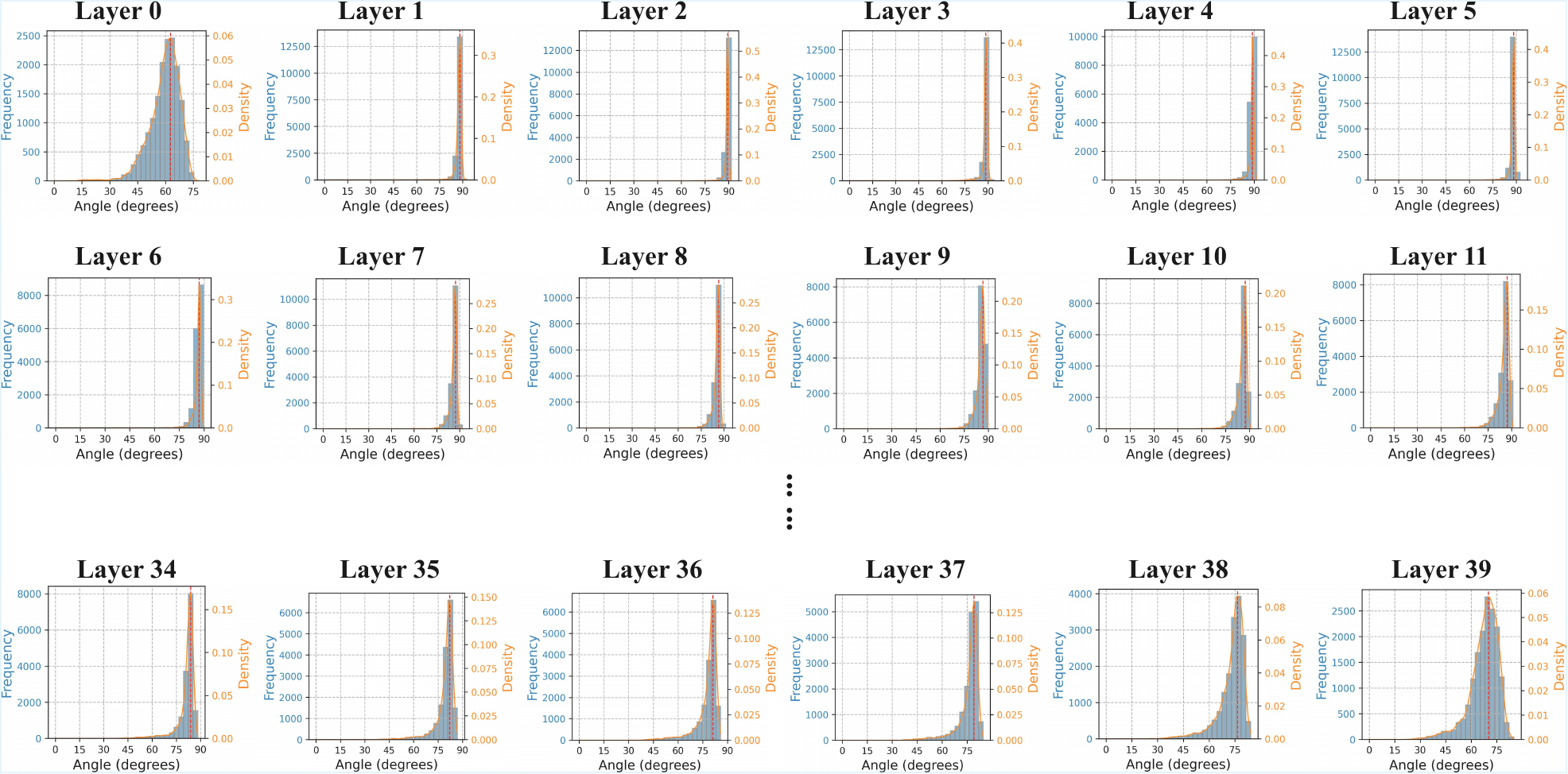}
    \caption{Representation shift of Llama2-13B.}
  \label{fig:Representations_Shifting_Llama2-13B}
\end{figure}

\begin{figure}
    \centering
    \includegraphics[width=0.96\textwidth]{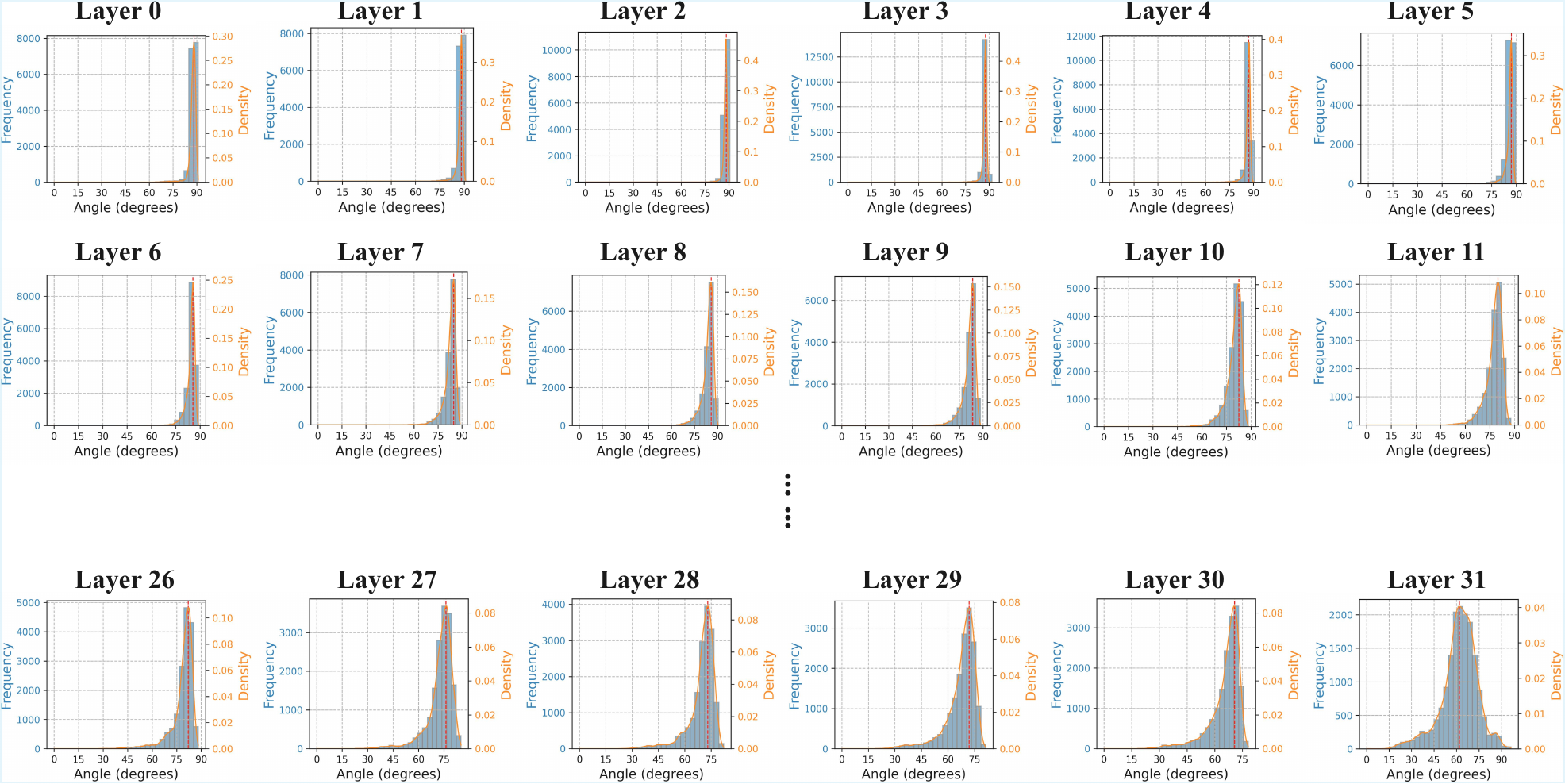}
    \caption{Representation shift of Llama3-8B.}
  \label{fig:Representations_Shifting_Llama3-8B}
\end{figure}

\begin{figure}
    \centering
    \includegraphics[width=0.96\textwidth]{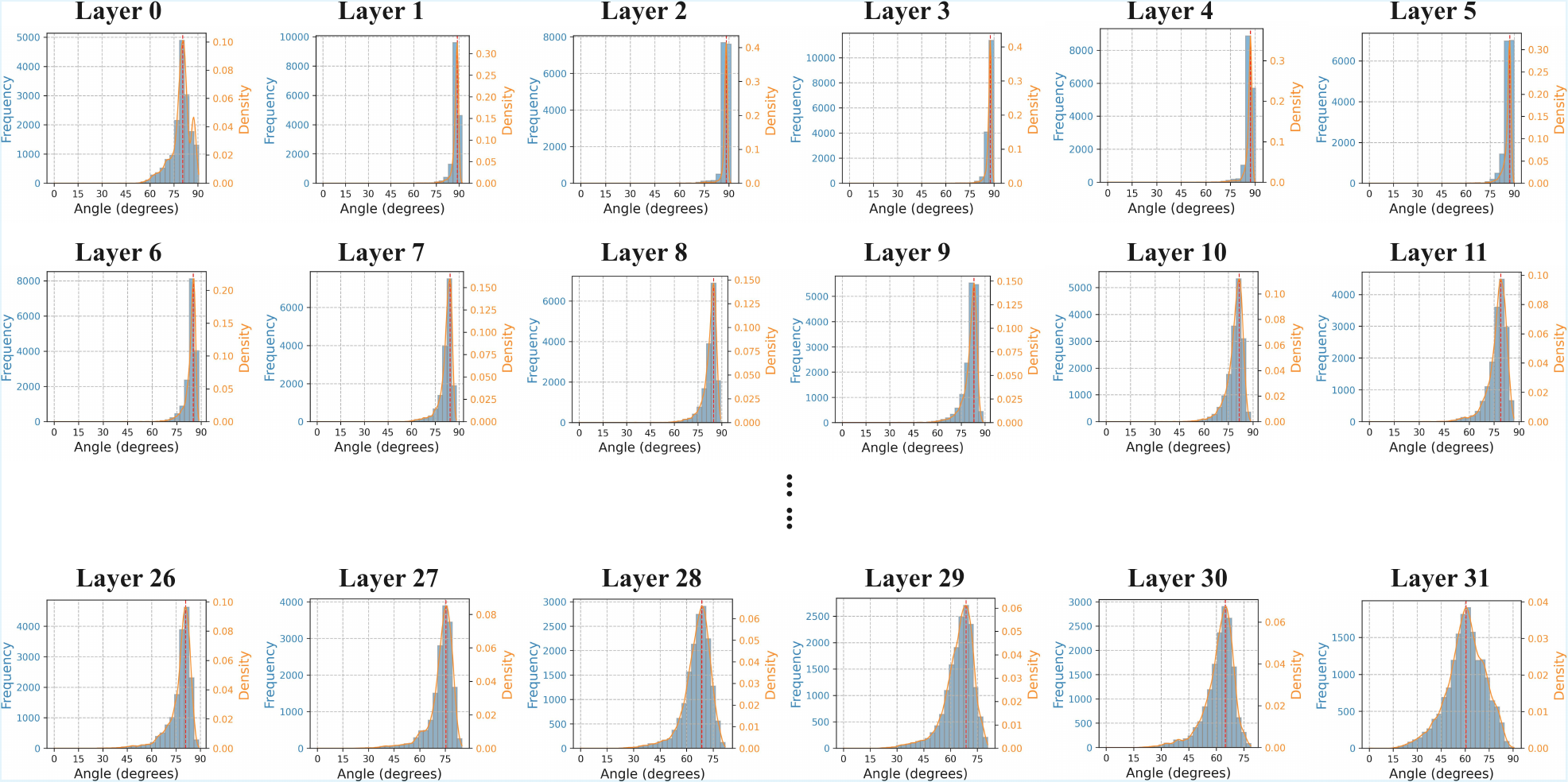}
    \caption{Representation shift of Llama3.1-8B.}
  \label{fig:Representations_Shifting_Llama3.1-8B}
\end{figure}

\begin{figure}
    \centering
    \includegraphics[width=0.96\textwidth]{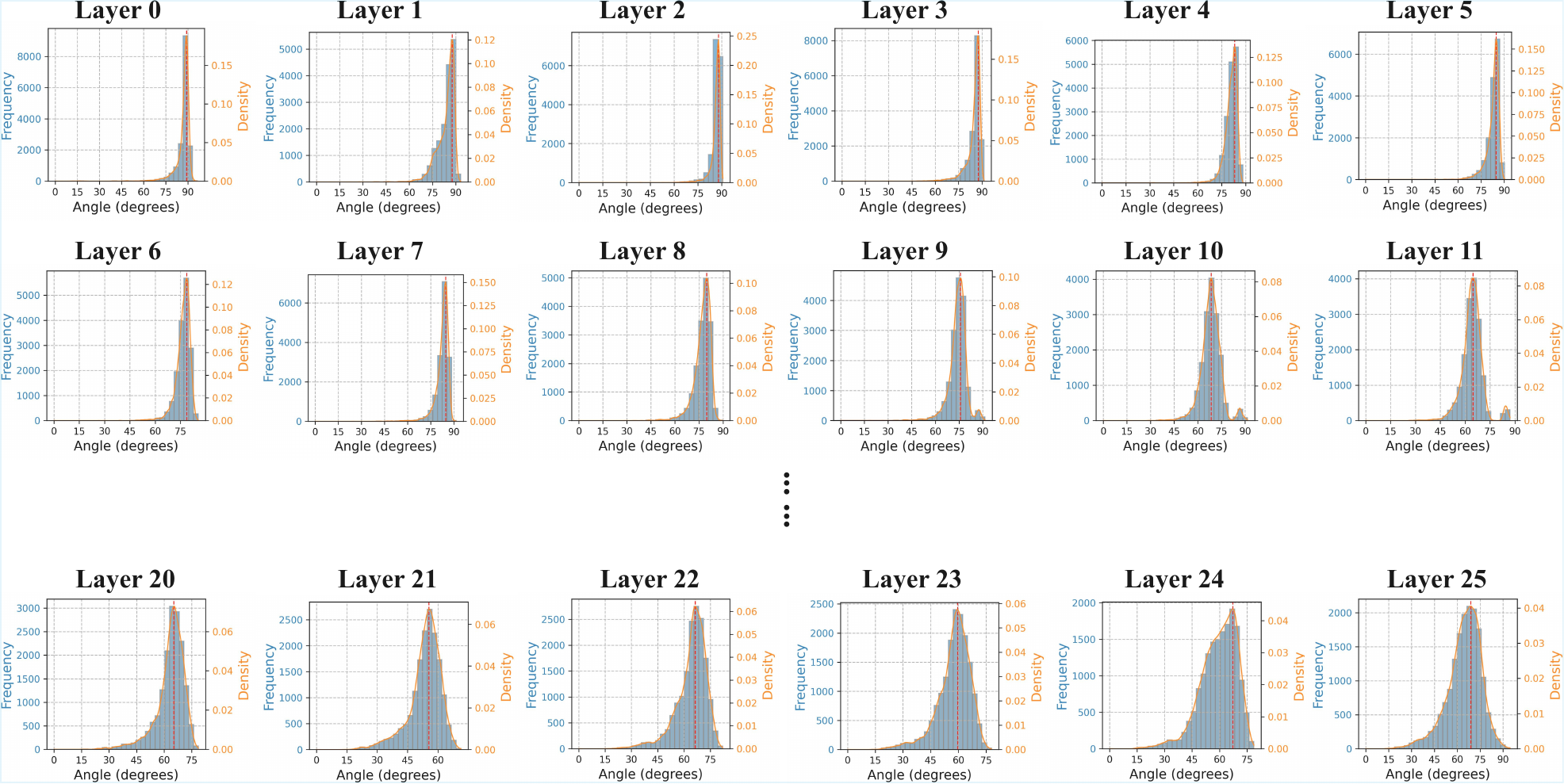}
    \caption{Representation shift of Gemma2-2B.}
  \label{fig:Representations_Shifting_Gemma2-2B}
\end{figure}

\begin{figure}
    \centering
    \includegraphics[width=0.96\textwidth]{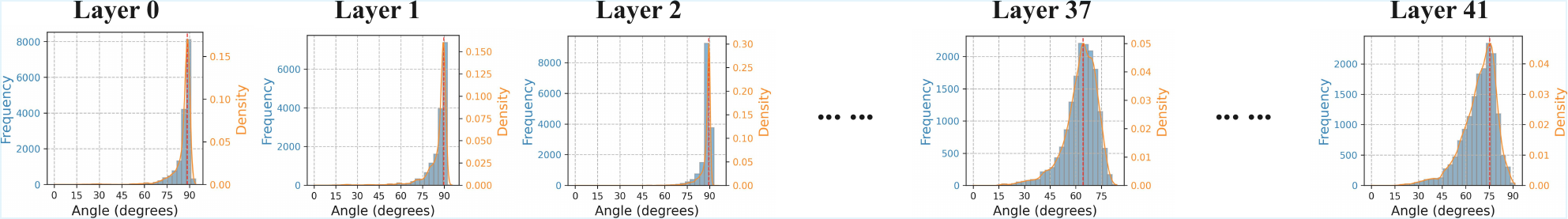}
    \caption{Representation shift of Gemma2-9B.}
  \label{fig:Representations_Shifting_Gemma2-9B}
\end{figure}


\begin{figure}
    \centering
    \includegraphics[width=0.96\textwidth]{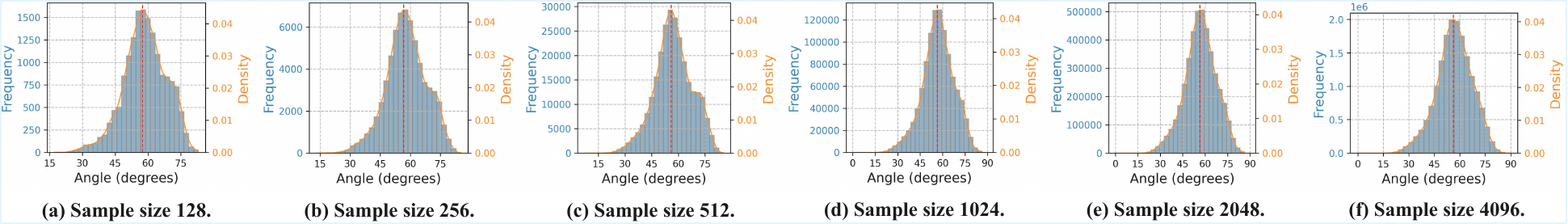}
    \caption{Representation shift of GPT2-Small under increasing sample sizes.}
  \label{fig:Representations_Shifting_GPT2-Small_ablation}
\end{figure}

\begin{figure}
    \centering
    \includegraphics[width=0.96\textwidth]{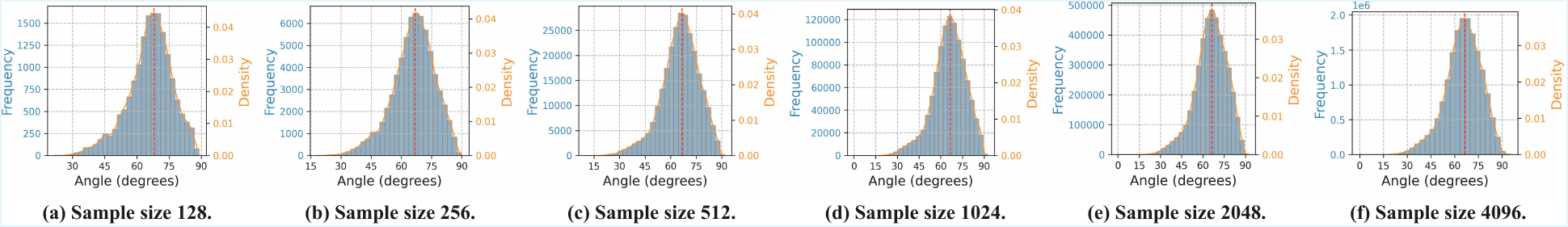}
    \caption{Representation shift of GPT2-Medium under increasing sample sizes.}
  \label{fig:Representations_Shifting_GPT2-Medium_ablation}
\end{figure}

\begin{figure}
    \centering
    \includegraphics[width=0.96\textwidth]{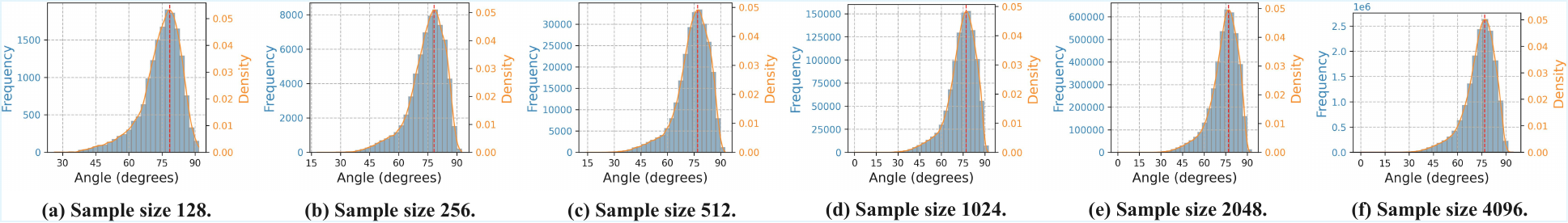}
    \caption{Representation shift of GPT2-Large under increasing sample sizes.}
  \label{fig:Representations_Shifting_GPT2-Large_ablation}
\end{figure}

\begin{figure}
    \centering
    \includegraphics[width=0.96\textwidth]{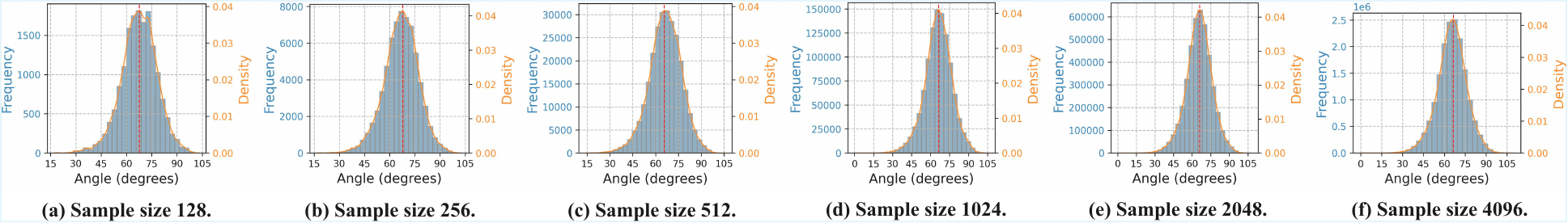}
    \caption{Representation shift of Gemma2-2B under increasing sample sizes.}
  \label{fig:Representations_Shifting_Gemma2-2B_ablation}
\end{figure}

\begin{figure}
    \centering
    \includegraphics[width=0.96\textwidth]{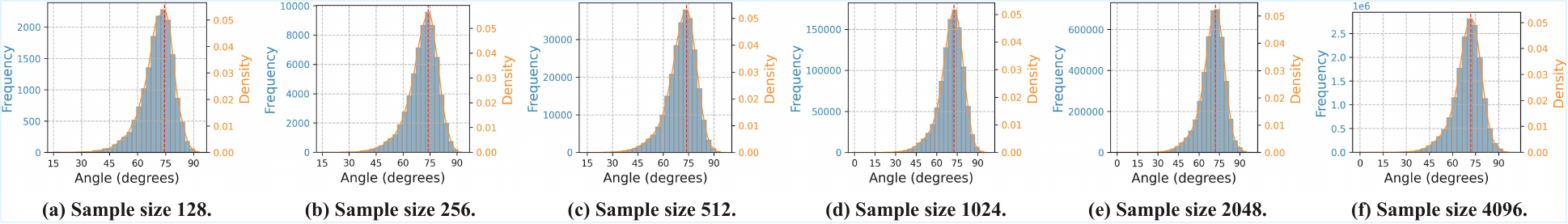}
    \caption{Representation shift of Gemma2-9B under increasing sample sizes.}
  \label{fig:Representations_Shifting_Gemma2-9B_ablation}
\end{figure}


\begin{figure}
    \centering
    \includegraphics[width=0.96\textwidth]{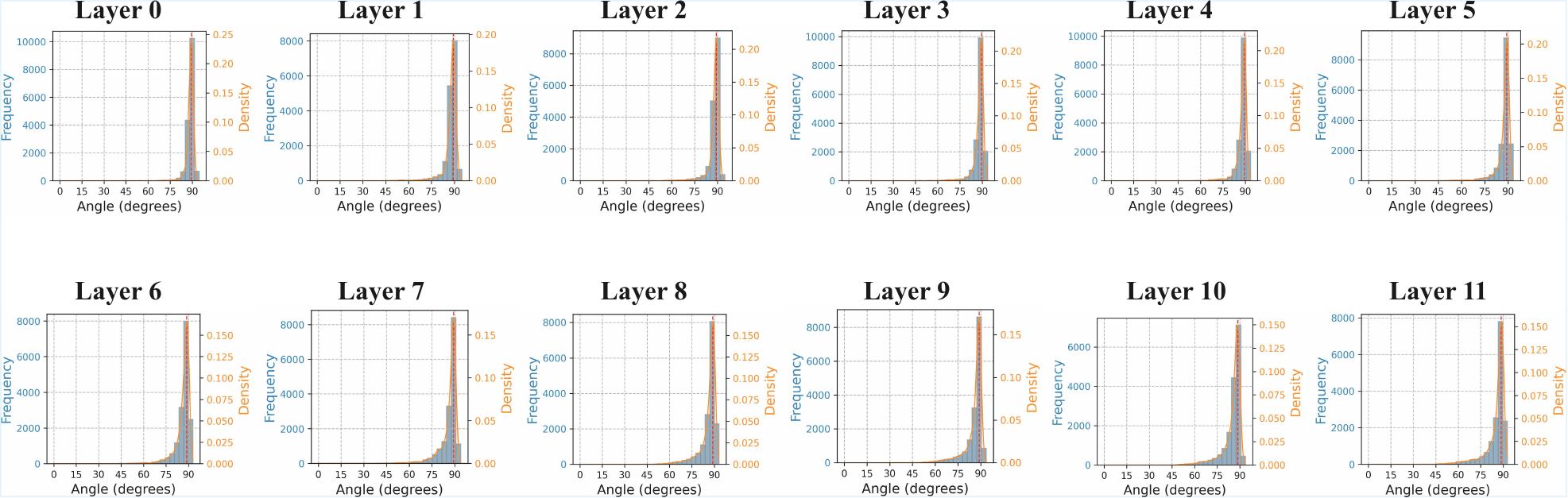}
    \caption{Correction of representation shift on GPT2-Small.}
  \label{fig:Solving_Representations_Shifting_GPT2-Small}
\end{figure}

\begin{figure}
    \centering
    \includegraphics[width=0.96\textwidth]{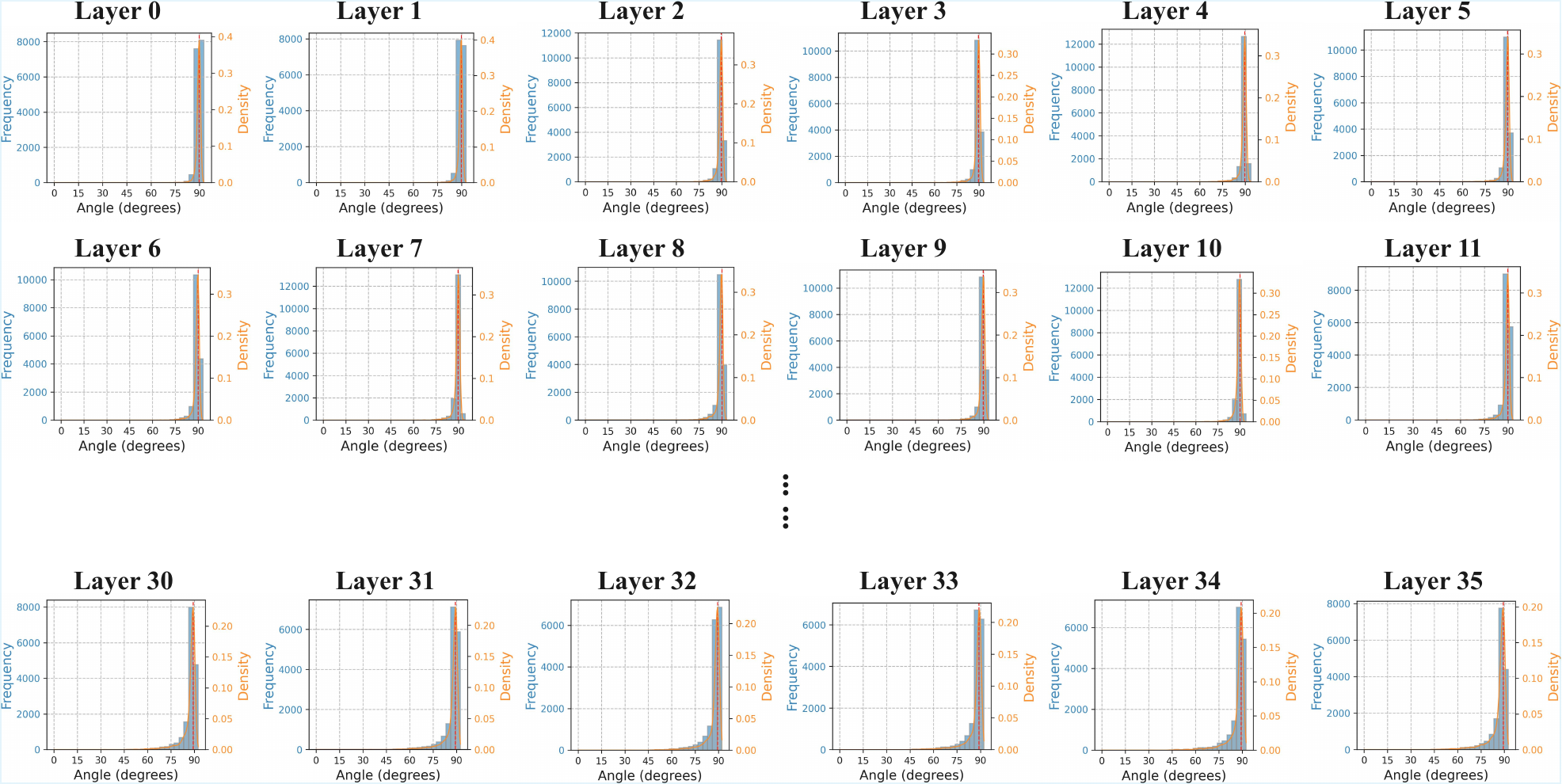}
    \caption{Correction of representation shift on GPT2-Large.}
  \label{fig:Solving_Representations_Shifting_GPT2-Large}
\end{figure}

\begin{figure}
    \centering
    \includegraphics[width=0.96\textwidth]{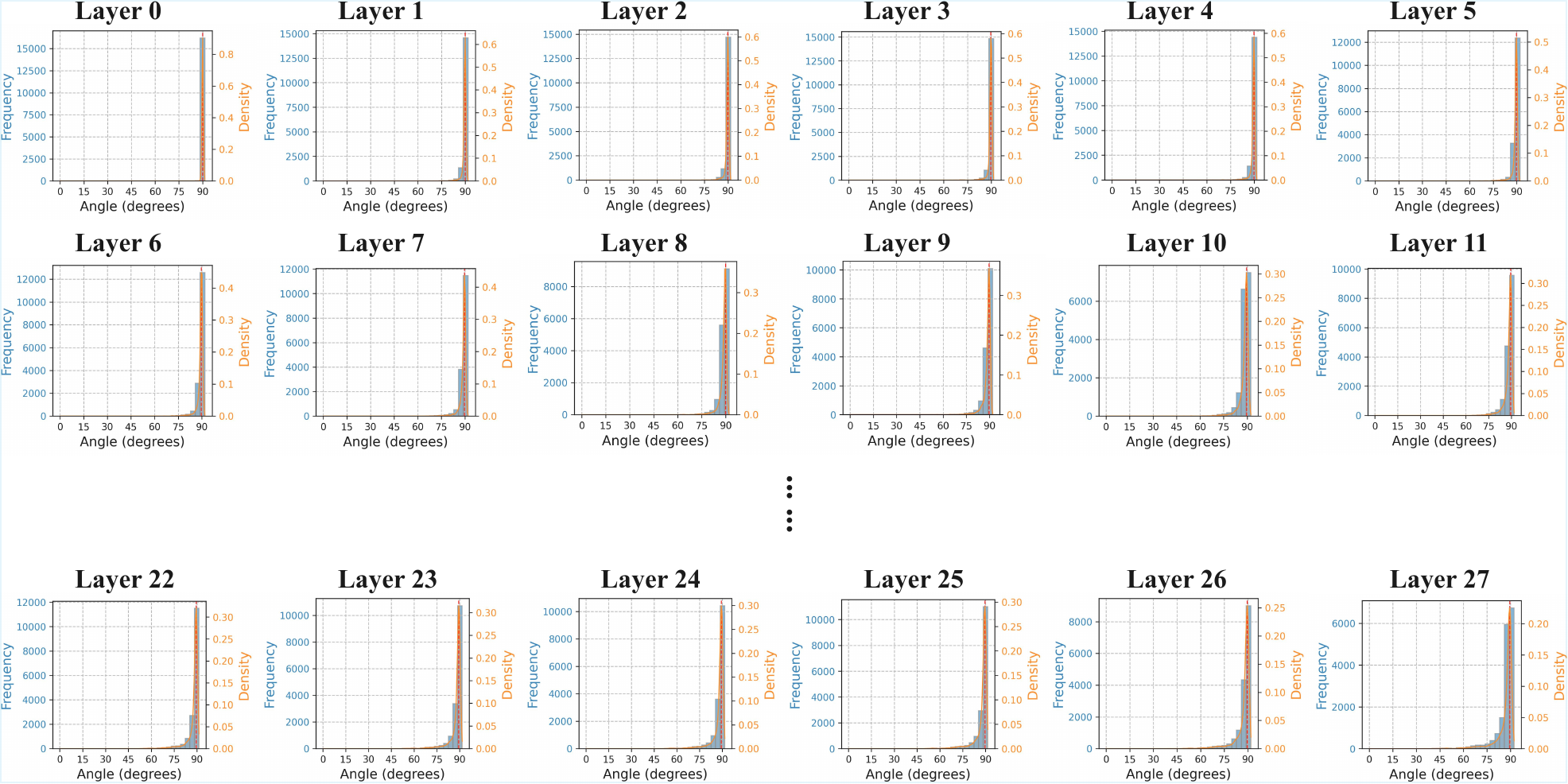}
    \caption{Correction of representation shift on GPT-J-6B.}
  \label{fig:Solving_Representations_Shifting_GPT-J-6B}
\end{figure}

\begin{figure}
    \centering
    \includegraphics[width=0.96\textwidth]{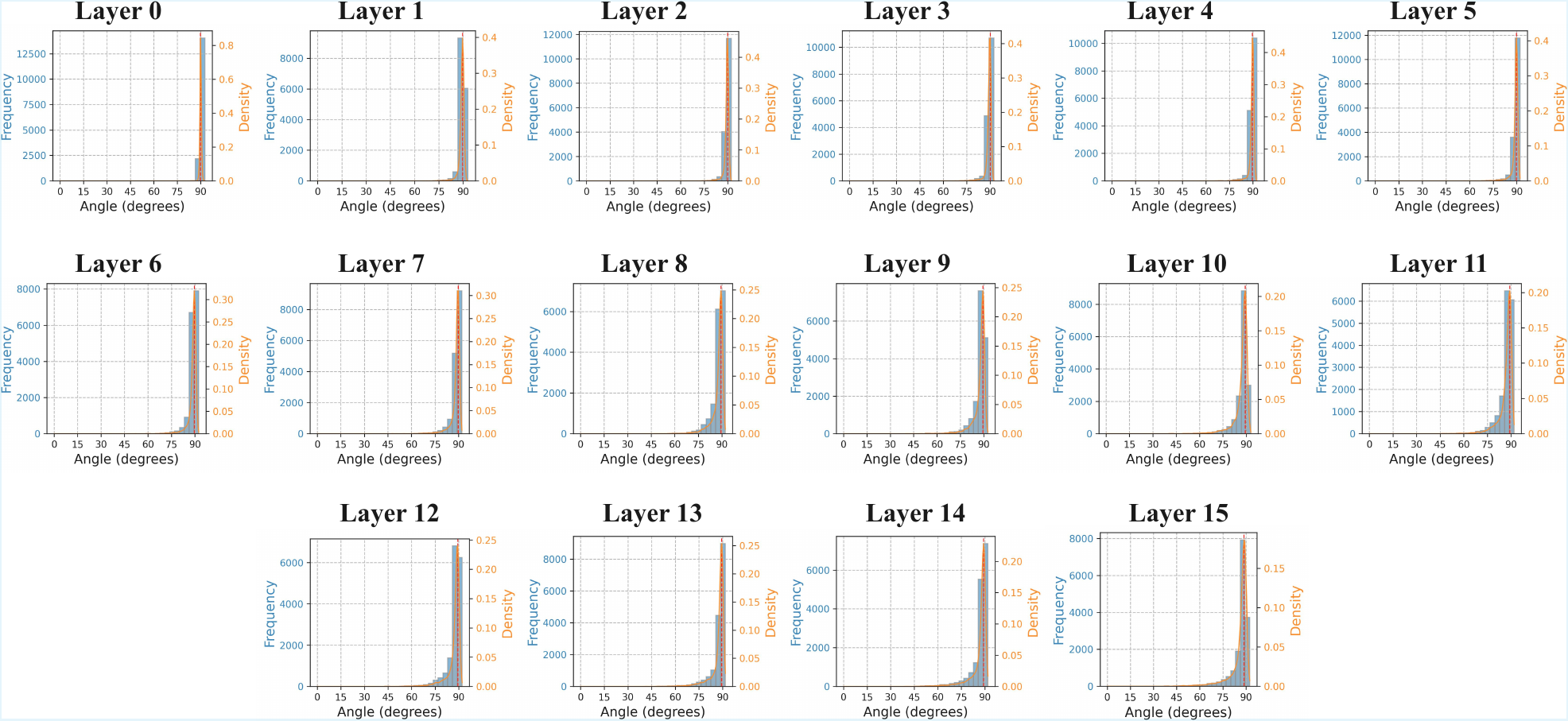}
    \caption{Correction of representation shift on Pythia-1B.}
  \label{fig:Solving_Representations_Shifting_Pythia-1B}
\end{figure}

\begin{figure}
    \centering
    \includegraphics[width=0.96\textwidth]{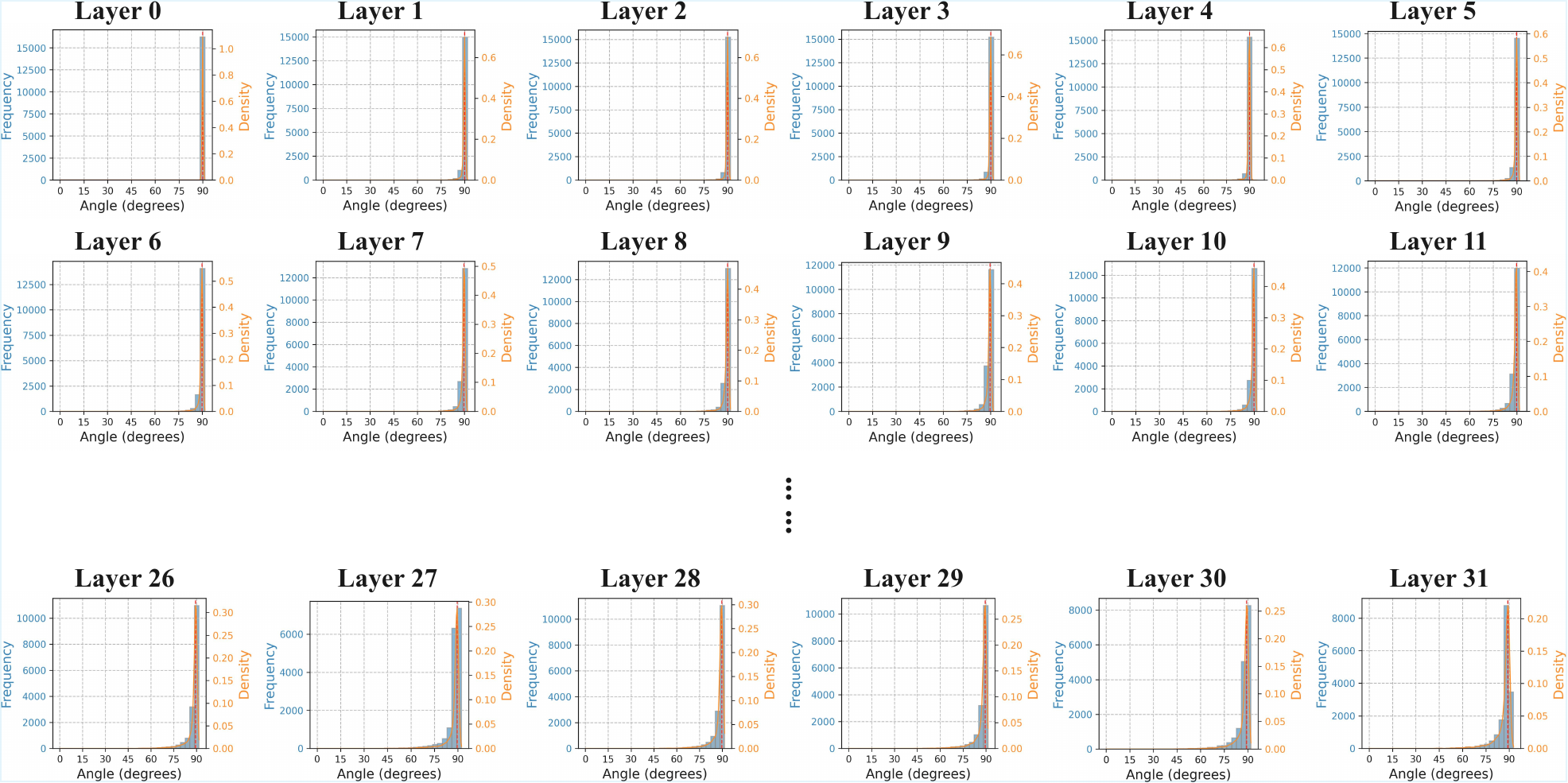}
    \caption{Correction of representation shift on Pythia-6.9B.}
  \label{fig:Solving_Representations_Shifting_Pythia-6.9B}
\end{figure}

\begin{figure}
    \centering
    \includegraphics[width=0.96\textwidth]{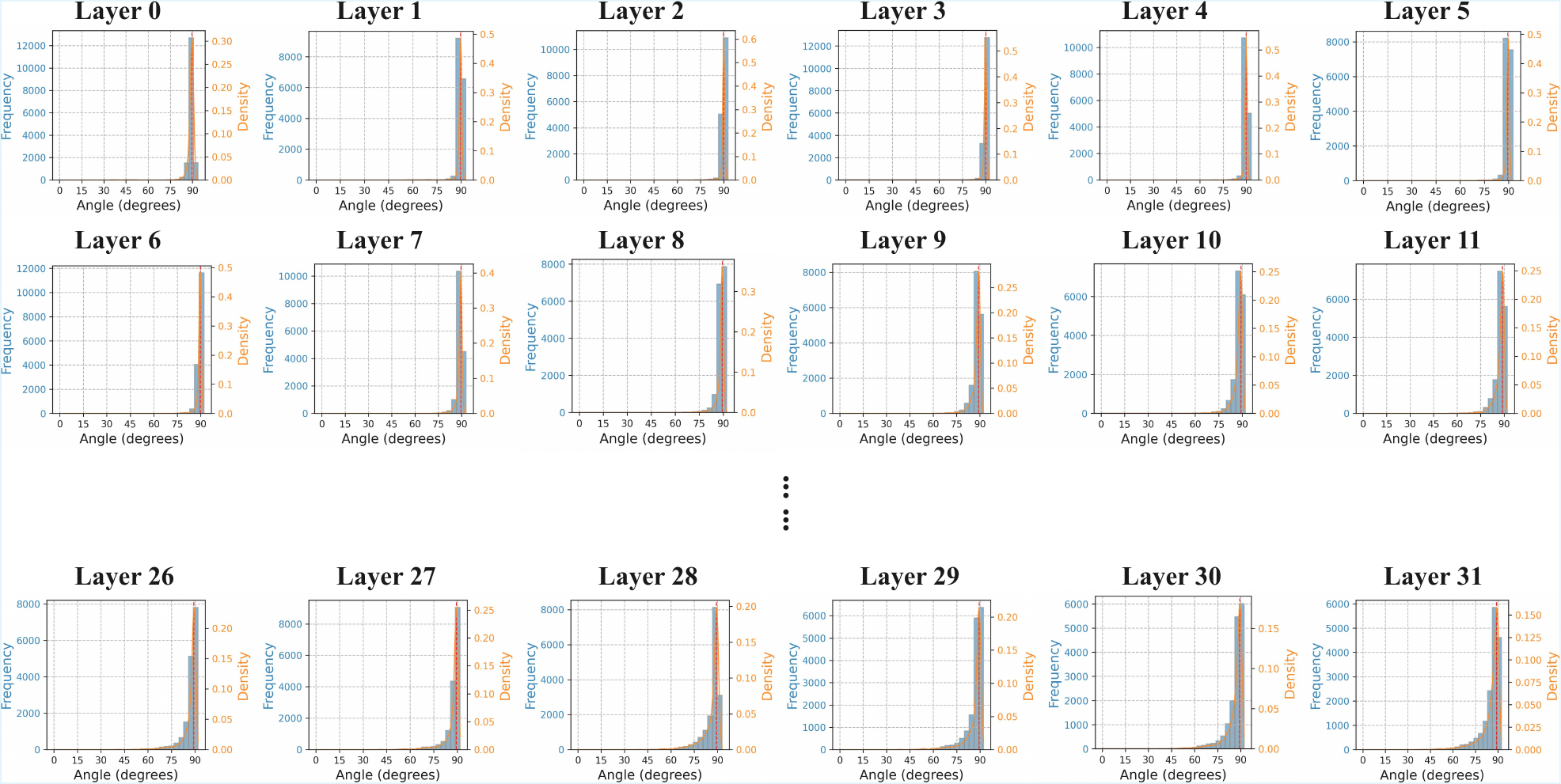}
    \caption{Correction of representation shift on Llama2-7B.}
  \label{fig:Solving_Representations_Shifting_Llama2-7B}
\end{figure}

\begin{figure}
    \centering
    \includegraphics[width=0.96\textwidth]{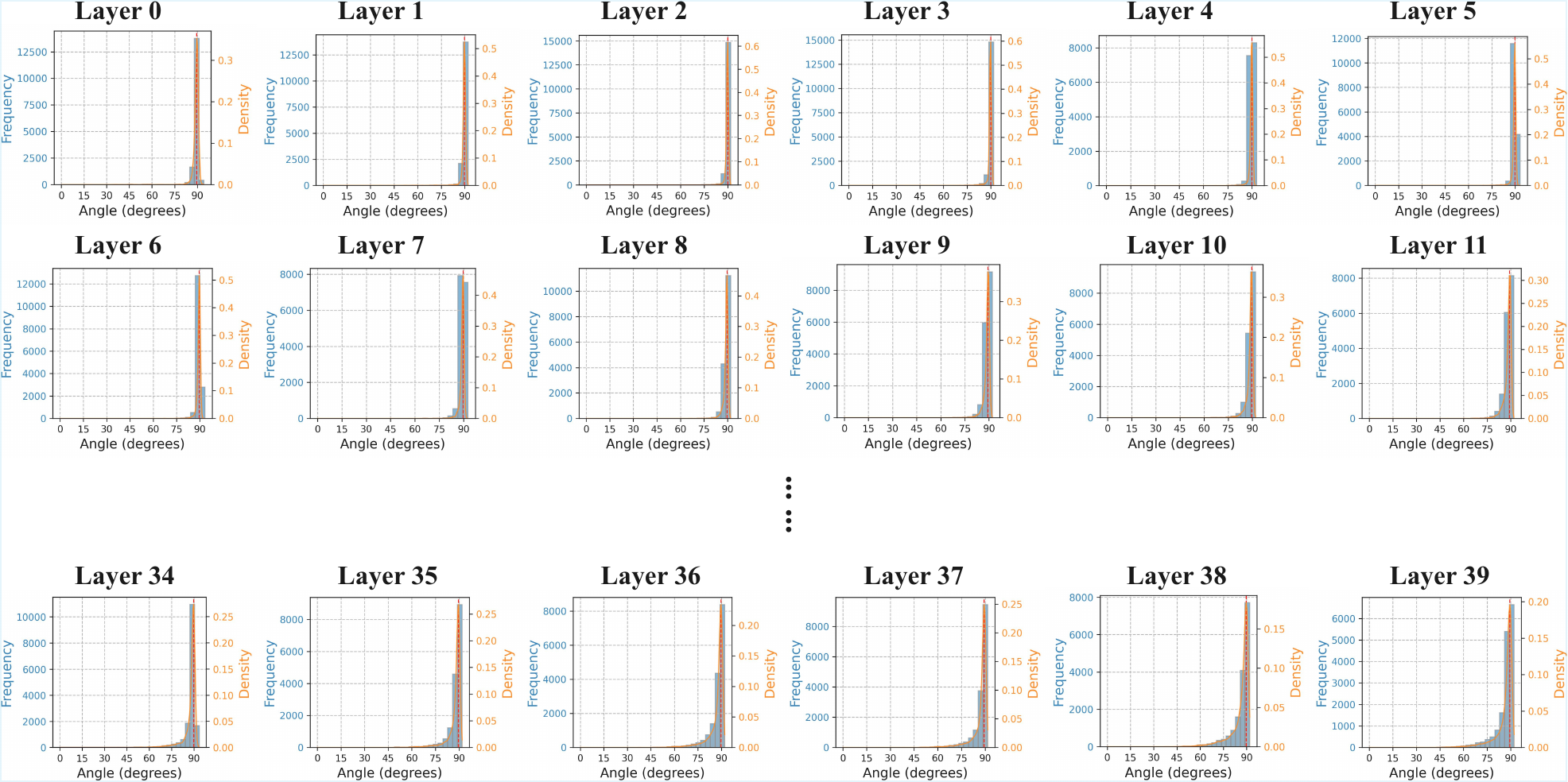}
    \caption{Correction of representation shift on Llama2-13B.}
  \label{fig:Solving_Representations_Shifting_Llama2-13B}
\end{figure}

\begin{figure}
    \centering
    \includegraphics[width=0.96\textwidth]{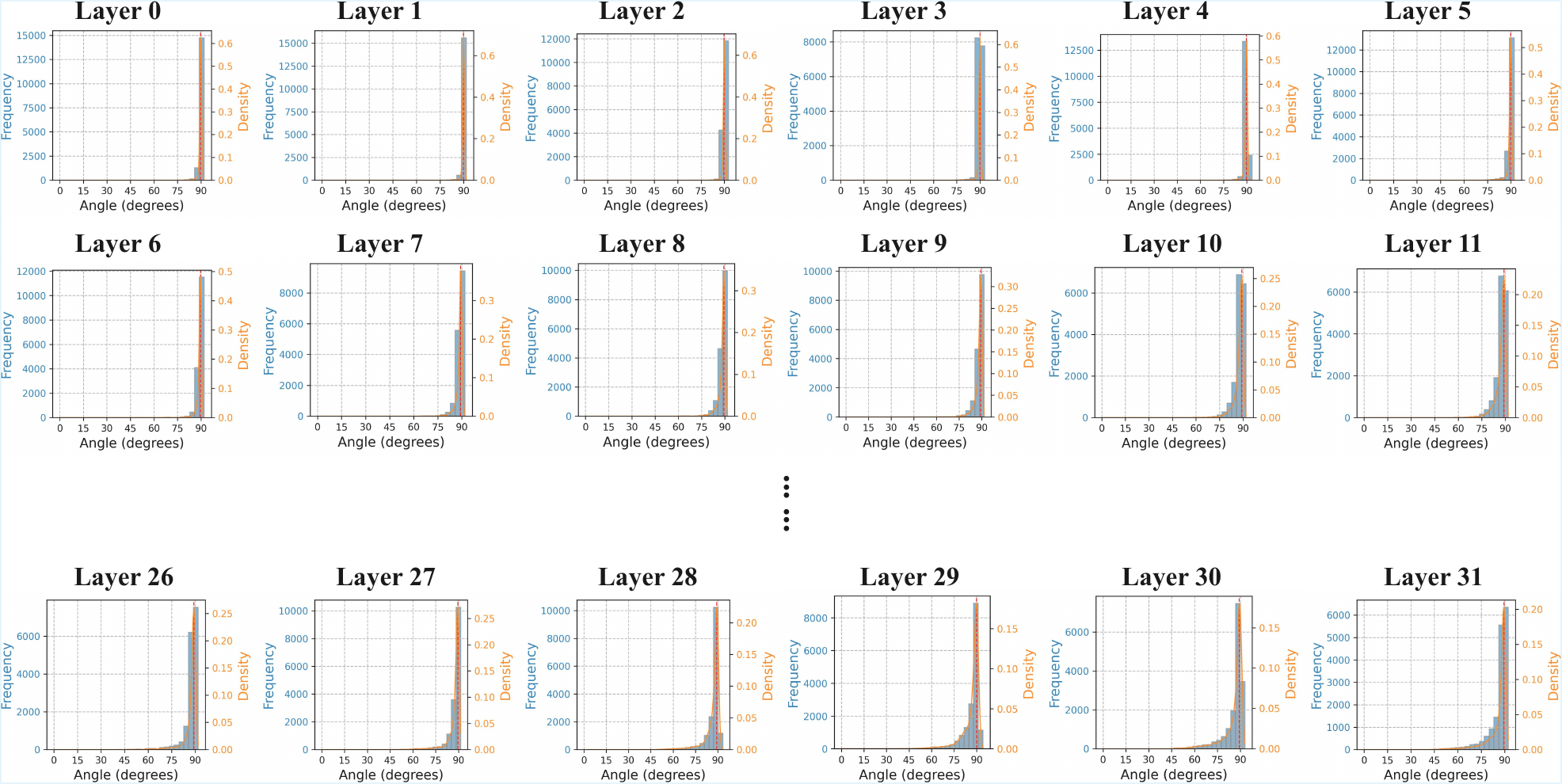}
    \caption{Correction of representation shift on Llama3-8B.}
  \label{fig:Solving_Representations_Shifting_Llama3-8B}
\end{figure}

\begin{figure}
    \centering
    \includegraphics[width=0.96\textwidth]{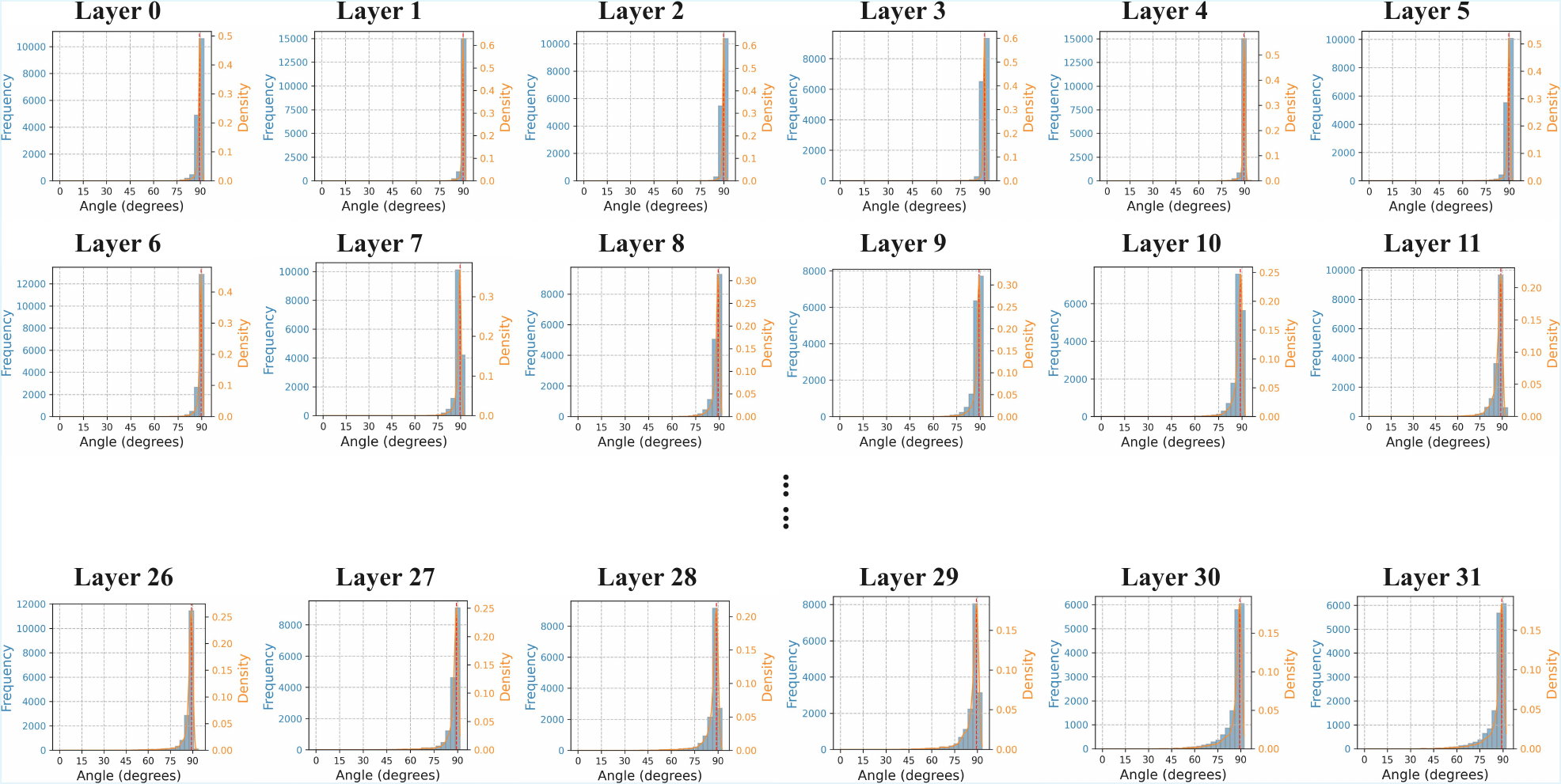}
    \caption{Correction of representation shift on Llama3.1-8B.}
  \label{fig:Solving_Representations_Shifting_Llama3.1-8B}
\end{figure}

\begin{figure}
    \centering
    \includegraphics[width=0.96\textwidth]{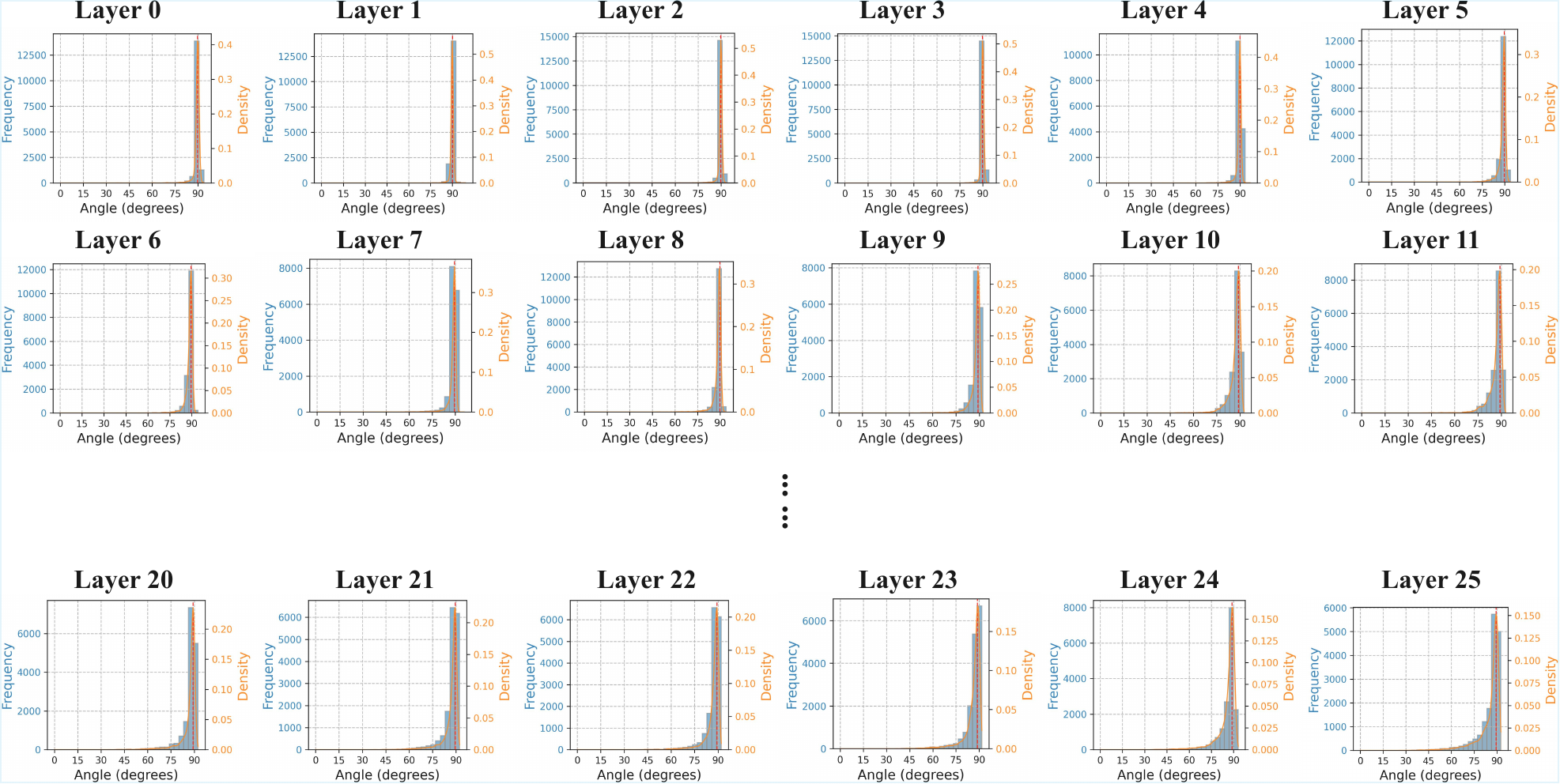}
    \caption{Correction of representation shift on Gemma2-2B.}
  \label{fig:Solving_Representations_Shifting_Gemma2-2B}
\end{figure}

\begin{figure}
    \centering
    \includegraphics[width=0.96\textwidth]{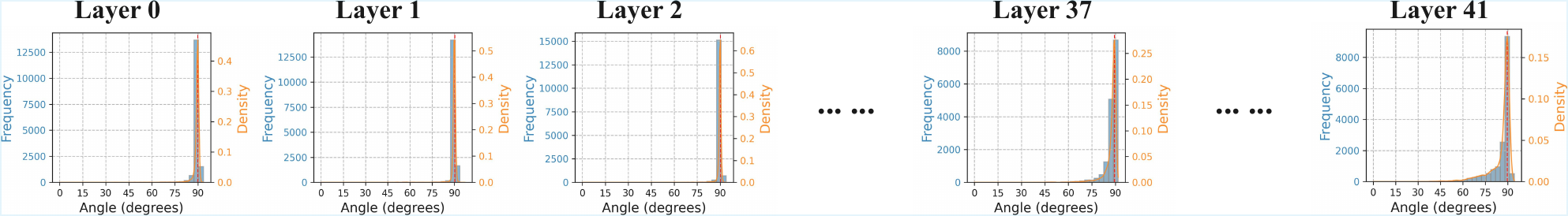}
    \caption{Correction of representation shift on Gemma2-9B.}
  \label{fig:Solving_Representations_Shifting_Gemma2-9B}
\end{figure}


\begin{figure}
    \centering
    \includegraphics[width=0.96\textwidth]{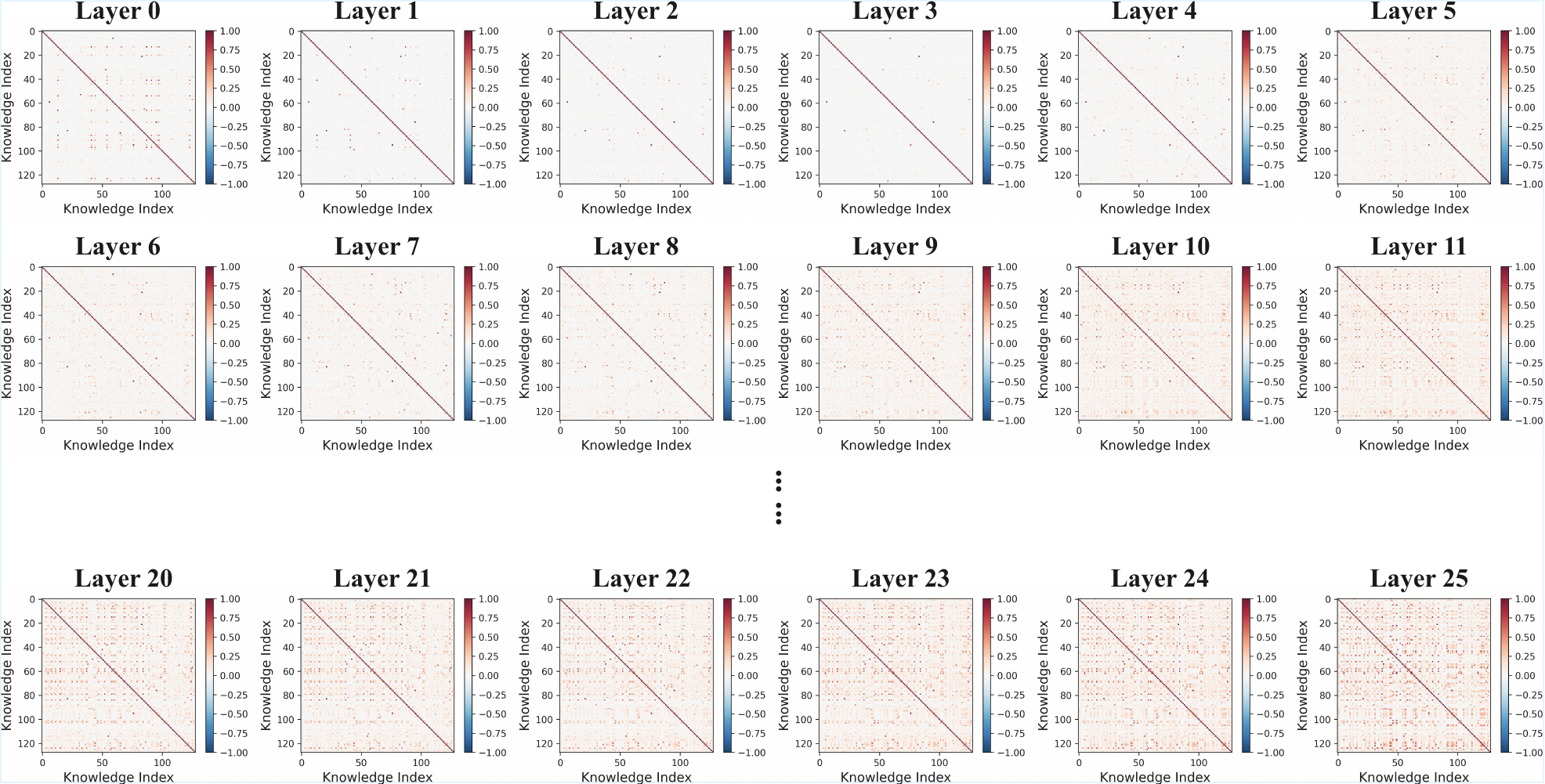}
    \caption{Superposition of activations on Gemma2-2B.}
  \label{fig:Superposition_Gemma2-2B}
\end{figure}

\begin{figure}
    \centering
    \includegraphics[width=0.96\textwidth]{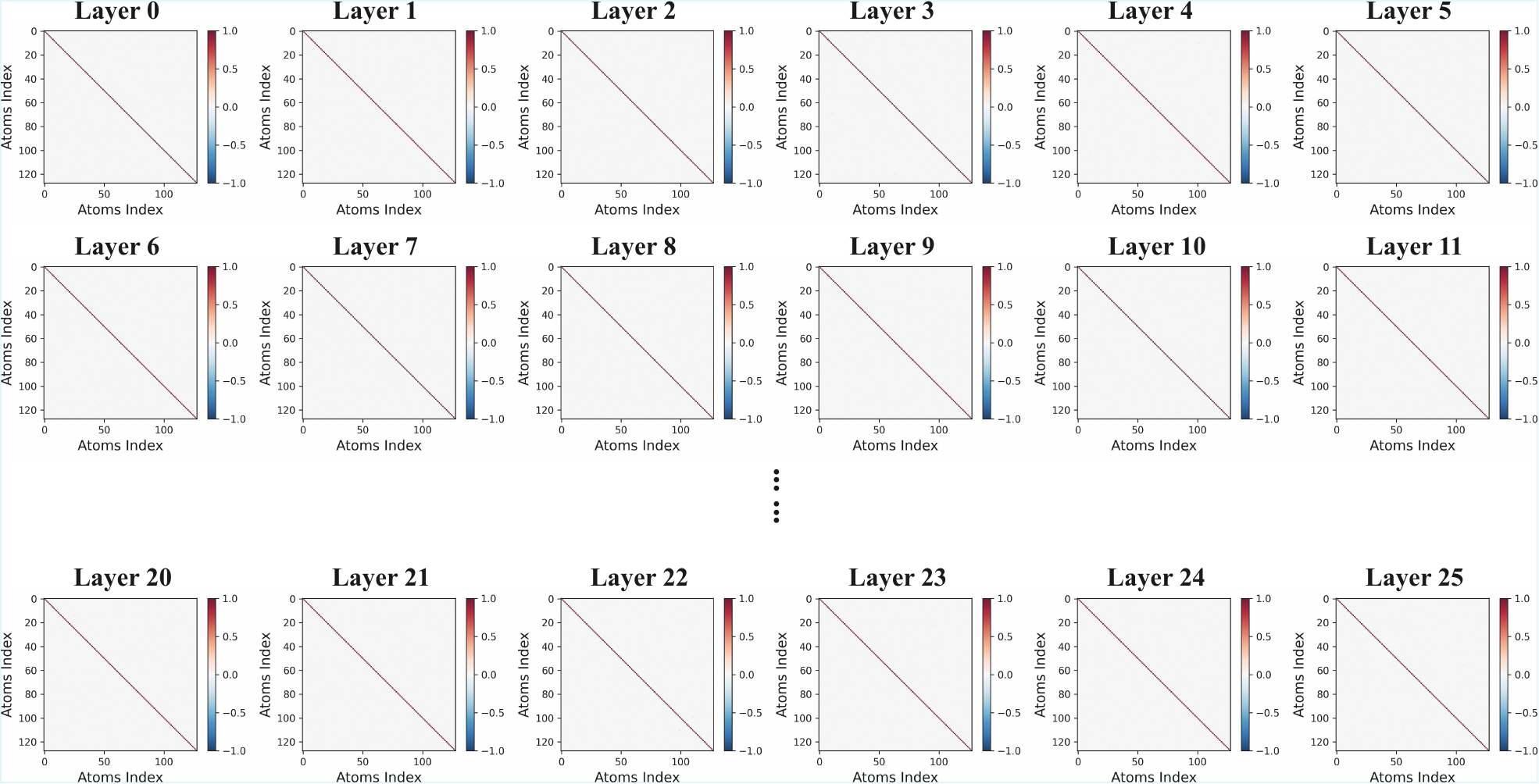}
    \caption{Solving superposition on Gemma2-2B.}
  \label{fig:Solving_Superposition_Gemma2-2B}
\end{figure}

\section{Atoms of LLMs}\label{appendix:Atoms_of_LLMs}

\subsection{Training Paradigm}\label{appendix:Atoms_of_LLMs_Training_Paradigm}

We train threshold-activated sparse autoencoders (TSAEs) on activations extracted by entity knowledge, a setting we term the knowledge atomization task. Unlike the common practice of training on activations from natural corpora, this formulation enables precise control and quantification of data scale, facilitating scalable and systematic study across model and dataset scales. Moreover, entity-induced activations exhibit higher normalized rank (Fig.~\ref{fig:Cumulative_Normalized_Rank}), spanning a broader set of representational dimensions and thus providing richer information.

We also conduct task-level ablations. Specifically, we train TSAEs on activations extracted from natural text (Wikipedia \citep{in2001wikipedia}) and complex mathematical reasoning data (MATH500 \citep{hendrycks2021measuring}), following the same training pipeline as knowledge atomization while matching model and data scales. As shown in Tab.~\ref{tab:task_ablation}, representations obtained from different sources yield consistent results across language models, demonstrating the generality of our findings.

\begin{figure}
    \centering
    \includegraphics[width=0.78\textwidth]{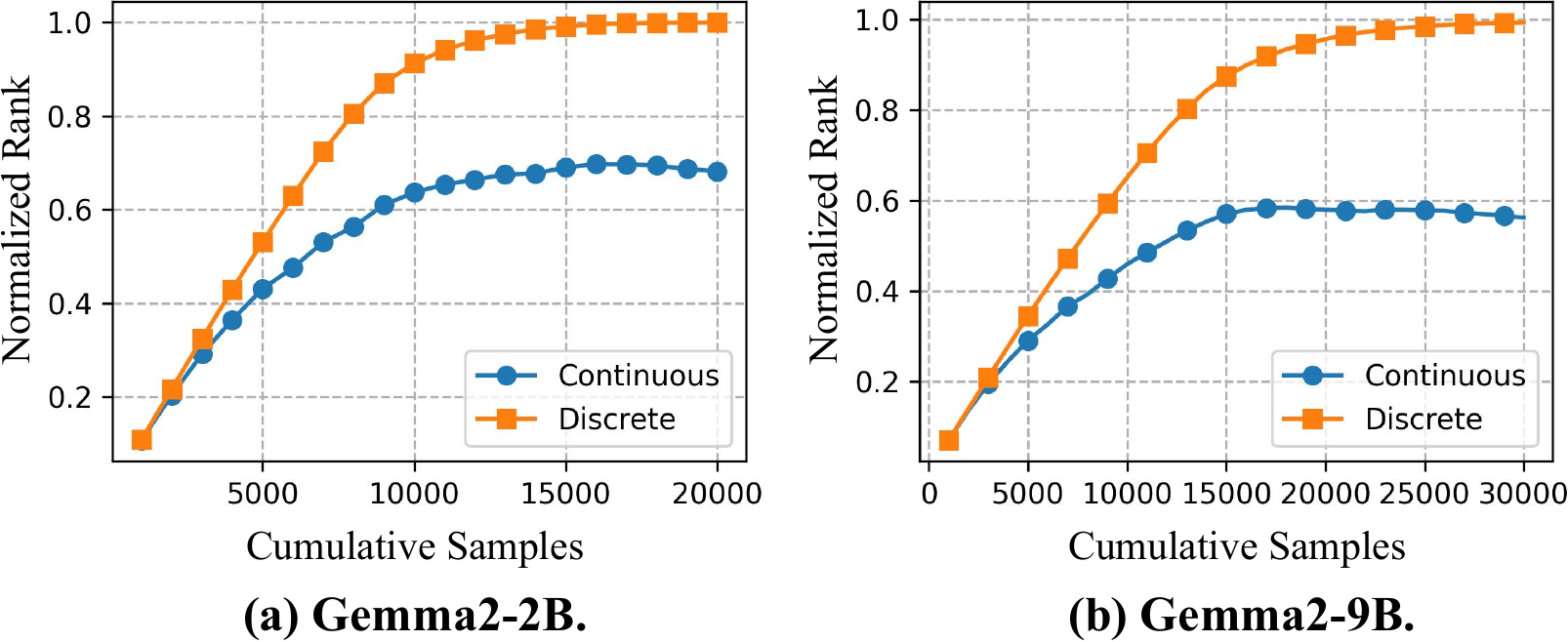}
    \caption{Cumulative normalized rank of (a) Gemma2-2B and (b) Gemma2-9B. Each data point corresponds to the ratio between the rank of the accumulated activation matrix (formed by stacking samples up to that point) and the total dimensionality (i.e., the theoretical maximum rank). Here we illustrate this for randomly selected early layers of Gemma2-2B and Gemma2-9B.}
  \label{fig:Cumulative_Normalized_Rank}
\end{figure}

\begin{table}[t]
\centering
\caption{Comparison of reconstruction quality and sparsity under different training data sources.}
\label{tab:task_ablation}
\begin{tabular}{lcc}
\toprule
Training Data Source & \(R^2\) & \(L_0\) \\
\midrule
General corpora (Wikipedia)       & 99.84\% & 9.36 \\
Complex Reasoning (MATH500)      & 99.86\% & 9.56 \\
Knowledge Atomization (WikiData) & 99.89\% & 8.32 \\
\bottomrule
\end{tabular}
\end{table}

\subsection{Data Collection}\label{appendix:Atoms_of_LLMs_Data_Collection}

In $\S$~\ref{section:Experiments_Representation_Shift} and $\S$~\ref{section:Experiments_TSAE_Capacity_meet_Data_Scale}, we use the WikiData dataset \citep{vrandevcic2014wikidata}, while in $\S$~\ref{section:Experiments_Neurons_or_Features_as_Ideal_Atoms} and $\S$~\ref{section:Experiments_Atoms_of_LLMs} we adopt the CounterFact dataset \citep{meng2022locating}.

Specifically, we collect activations from every layer of Gemma2-2B, Gemma2-9B, and Llama3.1-8B using the subject entities in the corresponding datasets (e.g., “Danielle Darrieux,” “Edwin of Northumbria,” and “Toko Yasuda”).

Activations are collected in a uniform manner: each subject name is used as a prompt, and hooks record activations at the final token of the subject mention, a position previously identified as critical for knowledge recall in language models \citep{meng2022locating}. The resulting activations are aggregated as static training data.

\subsection{Training Details}\label{appendix:Atoms_of_LLMs_Training_Details}

We employ single-layer SAEs with threshold activation, denoted as \(f:\bm{x}\mapsto \hat{\bm{x}} = W_{\mathrm{dec}}\,\sigma\!\left(W_{\mathrm{enc}}\bm{x}\right)\), and train it by minimizing a joint reconstruction–sparsity objective
\begin{equation}\label{equation:loss}
    \mathcal{L}(\bm{x})= \underbrace{\|\bm{x}-\hat{\bm{x}}\|_2^2}_{\mathcal{L}_{\mathrm{reconstruct}}}+ \lambda\,\underbrace{\|\sigma(\bm{z})\|_1}_{\mathcal{L}_{\mathrm{sparsity}}},
\end{equation}
where \(\bm{z}=W_{\mathrm{enc}}\bm{x}\), and \(\sigma\) is coordinate-wise JumpReLU activation \citep{rajamanoharan2024jumping},
\begin{equation}
    \big(\sigma(\bm{z})\big)_{i} = \begin{cases} 0, & z_i < \tau_i, \\z_i, & z_i \ge \tau_i,\end{cases}\quad \boldsymbol{\tau} = (\tau_i)_i.
\end{equation}

The key hyperparameters are the sparsity coefficient \(\lambda\) in the loss function (Eq.~\ref{equation:loss}) and the threshold initialization. We fix \(\lambda = 0.1\) (later shown to be insensitive) and initialize the threshold at 0.001 (or 0.0001), which provides a good trade-off between training efficiency and effectiveness: smaller initial thresholds facilitate satisfying the support-separation condition but substantially increase training time, whereas 0.001 serves as a stable and reliable default in our experiments. During training, we employ the straight-through estimator \citep{rajamanoharan2024jumping} to approximate gradients at the non-differentiable threshold.

We select the final model as the checkpoint on the Pareto front that optimally balances reconstruction error and sparsity. Fig.~\ref{fig:Pareto_Front_Gemma2-2B} illustrates the Pareto front for Gemma2-2B.

The specific computational cost is as follows:
\begin{itemize}
    \item Gemma2-2B (per layer): $\sim$24 GPU-hours on RTX 3090-24G (on average);
    \item Gemma2-9B (per layer): $\sim$56 GPU-hours on A100-80G (on average);
    \item Llama3.1-8B (per layer): $\sim$58 GPU-hours on A100-80G (on average);
    \item Largest TSAE trained in this work (Fig.~\ref{fig:Scaling_Saes}, top-right): $\sim$135 GPU-hours on A100-80G.
\end{itemize}

A minor training issue was observed in layers 30 and 31 of Llama 3.1-8B, where unusually large activations caused optimization to fail. Consequently, these layers are omitted from the reported results. This behavior is likely related to their proximity to the output, where activations may drive next-token prediction rather than encode entity-specific information. Further analysis indicates that the failure is mainly due to abnormally high data variance in these layers; normalizing the training data by its distributional standard deviation, followed by a simple reparameterization of the trained model, largely resolves the issue. By contrast, Gemma 2-2B and Gemma 2-9B did not exhibit this problem, possibly because their extensive use of RMSNorm mitigates such activation outliers.

\begin{figure}
    \centering
    \includegraphics[width=0.96\textwidth]{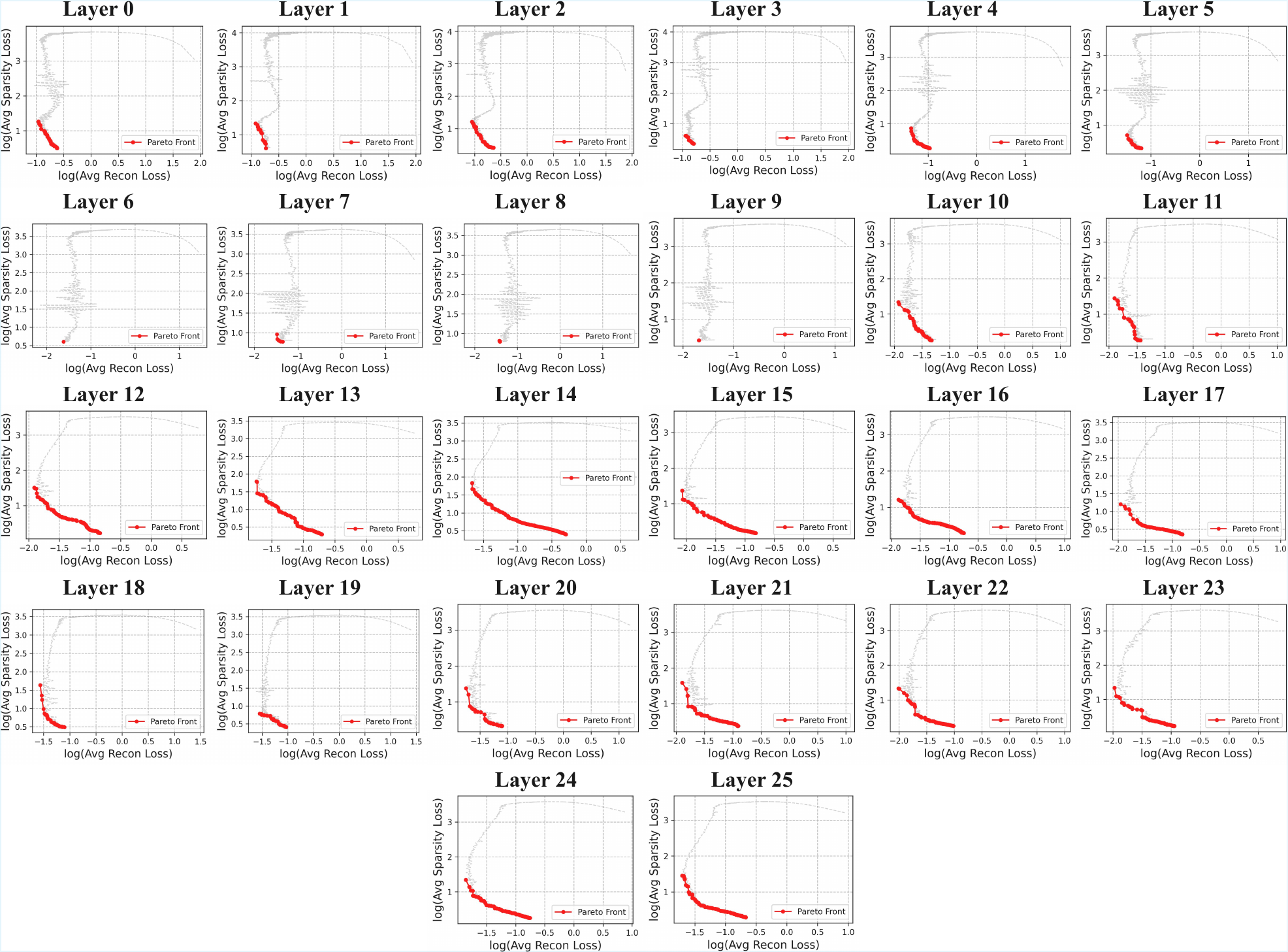}
    \caption{Pareto front during training on Gemma2-2B.}
  \label{fig:Pareto_Front_Gemma2-2B}
\end{figure}

\subsection{Baseline Details}\label{appendix:Atoms_of_LLMs_Baselines_Details}

The primary baselines used in this work are GemmaScope and LlamaScope. GemmaScope provides SAEs of widths 16k and 65k trained on the MLP layers of Gemma2-2B, as well as SAEs of widths 16k and 131k trained on the MLP layers of Gemma2-9B. LlamaScope offers SAEs with expansion factors of 8× and 32× trained on the MLP layers of Llama3.1-8B. Both GemmaScope and LlamaScope are widely regarded as open-source tools for feature extraction.

It is important to emphasize that these models are trained on activations derived from continuous text corpora. We use them as baselines not to demonstrate superior performance of our SAEs, but to highlight that feature-based reconstruction of raw activations remains unreliable in practice, whereas our results show that internal representations of language models can be reconstructed with high fidelity.

\subsection{Evaluation Details}\label{appendix:Atoms_of_LLMs_Evaluation_Details}

To practically assess the stability of identified representational units, we introduce quantile-based statistics that correspond to the prior conditions in Theorem~\ref{theorem:Uniqueness_of_Sparse_Representation_and_Exact_Recovery}. Specifically, we define two statistics:
\begin{itemize}
    \item \textbf{Quantile sparsity \(K_q\)}: The quantile of sparsity \(K_q\) is defined as
    \begin{equation}
        K_q = \inf \left\{k \in \mathbb{N} : \mathbb{P}_{\bm \delta \sim \mathcal{P}_{\Delta}}\left(K \le k\right) \ge q \right\},
    \end{equation}
    where \(\bm \delta\) is a coefficient vector sampled from the distribution \(\mathcal{P}_{\Delta}\), and the random variable \(K := \|\bm \delta\|_0\) represents the sparsity of the sampled coefficient vector. In simple terms, the quantile sparsity \(K_q\) indicates that at least \(q\) of the samples have sparsity no greater than \(K_q\).
    \item \textbf{Quantile coherence \(\mu_q\)}: Similarly, the quantile of coherence \(\mu_q\) is defined as
    \begin{equation}
        \mu_q = \inf \left\{\mu \ge 0 : \mathbb{P}_{(\mathcal{I}, \mathcal{J})}\left(C \le \mu | \tilde{D}\right) \ge q \right\},
    \end{equation}
    where \((\mathcal{I}, \mathcal{J})\) is uniformly sampled from all unordered pairs of indices (\(\mathcal{I} \neq \mathcal{J}\)), and the random variable \(C := |\langle \tilde{\boldsymbol{d}}_\mathcal{I}, \tilde{\boldsymbol{d}}_\mathcal{J} \rangle|\) represents the coherence between two atoms. In simple terms, this means that the probability of randomly selecting a pair of different atoms with coherence no greater than \(\mu_q\) is at least \(q\).
\end{itemize}
Based on these definitions, for the supports of most samples, if the condition \(\mu_q < \frac{1}{2K_q-1}\) holds, we can conclude that at least \(q\) proportion of the samples satisfy the sufficient conditions for uniqueness and recoverability.

To determine the maximal quantile \(q^*\) satisfying the theoretical criterion, we perform a binary search over the interval \([0,0.999999]\) for the quantile parameter \(\alpha\). At each iteration we compute the linear quantiles
\begin{equation}
    \mu_\alpha := \operatorname{Quantile}(\{\mu\}, \alpha),\quad K_\alpha := \operatorname{Quantile}(\{K\}, \alpha),
\end{equation}
and test whether \(\mu_\alpha < \frac{1}{2K_\alpha - 1}\) holds. If the condition is satisfied, the lower bound of the search interval is updated to \(\alpha\); otherwise the upper bound is reduced. Upon convergence, the maximal \(\alpha\) obtained is taken as the desired quantile \(q^*\), together with the corresponding values of \(\mu_\alpha\) and \(K_\alpha\).

Note that verifying Theorem~\ref{theorem:Uniqueness_of_Sparse_Representation_and_Exact_Recovery} requires the equality \(\tilde{D}\bm x=\tilde{\bm m}\). However, as shown in Fig.~\ref{fig:neurons_features_ideal_atoms_radar_chart}, features generally fail to achieve reliable reconstruction, so the quantile \(q\) obtained from the condition \(\mu_q < \frac{1}{2K_q-1}\) serves only as an ideal upper bound. In contrast, the learned atoms satisfy reliable reconstruction, and 99.85\% of atoms meet \(\mu_q < \frac{1}{2K_q-1}\) on average, confirming their favorable properties.

For further detail, Tabs.~\ref{tab:atom_quantiles_Gemma2-2B}-\ref{tab:atom_quantiles_Llama3.1-8B} report the corresponding values of \(R^2\) and \(q^*\) for identified units of Gemma2-2B, Gemma2-9B and Llama3.1-8B. In all three models, we primarily use TSAEs with JumpReLU activations \citep{erichson2019jumprelu,rajamanoharan2024jumping} to identify atoms. For comparison, on Gemma2-2B we also train standard SAEs with ReLU activations \citep{cunningham2023sparse} and find that they fail to identify units that satisfy the criteria of ideal atoms as fundamental representational units. Moreover, control experiments with ReLU SAEs of varying capacities show that increasing capacity does not improve performance (Tab.~\ref{tab:relu_sae_capacity}), which is also consistent with our theoretical expectation.

\begin{table}[htbp]
\centering
\caption{Faithfulness and stability across layers on Gemma2-2B.}
\begin{tabular}{c cc cc}
\toprule
& \multicolumn{2}{c}{TSAEs with JumpReLU} & \multicolumn{2}{c}{SAEs with ReLU} \\
\cmidrule(lr){2-3}\cmidrule(lr){4-5}
Layer 
& $R^2$ & $q^*$
& $R^2$ & $q^*$\\
\midrule
0  & 0.9986 & 0.9974& -  &  -   \\
1  & 0.9984 & 0.9978& -  &  -   \\
2  & 0.9987 & 0.9978& -  &  -   \\
3  & 0.9992 & 0.9988& -  &  -   \\
4  & 0.9994 & 0.9991& -  &  -   \\
5  & 0.9996 & 0.9983& 0.9680  &  0.8263   \\
6  & 0.9995 & 0.9983& 0.9510  &  0.6740   \\
7  & 0.9993 & 0.9994& 0.9624  &  0.6712   \\
8  & 0.9996 & 0.9976& 0.9650  &  0.6437   \\
9  & 0.9995 & 0.9987& 0.9383  &  0.5648   \\
10 & 0.9992 & 0.9983& 0.9133  &  0.4662   \\
11 & 0.9992 & 0.9979& 0.9179  &  0.4366   \\
12 & 0.9991 & 0.9911& 0.9129  &  0.4516   \\
13 & 0.9973 & 0.9960& 0.9104  &  0.4167   \\
14 & 0.9989 & 0.9993& -  &  -   \\
15 & 0.9988 & 0.9989& -  &  -   \\
16 & 0.9992 & 0.9972& -  &  -   \\
17 & 0.9994 & 0.9991& -  &  -   \\
18 & 0.9996 & 0.9945& -  &  -   \\
19 & 0.9997 & 0.9980& -  &  -   \\
20 & 0.9989 & 0.9989& -  &  -   \\
21 & 0.9997 & 0.9964& -  &  -   \\
22 & 0.9995 & 0.9922& -  &  -   \\
23 & 0.9994 & 0.9979& -  &  -   \\
24 & 0.9994 & 0.9982& -  &  -   \\
25 & 0.9993 & 0.9943& -  &  -   \\
\bottomrule
\end{tabular}
\label{tab:atom_quantiles_Gemma2-2B}
\end{table}

\begin{table}[htbp]
\centering
\caption{Faithfulness and stability across layers on Gemma2-9B.}
\begin{tabular}{c cc}
\toprule
& \multicolumn{2}{c}{TSAEs with JumpReLU} \\
\cmidrule(lr){2-3}
Layer 
& $R^2$ & $q^*$\\
\midrule
0  & 0.9996  &  0.9915  \\
1  & 0.9993  &  0.9992  \\
2  & 0.9995  &  0.9981  \\
3  & 0.9996  &  0.9975  \\
4  & 0.9996  &  0.9985  \\
5  & 0.9997  &  0.9999  \\
6  & 0.9993  &  0.9961  \\
7  & 0.9996  &  0.9995  \\
8  & 0.9996  &  0.9996  \\
9  & 0.9996  &  0.9997  \\
10 & 0.9997  &  0.9997  \\
11 & 0.9996  &  0.9994  \\
12 & 0.9994  &  0.9996  \\
13 & 0.9993  &  0.9991  \\
14 & 0.9989  &  0.9991  \\
15 & 0.9992  &  0.9997  \\
16 & 0.9990  &  0.9997  \\
17 & 0.9994  &  0.9996  \\
18 & 0.9995  &  0.9995  \\
19 & 0.9990  &  0.9996  \\
20 & 0.9991  &  0.9995  \\
21 & 0.9991  &  0.9994  \\
22 & 0.9992  &  0.9993  \\
23 & 0.9993  &  0.9991  \\
24 & 0.9993  &  0.9990  \\
25 & 0.9994  &  0.9988  \\
26 & 0.9964  &  0.9974  \\
27 & 0.9997  &  0.9980  \\
28 & 0.9994  &  0.9997  \\
29 & 0.9993  &  0.9993  \\
30 & 0.9997  &  0.9982  \\
31 & 0.9995  &  0.9982  \\
32 & 0.9996  &  0.9982  \\
33 & 0.9998  &  0.9986  \\
34 & 0.9998  &  0.9987  \\
35 & 0.9997  &  0.9997  \\
36 & 0.9997  &  0.9991  \\
37 & 0.9995  &  0.9992  \\
38 & 0.9993  &  0.9996  \\
39 & 0.9990  &  0.9998  \\
40 & 0.9988  &  0.9999  \\
41 & 0.9995  &  0.9951  \\
\bottomrule
\end{tabular}
\label{tab:atom_quantiles_Gemma2-9B}
\end{table}

\begin{table}[htbp]
\centering
\caption{Faithfulness and stability across layers on Llama3.1-8B.}
\begin{tabular}{c cc}
\toprule
& \multicolumn{2}{c}{TSAEs with JumpReLU} \\
\cmidrule(lr){2-3}
Layer 
& $R^2$ & $q^*$\\
\midrule
0  & 0.9985   &  0.9968   \\
1  & 0.9998   &  0.9996   \\
2  & 0.9930   &  0.9998   \\
3  & 0.9945   &  0.9999   \\
4  & 0.9992   &  0.9999   \\
5  & 0.9992   &  0.9998   \\
6  & 0.9971   &  0.9999   \\
7  & 0.9961   &  0.9999   \\
8  & 0.9992   &  0.9999   \\
9  & 0.9988   &  0.9999   \\
10 & 0.9989   &  0.9998   \\
11 & 0.9987   &  0.9999   \\
12 & 0.9993   &  0.9997   \\
13 & 0.9970   &  0.9999   \\
14 & 0.9992   &  0.9999   \\
15 & 0.9986   &  0.9999   \\
16 & 0.9989   &  0.9999   \\
17 & 0.9992   &  0.9998   \\
18 & 0.9992   &  0.9994   \\
19 & 0.9993   &  0.9998   \\
20 & 0.9991   &  0.9998   \\
21 & 0.9993   &  0.9996   \\
22 & 0.9997   &  0.9992   \\
23 & 0.9995   &  0.9996   \\
24 & 0.9996   &  0.9993   \\
25 & 0.9990   &  0.9999   \\
26 & 0.9995   &  0.9997   \\
27 & 0.9993   &  0.9999   \\
28 & 0.9982   &  0.9993   \\
29 & 0.9979   &  0.9940   \\
\bottomrule
\end{tabular}
\label{tab:atom_quantiles_Llama3.1-8B}
\end{table}

\begin{table}[t]
\centering
\caption{Performance of ReLU SAEs with different capacities on Gemma2-2B at layer 10.}
\label{tab:relu_sae_capacity}
\begin{tabular}{lcccc}
\toprule
Capacity & 4$\times$ & 5$\times$ & 6$\times$ & 7$\times$ \\
\midrule
$R^2$ & 0.9183 & 0.9173 & 0.9150 & 0.9143 \\
$q^*$ & 0.4768 & 0.4655 & 0.4514 & 0.4532 \\
\bottomrule
\end{tabular}
\end{table}

\subsection{Experimental Analysis}\label{appendix:Atoms_of_LLMs_Experimental_Analysis}

Notably, the training process is largely insensitive to hyperparameters: using sparsity coefficients \(\lambda \in \{0.01, 0.1, 1\}\) yields nearly identical learning curves (Fig.~\ref{fig:Robust_Train_Loss}), indicating strong robustness. This suggests that high-fidelity reconstruction primarily reflects the intrinsic sparsifiability of the representations, rather than careful hyperparameter tuning.

The encoder and decoder of SAEs converge to alignment under atomic inner product, namely parameterization of \(W_{\!dec}\!\!=\!\!D\) and \(W_{\!enc}\!\!=\!\!D^\top\! \tilde S\), consistent with Theorem~\ref{theorem:Identifiability_of_SAEs}, as shown in Figs.~\ref{fig:Encoder_After_Training}-\ref{fig:Encoder_After_Training_Llama3.1-8B} for Gemma2-2B, Gemma2-9B and Llama3.1-8B.

By Definition~\ref{definition:Atoms}, atoms must satisfy approximate orthogonality under the normalized atomic inner product (NAIP), ensuring their mutual distinguishability. The NAIP among all atoms can be computed by directly evaluating the matrix \(G = \tilde{D}^\top \tilde{D}\), with a more practical procedure, similar to Corollary~\ref{corollary:NAIP}, given by
\begin{equation}
    G = \frac{D^\top S D}{\sqrt{\text{diag}(D^\top S D)} \times \sqrt{\text{diag}(D^\top S D)}},
\end{equation}
where \(S = (DD^\top)^{-1}\), \(\text{diag}(D^\top S D)\) denotes the diagonal of \(D^\top S D\), \(\sqrt{\text{diag}(D^\top S D)}\) denotes its element-wise square root, and \(\times\) indicates the outer product. If the vectors learned by the SAEs exhibit atomicity, the off-diagonal elements \(G_{ij} = \langle \tilde{\bm{d}_i}, \tilde{\bm{d}_j} \rangle\) should cluster near zero with very small variance, demonstrating approximate orthogonality, while the diagonal entries are normalized. 

As shown in Figs.~\ref{fig:Naip_Distribution_Gemma2-2B_Full}-\ref{fig:Naip_Distribution_Llama3.1-8B_Full}, across all layers of Gemma2-2B, Gemma2-9B, and Llama3.1-8B, the off-diagonal elements of the matrices \(D\) are tightly concentrated near zero, closely matching the theoretical Dirac delta distribution. This accords with Definition~\ref{definition:Approximately_Orthogonal_Atoms}: although strict orthogonality is unattainable, sparsity drive convergence to approximately orthogonal atoms.

\begin{figure}
    \centering
    \includegraphics[width=0.88\textwidth]{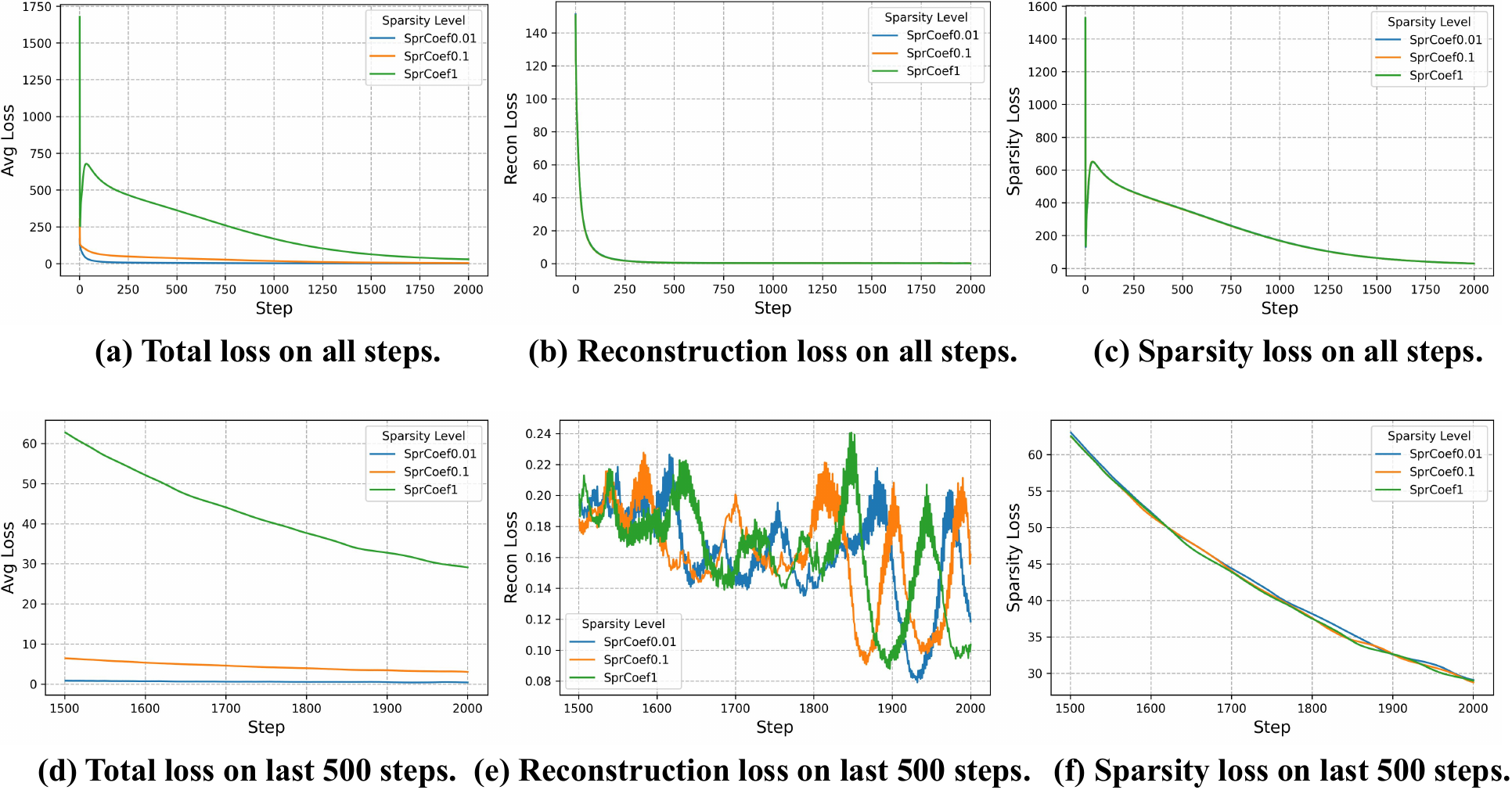}
    \caption{Training loss is robust to hyperparameter selection on \(\lambda\), maintaining stable performance across different configurations.}
  \label{fig:Robust_Train_Loss}
\end{figure}


\begin{figure}
    \centering
    \includegraphics[width=0.88\textwidth]{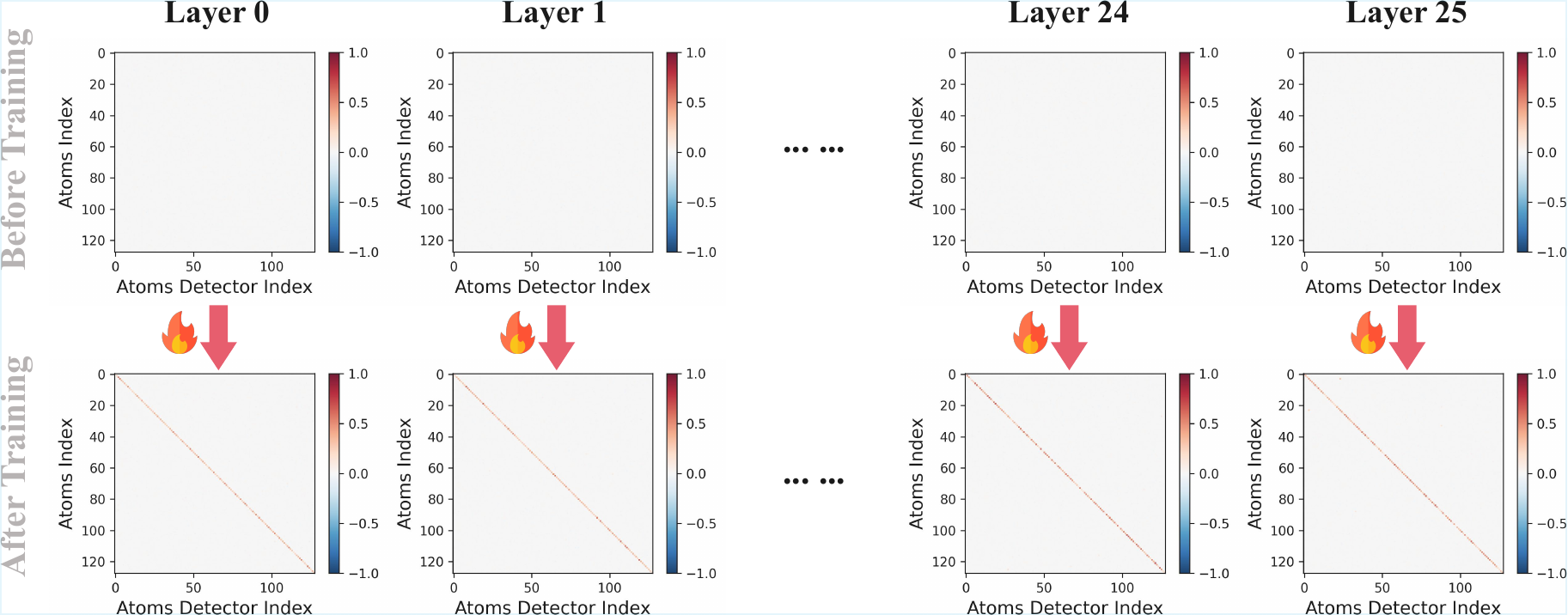}
    \caption{Spontaneous alignment between the encoder and decoder during training on Gemma2-2B.}
  \label{fig:Encoder_After_Training}
\end{figure}

\begin{figure}
    \centering
    \includegraphics[width=0.88\textwidth]{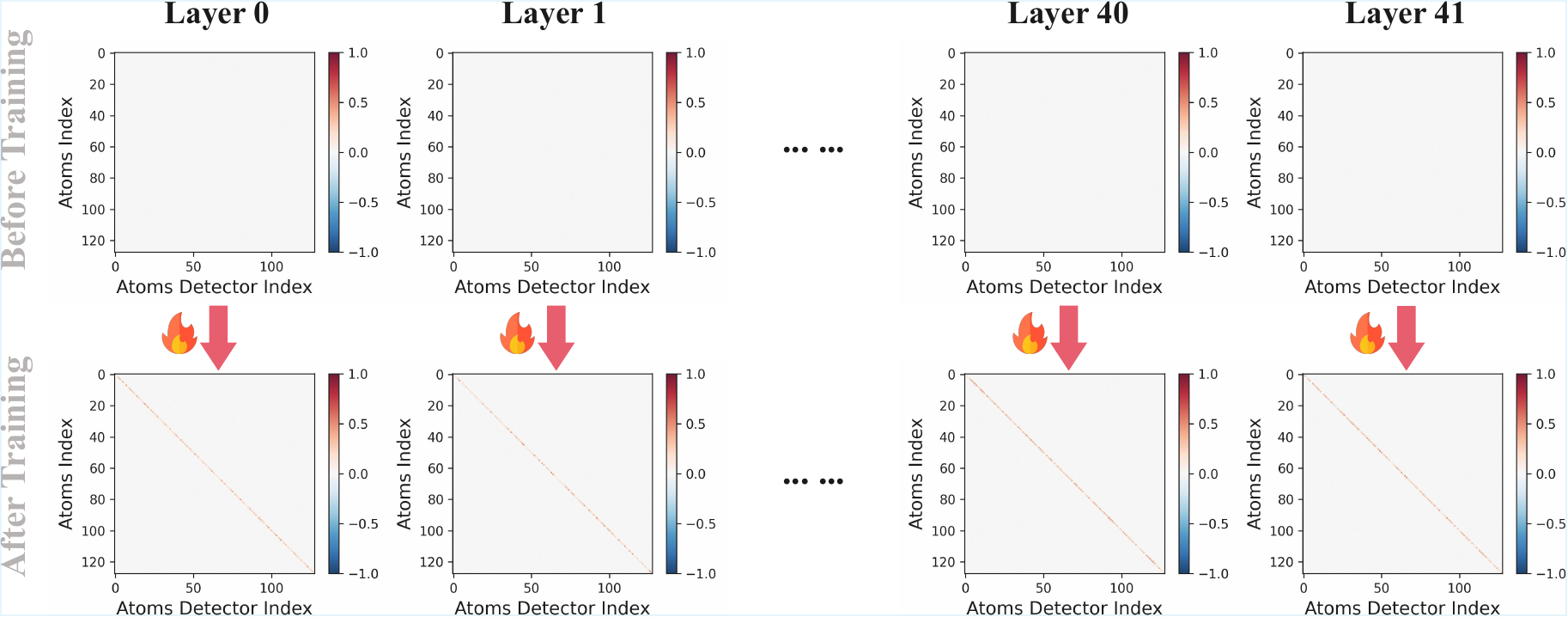}
    \caption{Spontaneous alignment between the encoder and decoder during training on Gemma2-9B.}
  \label{fig:Encoder_After_Training_Gemma2-9B}
\end{figure}

\begin{figure}
    \centering
    \includegraphics[width=0.88\textwidth]{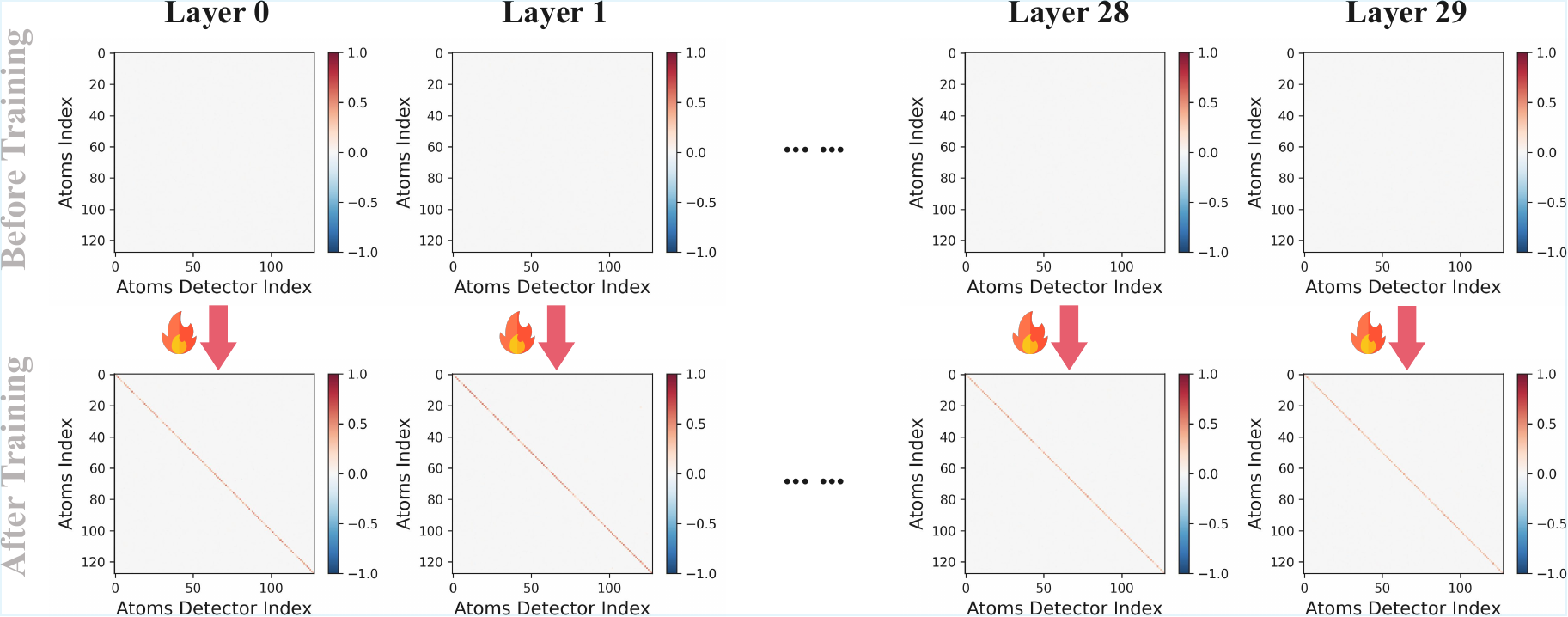}
    \caption{Spontaneous alignment between the encoder and decoder during training on Llama3.1-8B.}
  \label{fig:Encoder_After_Training_Llama3.1-8B}
\end{figure}


\begin{figure}
    \centering
    \includegraphics[width=0.88\textwidth]{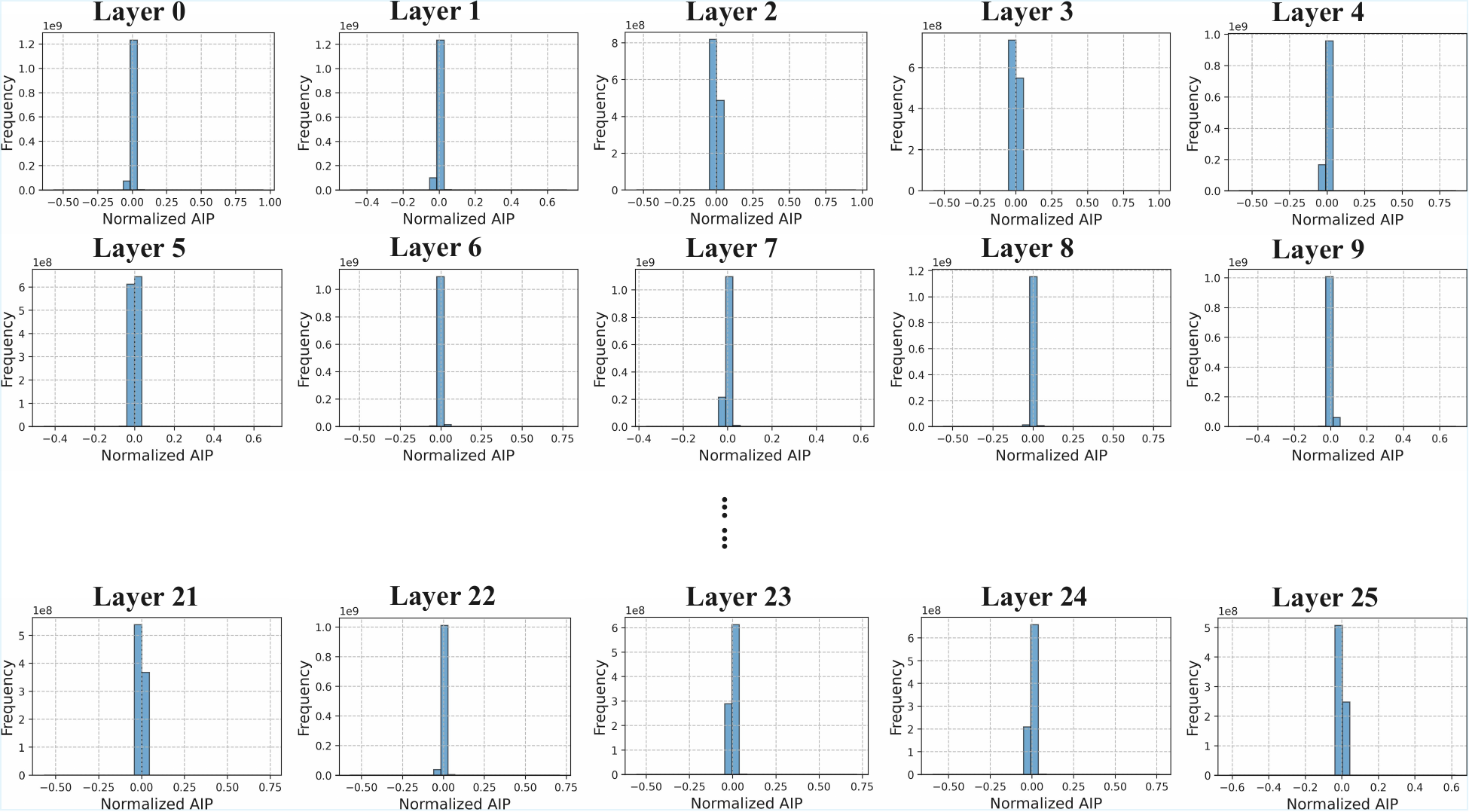}
    \caption{NAIP distribution of atoms across all layers of the Gemma2-2B.}
  \label{fig:Naip_Distribution_Gemma2-2B_Full}
\end{figure}

\begin{figure}
    \centering
    \includegraphics[width=0.88\textwidth]{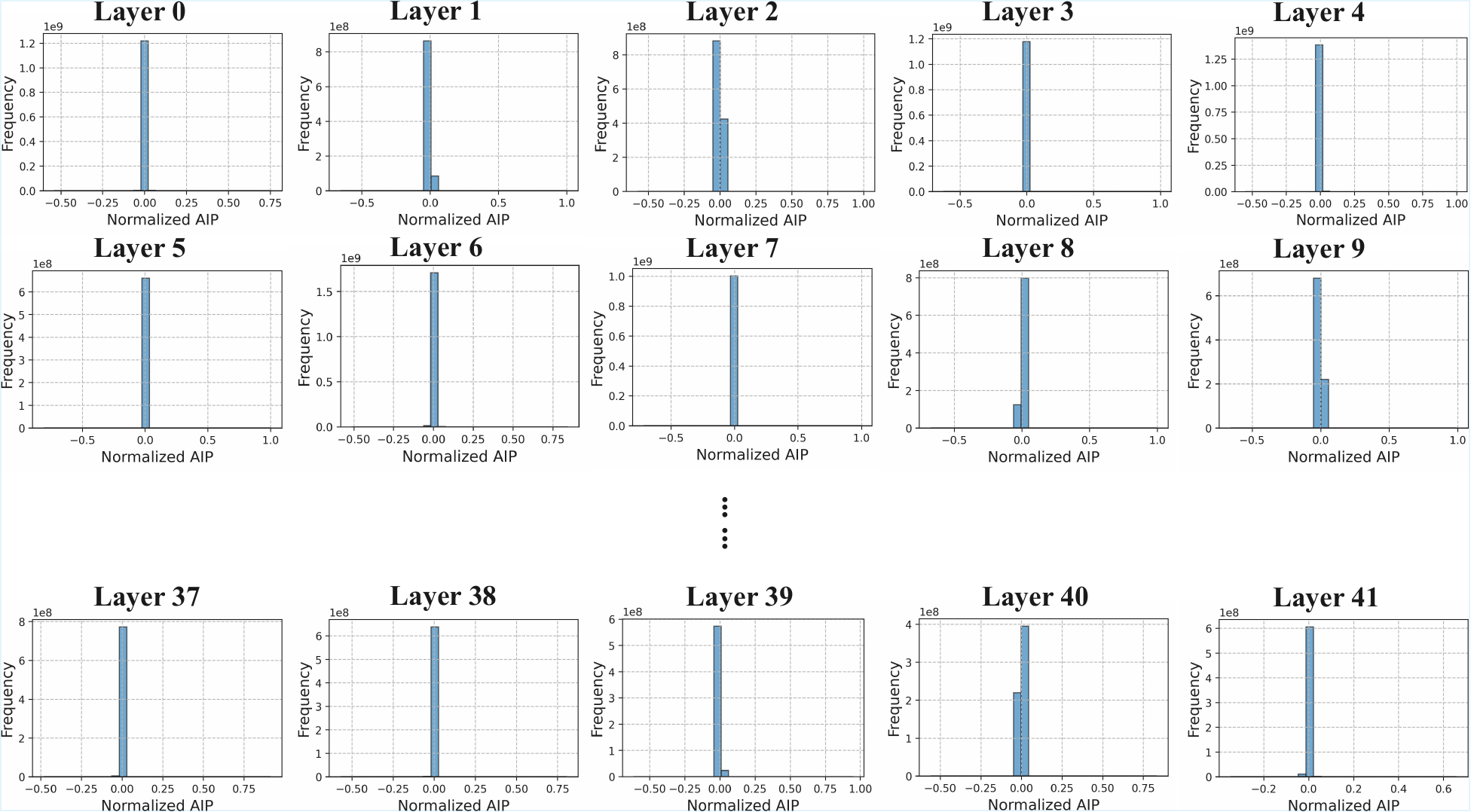}
    \caption{NAIP distribution of atoms across all layers of the Gemma2-9B.}
  \label{fig:Naip_Distribution_Gemma2-9B_Full}
\end{figure}

\begin{figure}
    \centering
    \includegraphics[width=0.88\textwidth]{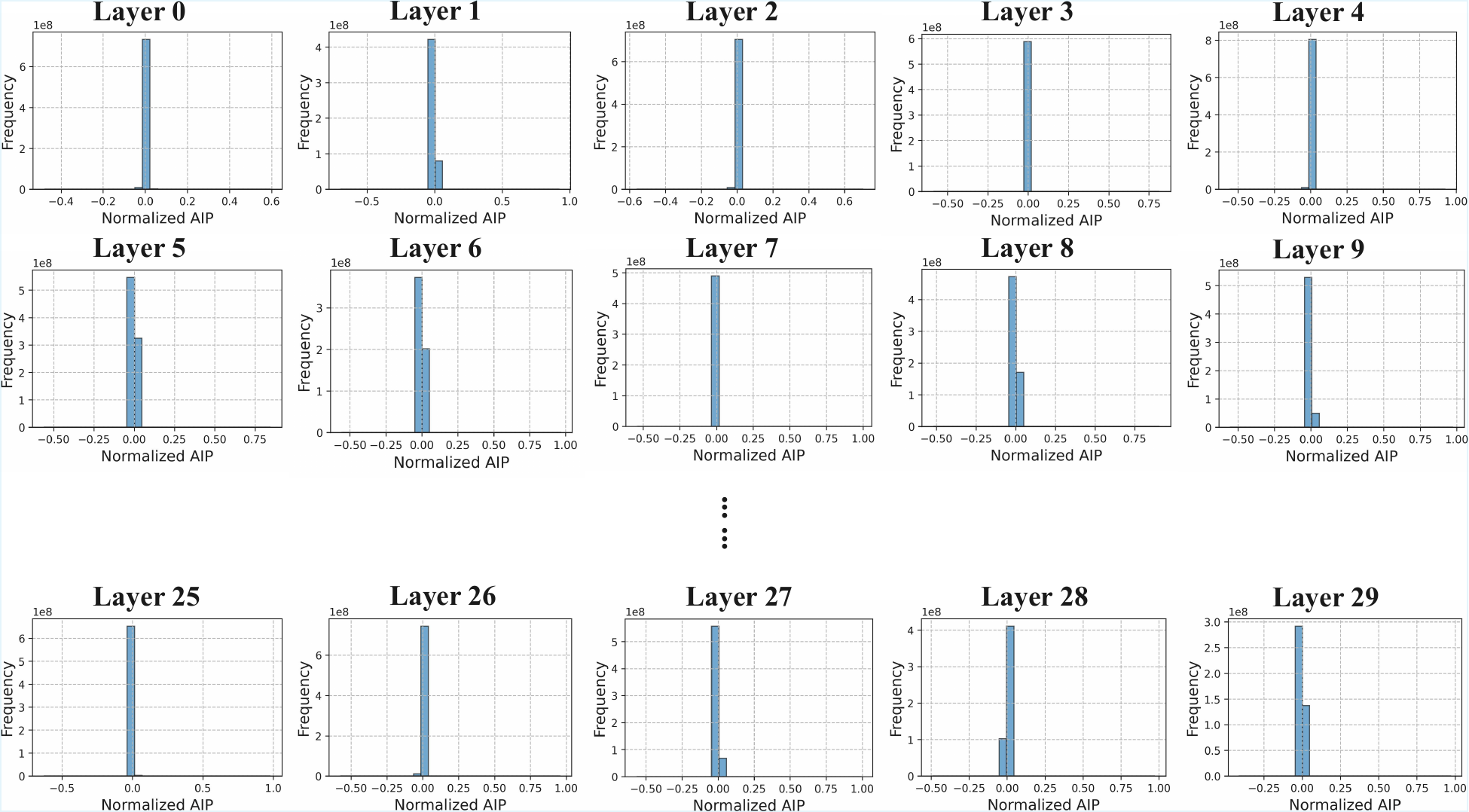}
    \caption{NAIP distribution of atoms across all layers of the Llama3.1-8B.}
  \label{fig:Naip_Distribution_Llama3.1-8B_Full}
\end{figure}

\subsection{Monosemanticity Evaluation}\label{appendix:Atoms_of_LLMs_Monosemanticity_Evaluation}

To evaluate the monosemanticity of the representational units, we adopt LLM-as-a-judge.

Specifically, to ensure diversity, we first manually select a set of heterogeneous entities (e.g., "United Kingdom", "Google Maps", "Suzuki GSX-R750", "Windows Vista", "Intel 80286", "Beijing", "Hawaii", "Tim Duncan", "Microsoft Word", "Vladimir Putin", "Apple Watch", "Chrome OS"). We then collect the representational units activated by these entities and aggregate them into a candidate pool. From this pool, we randomly sample ten units per selected layer: Gemma2-2B (layers 0, 5, 10, 15, 20, 25), Gemma2-9B (layers 0, 7, 14, 21, 28, 35, 41), and Llama3.1-8B (layers 0, 6, 12, 18, 24, 29). For each sampled unit, we retrieve all entities that activate it and use GPT-5.2 to assess its monosemanticity. The evaluation prompt is provided below:

\begin{tcolorbox}[
    colback=gray!8,
    colframe=gray!40,
    boxrule=0.5pt,
    arc=2pt,
    left=6pt,
    right=6pt,
    top=6pt,
    bottom=6pt,
    breakable
]
You are given a list of entities enclosed in square brackets [ ].

Inside the brackets, each entity is separated by a semicolon (;).

Your task is to analyze the entities and determine how many of them belong to the same semantic category (i.e., refer to the same type of real-world concept).

Important instructions:

- You should identify the largest group of entities that belong to the same category.

- Only count entities that clearly belong to the same category.

- Your answer must be a single integer.

You must provide your final answer strictly inside a box using the following format:

\texttt{\textbackslash box\{NUMBER\}}

Here is the list of entities:
\texttt{[\{entities\}]}
\end{tcolorbox}

We then compute, for each representational unit, the proportion of activated entities that are monosemantic, and report the mean and standard error of the mean (SEM) over units sampled at selected layers. Full results are shown in Fig.~\ref{fig:monosemanticity_score_across_models_and_layers}, revealing that monosemanticity increases with model scale and is generally higher in deeper layers than in shallower ones. Aggregating across layers within each model yields the results presented in Fig.~\ref{fig:monosemanticity_score_across_models}.

For example, in Gemma2-9B, an atom (ID 11346) in layer 28 is activated by entities including “Honolulu,” “aloha,” “Mufi Hannemann,” “Kirk Caldwell,” “Hawaii,” “Hawaiian Islands,” “Mauna Kea,” “USS Honolulu,” and “Aloha Stadium.” Notably, Mufi Hannemann was born in Honolulu and served as its mayor; Kirk Caldwell is a former Hawaii state representative and former mayor of Honolulu; and Mauna Kea is a volcano in the Hawaiian Islands. This example demonstrates that the atom consistently captures a semantically coherent “Hawaii–Honolulu” concept region, exhibiting clear monosemanticity.

Furthermore, we analyze atoms activated by "Beijing" in layers 1–6 of Gemma2-2B, and examine, at each layer, all entities that activate these atoms to characterize their corresponding concept regions (Tabs.~\ref{tab:beijing_atom_entities_layer1_gemma2-2b}–\ref{tab:beijing_atom_entities_layer6_gemma2-2b}).

\begin{figure}
    \centering
    \includegraphics[width=0.88\textwidth]{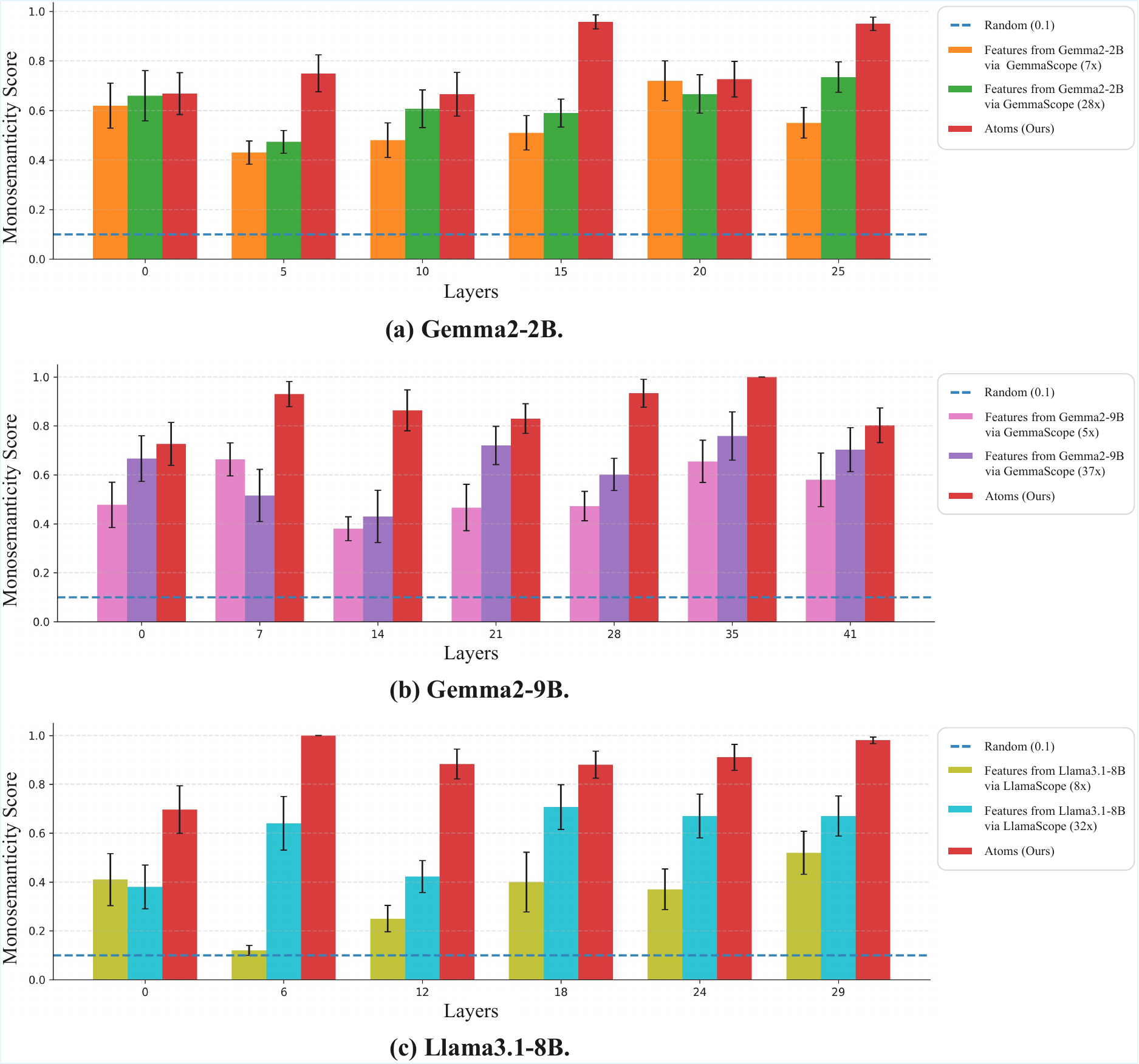}
    \caption{Monosemanticity scores of representational units across models and layers, using GPT-5.2 with manual verification. The blue dashed line indicates random-guess performance (0.1).}
  \label{fig:monosemanticity_score_across_models_and_layers}
\end{figure}

\begin{table}[htbp]
\centering
\caption{Entities grouped by atoms ID for \textit{Beijing} on layer 1 of Gemma2-2B.}
\begin{tabular}{@{}l p{0.78\linewidth}@{}}
\toprule
\textbf{Atoms ID} & \textbf{Entities} \\
\midrule
15264 & Beijing, Seoul, 1 Maccabees, Ulysses Dove \\
\midrule
15982 & Beijing, Siikainen, 36 China Town, Jim Allchin \\
\midrule
23987 & Beijing, Swann Memorial Fountain, Charles Chilton, Otto Neurath \\
\midrule
31322 & Shanghai, Beijing \\
\midrule
35951 & Beijing, Russia, Arkansas, Paris \\
\midrule
36035 & Beijing, Meiert Avis, Aviation Industry Corporation of China \\
\bottomrule
\end{tabular}
\label{tab:beijing_atom_entities_layer1_gemma2-2b}
\end{table}

\begin{table}[htbp]
\centering
\caption{Entities grouped by atoms ID for \textit{Beijing} on layer 2 of Gemma2-2B.}
\begin{tabular}{@{}l p{0.78\linewidth}@{}}
\toprule
\textbf{Atoms ID} & \textbf{Entities} \\
\midrule
620 & Shanghai, Beijing, Hanoi, Tokyo, Adam Maida \\
\midrule
6258 & Beijing, Majorca, Thailand, Greg Dyke \\
\midrule
7540 & 1300 Oslo, Beijing, Miami Horror, Lille \\
\midrule
10761 & Moscow, Beijing, Canberra, Pyongyang \\
\midrule
11519 & Karl Polanyi, Beijing, Cevdet Sunay, Mary Gaunt, Cyd Hayman, Les diamants de la couronne \\
\midrule
13418 & Beijing, Tarnobrzeg Voivodeship, Yakuza, Longs Peak, Jeep Wrangler \\
\midrule
15585 & Beijing, Ivan Koloff, Olinto Cristina \\
\midrule
22622 & Shanghai, Cleveland, Beijing, Delhi, Saint Lucia, St Lucia, Venice \\
\midrule
26002 & Beijing, Alte Oper, Intimate Stories, Seventeen, Five Star Krishna \\
\midrule
27116 & Ankara, Mandarin Oriental, Bangkok, Cairo, Beijing, Dublin, Jakarta, Amsterdam, Bratislava, Toronto, Sydney, Edinburgh, London, Honolulu, Auckland, Bali, Tokyo, Manila, Queens Gardens, Brisbane, Budapest, Montreal, Perth, Kolkata, Dubai, Melbourne, Copenhagen, Nairobi, Bangkok, Bangalore \\
\bottomrule
\end{tabular}
\label{tab:beijing_atom_entities_layer2_gemma2-2b}
\end{table}

\begin{table}[htbp]
\centering
\caption{Entities grouped by atoms ID for \textit{Beijing} on layer 3 of Gemma2-2B.}
\begin{tabular}{@{}l p{0.78\linewidth}@{}}
\toprule
\textbf{Atoms ID} & \textbf{Entities} \\
\midrule
9444  & Shanghai, Beijing \\
\midrule
24724 & Moscow, Beijing, Russia \\
\midrule
30463 & Beijing, Thailand \\
\midrule
32854 & Beijing, Madrid, Mariano Gonzalvo \\
\bottomrule
\end{tabular}
\label{tab:beijing_atom_entities_layer3_gemma2-2b}
\end{table}

\begin{table}[htbp]
\centering
\caption{Entities grouped by atoms ID for \textit{Beijing} on layer 4 of Gemma2-2B.}
\begin{tabular}{@{}l p{0.78\linewidth}@{}}
\toprule
\textbf{Atoms ID} & \textbf{Entities} \\
\midrule
1578  & Beijing, Cadbury \\
\midrule
11098 & Beijing, Jakarta \\
\midrule
11158 & Beijing \\
\midrule
15601 & Oslo, Moscow, Stockholm, Berlin, Athens, Helsinki, Beijing, Vienna, Geneva, Amsterdam, Seoul, Prague, Madrid, London, Warsaw, Kyoto, Naples, Tokyo, Budapest, Paris, Rome, Bangkok \\
\midrule
25755 & Stockholm, Helsinki, Beijing, Minneapolis, Minecraft, Copenhagen, Nairobi \\
\midrule
33322 & Shanghai, Beijing, Guangzhou, Macau, Hong Kong, Chongqing, Shenzhen, Wuhan \\
\bottomrule
\end{tabular}
\label{tab:beijing_atom_entities_layer4_gemma2-2b}
\end{table}

\begin{table}[htbp]
\centering
\caption{Entities grouped by atoms ID for \textit{Beijing} on layer 5 of Gemma2-2B.}
\begin{tabular}{@{}l p{0.78\linewidth}@{}}
\toprule
\textbf{Atoms ID} & \textbf{Entities} \\
\midrule
11453 & Beijing, The Great Citizen \\
\midrule
12661 & Beijing, Holycross-Ballycahill GAA \\
\midrule
19018 & Beijing, Registro, 4th of August Regime, Witnesses \\
\midrule
23750 & Moscow, Ankara, Beijing, Jakarta, Madrid \\
\bottomrule
\end{tabular}
\label{tab:beijing_atom_entities_layer5_gemma2-2b}
\end{table}

\begin{table}[htbp]
\centering
\caption{Entities grouped by atoms ID for \textit{Beijing} on layer 6 of Gemma2-2B.}
\begin{tabular}{@{}l p{0.78\linewidth}@{}}
\toprule
\textbf{Atoms ID} & \textbf{Entities} \\
\midrule
7533 & Johannesburg, Shanghai, Beijing, Colombo, Prafulla Chandra Ghosh \\
\midrule
16414 & Shenyang, Shanghai, Beijing, Guangzhou, Yangtze, Google China, Taobao, Tianjin, Chongqing, National Development and Reform Commission, Shenzhen, Qing dynasty, Aviation Industry Corporation of China, Qzone, Youku, Wuhan, People's Republic of China \\
\midrule
22386 & Beijing \\
\midrule
33958 & Carol Zhao, Shenyang, Shanghai, Beijing, Guangzhou, Seoul, Yangtze, Macau, Hanoi, Taipei, Hong Kong, Kaohsiung, South Korea, Busan, United States Army Military Government in Korea, Tianjin, Pyongyang, Incheon, Chongqing, Vietnam, Dennis Hwang, Shenzhen, Daejeon, North Korea, Wuhan \\
\bottomrule
\end{tabular}
\label{tab:beijing_atom_entities_layer6_gemma2-2b}
\end{table}


\end{document}